\documentclass{article}

\usepackage[accepted]{icml2024}

\usepackage{microtype}
\usepackage{graphicx}
\usepackage{booktabs} %

\usepackage{hyperref}

\usepackage{amsmath}
\usepackage{amssymb}
\usepackage{mathtools}
\usepackage{amsthm}

\usepackage[utf8]{inputenc} %
\usepackage[T1]{fontenc}    %
\usepackage{hyperref}       %
\usepackage{url}            %
\usepackage{booktabs}       %
\usepackage{amsfonts}       %
\usepackage{nicefrac}       %
\usepackage{microtype}      %
\usepackage{xcolor}         %

\usepackage{adjustbox}
\usepackage{amsmath}
\usepackage{amssymb}
\usepackage{mathtools}
\usepackage{amsthm}
\usepackage{mathtools}
\usepackage[ruled,linesnumbered,algo2e]{algorithm2e}
\usepackage{multirow}
\usepackage{wrapfig}
\usepackage{enumitem}
\usepackage{subcaption}
\usepackage{colortbl}
\usepackage{bigstrut}

\usepackage[toc, page, header]{appendix}
\definecolor{Gray}{gray}{0.95}
\definecolor{DarkGray}{gray}{0.5}

\DeclarePairedDelimiter{\ceil}{\lceil}{\rceil}

\usepackage[capitalize,noabbrev]{cleveref}

\usepackage[capitalize,noabbrev]{cleveref}

\theoremstyle{plain}
\newtheorem{theorem}{Theorem}[section]

\theoremstyle{definition}

\theoremstyle{remark}

\usepackage{amsmath,amssymb,amsthm}
\usepackage{color,mathtools}

\usepackage{multirow}
\usepackage{enumitem}

\usepackage{tablefootnote}

\usepackage{bm}

\usepackage{xspace}

\usepackage{amssymb,amsmath,amsthm,dsfont}

  \providecommand{\R}{\mathbb{R}} %

  \renewcommand{\epsilon}{\varepsilon}

\DeclarePairedDelimiterX{\inp}[2]{\langle}{\rangle}{#1, #2}
\DeclarePairedDelimiterX{\abs}[1]{\lvert}{\rvert}{#1}
\DeclarePairedDelimiterX{\norm}[1]{\lVert}{\rVert}{#1}
\DeclarePairedDelimiterX{\cbr}[1]{\{}{\}}{#1} %
\DeclarePairedDelimiterX{\rbr}[1]{(}{)}{#1} %
\DeclarePairedDelimiterX{\sbr}[1]{[}{]}{#1} %

\definecolor{mydarkblue}{rgb}{0,0.08,0.45}

\usepackage{xcolor}
\usepackage{color}
\usepackage{graphicx}

\usepackage{verbatim}

\usepackage{multirow}

\newcommand{\cD}{{\cal D}}

\newcommand{\cX}{{\cal X}}

\DeclareMathOperator{\rank}{rank}       %

\newcommand{\EE}[2]{{\bf E}_{#1}\left[#2\right] }

\usepackage{url}

\usepackage[colorinlistoftodos,bordercolor=orange,backgroundcolor=orange!20,linecolor=orange,textsize=scriptsize]{todonotes}

\newcommand{\tool}{{\sc Maestro}\xspace}
\newcommand{\technique}{{\sc LoD}\xspace}

\newif\ifarxiv

\arxivfalse
\arxivtrue

\ifarxiv
\usepackage{background}
\backgroundsetup{angle=0, 
scale=1,
color=black,
firstpage=true,
position=current page.north,
hshift=0pt,
vshift=-20pt,
contents={\ifnum\value{page}=1 PREPRINT: Accepted at the 41st Conference on Machine Learning (ICML 2024) \else \fi}
}
\newcommand{\rebb}[1]{\textcolor{black}{#1}}
\newcommand{\red}[1]{{#1}}
\newcommand{\camready}[1]{{#1}}
\else
\newcommand{\rebb}[1]{\textcolor{black}{#1}}
\newcommand{\red}[1]{{#1}}
\newcommand{\camready}[1]{{#1}}
\fi

\icmltitlerunning{Maestro: Uncovering Low-Rank Structures via Trainable Decomposition}

\begin{document}

\twocolumn[
\icmltitle{Maestro: Uncovering Low-Rank Structures via Trainable Decomposition}

\icmlsetsymbol{equal}{*}

\begin{icmlauthorlist}
\icmlauthor{Samuel Horv\'{a}th}{mbz}
\icmlauthor{Stefanos Laskaridis}{bra}
\icmlauthor{Shashank Rajput}{db}
\icmlauthor{Hongyi Wang}{cmu}
\end{icmlauthorlist}

\icmlaffiliation{mbz}{Mohamed bin Zayed
University of Artificial Intelligence (MBZUAI), Abu Dhabi, UAE}
\icmlaffiliation{bra}{Brave Software, London, UK}
\icmlaffiliation{db}{DataBricks, San Fransisco, USA}
\icmlaffiliation{cmu}{Carnegie Mellon University, Pittsburgh, USA}

\icmlcorrespondingauthor{Samuel Horv\'{a}th}{samuel.horvath@mbzuai.ac.ae}
\icmlcorrespondingauthor{Stefanos Laskaridis}{mail@stefanos.cc}

\icmlkeywords{Machine Learning, ICML}

\vskip 0.3in
]

\printAffiliationsAndNotice{}  %

\begin{abstract}

\red{Deep Neural Networks (DNNs) have been a large driver for AI breakthroughs in recent years.
However, these models have been getting increasingly large as they become more accurate and safe. This means that their training becomes increasingly costly and time-consuming and typically yields a single model to fit all targets.
Various techniques have been proposed in the literature to mitigate this, including pruning, sparsification, or quantization of model weights and updates. While achieving high compression rates, they often incur significant computational overheads at training or lead to non-negligible accuracy penalty. Alternatively, factorization methods have been leveraged for low-rank compression of DNNs. Similarly, such techniques (e.g.,~SVD) frequently rely on heavy iterative decompositions of layers and are potentially sub-optimal for non-linear models, such as DNNs.}
\red{We take a further step in designing efficient low-rank models and propose \tool, a framework for trainable low-rank layers. Instead of iteratively applying a priori decompositions, the low-rank structure is baked into the training process through \technique,
a low-rank ordered decomposition. Not only is this the first time importance ordering via sampling is applied on the decomposed DNN structure, but it also allows selecting ranks at a layer granularity.
Our theoretical analysis demonstrates that in special cases \technique recovers the SVD decomposition 
and PCA 
. Applied to DNNs, \tool enables the extraction of lower footprint models that preserve performance. Simultaneously, it enables the graceful trade-off between accuracy-latency for deployment to even more constrained devices without retraining.}
\end{abstract}

\section{Introduction}
\label{sec:introduction}
Deep Learning has been experiencing an unprecedented uptake, with models achieving a \mbox{(super-)human} level of performance in several tasks across modalities, giving birth to even more intelligent assistants~\citep{radford2023robust} and next-gen visual perception and generative systems~\citep{radford2021learning}. However, the price of this performance is that models are getting significantly larger, with training and deployment becoming increasingly costly~\citep{laskaridis2024melting}. Therefore, techniques from Efficient ML become evermore relevant~\citep{wan2023efficient}, and a requirement for deployment in constrained devices, such as smartphones or IoT devices~\citep{laskaridis2022future}.

Typical techniques to compress the network involve \textit{i)~quantization}, i.e.,~reducing precision of the model~\citep{wang2019deep} or communicated updates~\citep{seide20141, alistarh2017qsgd}, \textit{ii)~pruning} the model during training, e.g.,~through Lottery Ticket Hypothesis (LTH)~\citep{frankle2018the}, \textit{iii)~sparsification} of the network representation and updates, i.e.,~dropping the subset of coordinates~\citep{suresh2017distributed, alistarh2018convergence} or \textit{iv)~low-rank approximation~\citep{wang2021pufferfish,shrinkml2019}}, i.e.~keeping the most relevant ranks of the decomposed network. Despite the benefits during deployment, that is a lower footprint model, in many cases, the overhead during training time or the accuracy degradation can be non-negligible. Moreover, many techniques can introduce multiple hyperparameters or the need to fine-tune to recover the lost accuracy.

In this work, we focus training low-rank factorization. Specifically, we pinpoint the challenges of techniques~\citep{wang2021pufferfish,wang2023cuttlefish} when decomposing the parameters of each layer in low-rank space and the need to find the optimal ranks for each one at training time. To solve this, we \red{propose \technique (Low-rank ordered Decomposition), a non-trivially extended version of Ordered Dropout technique from~\citet{horvath2021fjord}, applied to progressively find the optimal decomposition for each layer of a DNN while training (Fig.~\ref{fig:tool}).} Critical differences to prior work include \textit{i)}~the non-uniformity of the search space (i.e. we allow for different ranks per layer), \textit{ii)}~the trainable aspect of the decomposition to reflect the data distribution, and \textit{iii)}~the gains to training and deployment time without sacrificing accuracy. Nevertheless, we also provide a latency-accuracy trade-off mechanism to deploy the model on even more constrained devices.

\begin{figure*}[t]
    \centering
    \vspace{-0.2cm}
    \includegraphics[trim={0 18cm 0 18cm}, clip, width=0.99\textwidth]{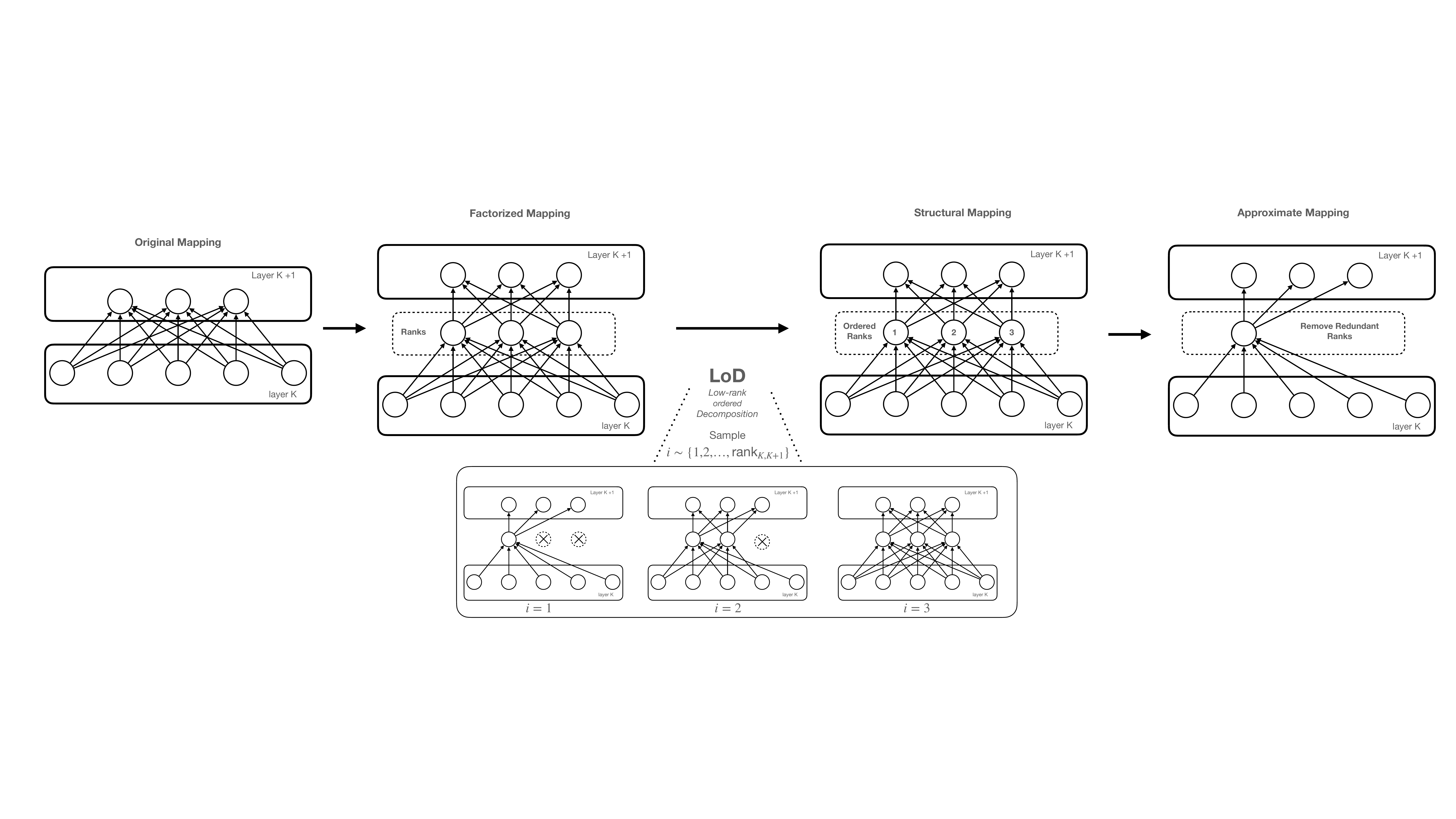}
    \captionsetup{font=small,labelfont=bf}
    \vspace{-0.1cm}
    \caption{\rebb{\tool's construction. To obtain low-rank approximation, the given linear map is decomposed and trained with \red{\technique} to obtain an ordered representation that can be efficiently pruned. 
    }}
    \label{fig:tool}
    \vspace{-0.6cm}
\end{figure*}
Our contributions can be summarized as follows:
\begin{itemize}[leftmargin=*,noitemsep,topsep=0pt]
    \item We propose \tool\footnote{\camready{The implementation can be found here: \texttt{\href{https://github.com/SamuelHorvath/Maestro-LoD}{https://github.com/SamuelHorvath/Maestro-LoD}}}}, a novel layer decomposition technique that enables learning low-rank layers in a progressive manner while training. 
    We fuse layer factorization and ordered dropout \red{into \technique, an extended variant of the ordered dropout}, \rebb{in a novel manner}, by embedding \red{ordered importance} directly into the factorized weights. By decomposing layers and training on stochastically sampled low-rank models, we apply ordered importance decomposed representation of each layer. We combine this with a \emph{ hierarchical group-lasso} term~\citep{yuan2006model} in the loss function to zero out redundant ranks and \emph{progressively shrink} the rank space. This way, we enable computationally efficient training achieved by the proposed decomposition  without relying on inexact and potentially computationally expensive \red{iterative decompositions} such as Singular Value Decomposition (SVD).
    \item \tool is \red{fundamentally} a theoretically motivated approach that embeds decomposition into training. First, we show that our new objective can recover \textit{i)}~the SVD of the target linear mapping for the particular case of uniform data distribution and \textit{ii)}~the Principal Component Analysis (PCA) of the data in the case of identity mapping.
    \item As \tool's decomposition is part of the training \red{process}, it also accounts for data distribution and the target function, contrary to SVD, which operates directly on learned weights. We show that this problem \emph{already arises} for a simple linear model and empirically generalize our results in the case of DNNs, by applying our method to different types of layers (including fully-connected, convolutional, and attention) spanning across three datasets and modalities. 
    \item \red{We illustrate that our technique achieves better results than SVD-based baselines at a lower cost. Indicatively, \tool can achieve on par or better results than low-rank SOTA methods on vision datasets (CIFAR-10, ImageNet) with lower training overhead due to progressive shrinking, while at the same time it reaches 6\% lower perplexity at a quarter of the computational cost and half of the parameters of SVD-variants in  Transformer models.}
\end{itemize}

\section{Related Work}
\label{sec:related_work}

The topic of Efficient ML has received a lot of attention throughout the past decade as networks have been getting increasingly computationally expensive. We distinguish between training and deployment time, with the latter having a more significant impact and thus amortizes the potential overhead during training. Nevertheless, \red{cost optimisation is evermore relevant for training large models, and the advent of} Federated Learning~\citep{mcmahan2017communication}, efficient training becomes increasingly relevant to remain tractable.

\noindent \textbf{Efficient inference.}
For efficient deployment, \rebb{various techniques have been proposed that either optimize} the architecture of the DNN in a hand-crafted~\citep{howard2017mobilenets} or automated manner (i.e.~NAS)~\citep{tan2019efficientnet}, they remove redundant computation by means of pruning parts of the network~\citep{han2015deep,carreira2018learning,frankle2018the, chen2021only,sreenivasan2022rare,li2016pruning, wen2016learning, hu2016network, wen2016learning, zhu2017prune, he2017channel, yang2017designing, liu2018rethinking, yu2019universally}, \rebb{in a structured or unstructured manner}, or utilise low-precision representation~\citep{wang2019deep} of the neurons and activations. \rebb{However, such techniques may involve non-negligible training overheads or lack flexibility of variable footprint upon deployment.}
Closer to our method, there have been techniques leveraging low-rank approximation (e.g.~SVD) for efficient inference~\citep{xue2013restructuring, sainath2013low, jaderberg2014speeding, wiesler2014mean,shrinkml2019}. Last, there is a category of techniques that dynamically resize the network at runtime for compute, memory or energy efficiency, based on early-exiting~\citep{laskaridis2021adaptive} or dynamic-width~\citep{yu2018slimmable} and leverage the accuracy-latency tradeoff.

\noindent \textbf{Efficient training.}
On the other hand, techniques for efficient training become very relevant nowadays when scaling DNNs sizes~\cite{hu2021lora} or deploying to embedded devices~\citep{lin2022ondevice}, and oftentimes offer additional gains at deployment time. Towards this goal, there have been employed methods where part of the network is masked~\citep{sidahmed2021efficient} or dropped~\citep{alam2022fedrolex,caldas2019expanding,Wu_2018_CVPR} during training, with the goal of minimizing the training footprint. Similarly to early-exiting, \rebb{multi-exit variants for efficient training~\citep{kim2023depthfl,liu2022federated} have been proposed}, and the same applies for width-based scaling~\cite{horvath2021fjord,diao2021heterofl}. Last but not least, in the era of transformers and LLMs, where networks have scaled exponentially in size, PEFT-based techniques, such as adapter-based fine-tuning~\citep{houlsby2019parameter} (such as LoRA~\citep{hu2021lora}), become increasingly important and make an important differentiator for tackling downstream tasks.

\textbf{Learning ordered representation.} Ordered Dropout (OD) was proposed as a mechanism for importance-based pruning for the easy extraction of sub-networks devised to allow for heterogeneous federated training~\citep{horvath2021fjord}. \red{Similar constructions were proposed to be applied in the representation layer of autoencoders~\citep{rippel2014learning} to enforce identifiability of the learned representation or the last layer of the feature extractor~\citep{horvath2021hyperparameter} to learn an ordered set of features for transfer learning.}
\red{Contrary to prior work, \tool's \technique non-trivially extends ordered representation in three meaningful ways. First, it is the first work that is applied to the decomposed network, which gets progressively shrunk as redundant ranks converge to zero. This is achieved through hierarchical group lasso penalty, as described in Sec.~\ref{sec:training}. Second, \technique allows for heterogeneous (non-uniform) ranks per layer, yielding a much richer operating space. Last, through \technique we can leverage the ordered representation of ranks at inference time to further compress the model, allowing a graceful accuracy-latency tradeoff for deployment on more constrained devices, without the need to retrain.}

\vspace{-0.3cm}
\section{\tool}
\label{sec:tool}
In this work, we focus on low-rank models as a technique to reduce the neural network model's computational complexity and memory requirements. The main challenge that we face is the selection of the optimal rank or the trade-off between the efficiency and the rank for the given layer. Therefore, we devise an \emph{importance-based} training technique, \tool, which learns not only a mapping between features and responses but also \emph{learns the decomposition} of the trained network. This is achieved by factorizing all the layers in the network.

\subsection{Formulation}

\noindent
\textbf{Low-rank approximation.} Our inspiration comes from the low-rank matrix approximation of a matrix $A \in \R^{m \times n}$. For simplicity, we assume that $A$ has {rank at most $r = \min\cbr{m, n}$ with $k \leq r$} distinct non-zero singular values $\tilde{\sigma}_1 > \tilde{\sigma}_2 > \hdots > \tilde{\sigma}_k > 0$, with corresponding left and right singular vectors $\tilde{u}_1, \tilde{u}_2, \hdots, \tilde{u}_k \in R^m$ and $\tilde{v}_1, \tilde{v}_2, \hdots, \tilde{v}_k \in R^n$, respectively. For such a matrix, we can rewrite its best $l$-rank approximation as the following minimization problem
{%
\begin{align}
    \label{eq:svd_optim}
    & \textstyle\min_{U \in \R^{m \times l}, V \in \R^{n \times l}} \norm*{\textstyle\sum_{i=1}^l u_i v_i^\top - A}_F^2
\end{align}}
where $c_i$ denotes the $i$-th column of matrix $C$ and $\norm{\cdot}_F$ denotes Frobenius norm. We note that Problem~\eqref{eq:svd_optim} is non-convex and non-smooth. However, \cite{ye2021global} showed that the randomly initialized gradient descent algorithm solves this problem in polynomial time. In this work, we consider the best rank approximation across all the ranks. Eckart–Young–Mirsky theorem leads to the following objective
{%
\begin{align}
    \label{eq:od_optim}
    \begin{split}
        &\textstyle\min_{U \in \R^{m \times r}, V \in \R^{n \times r}} \frac{1}{r} \sum_{b=1}^r \norm*{U_{:b}V_{:b}^\top - A}_F^2,
    \end{split}
\end{align}}
where $C_{:b}$ denotes the first $b$ columns of matrix $C$. This objective, up to scaling, recovers SVD of $A$ exactly, and for the case of distinct non-zero singular values, the solution is, up to scaling, unique~\cite{horvath2021fjord}. This formulation, however, does not account for the data distribution, \textit{i.e., 
it cannot tailor the decomposition to capture specific structures that appear in the dataset.}

\textbf{LoD for data-dependent low-rank approximation.}
Therefore, the next step of our construction is to extend this problem formulation with data that can further improve compression, reconstruction, and generalization and incorporate domain knowledge. We assume that data comes from the distribution $x \sim \cX$ centered around zero, i.e., $\EE{x \sim \cX}{x} = 0$.\footnote{We make this assumption for simplicity. It can be simply overcome by adding a bias term into the model.}, and the response is given by $y = Ax$. In this particular case, we can write the training loss as
{%
\begin{align}
    \label{eq:low_rank_optim}
    \begin{split}
    \textstyle
        &\min_{U \in \R^{m \times r}, V \in \R^{n \times r}} \EE{x, y \sim \cX}{\sum_{b=1}^r \frac{1}{r}\norm*{U_{:b}V_{:b}^\top x - y}^2}.
    \end{split}
\end{align}}
\red{It is important to note that the introduced problem formulation~\eqref{eq:low_rank_optim} for the neural network with a single hidden layer and no activations can be solved using stochastic algorithms by sampling from the data distribution $\cX$ (subsampling) and rank distribution $\cD$. When we apply \technique to DNNs, contrary to any prior work \citep{horvath2021fjord,rippel2014learning,diao2021heterofl}, our formulation is the first one to apply importance-based ordered dropout to the decomposed representation of each layer, thus preserving the dimensionality. More importantly, we decompose each layer \emph{independently} and allow for finding the optimal rank per layer in a data-informed manner. We discuss details in the next paragraph.}

\textbf{DNN low-rank approximation.} For Deep Neural Networks (DNNs), we seek to uncover the optimal ranks for a set of $d$ linear mappings $W^1 \in \R^{m_1 \times n_1}, \hdots, W^d \in \R^{m_d \times n_d}$ , where $W^i$'s are model parameters and {$d$ is model depth}, e.g., weights corresponding to linear layers\footnote{We can apply our decomposition on different types of layers, such as Linear, Convolutional and Transformers as shown in Sec.~\ref{sec:layer_factorization}.}, by decomposing them as $W^i = U^i\left(V^i\right)^\top$. 
We discuss how these are selected in the next section. To decompose the network, we aim to minimize the following objective
\begin{align}
    \label{eq:low_rank_ful}
    \begin{split}
    \textstyle\mathbf{E}_{x, y \sim \cX}&\left[\frac{1}{\sum_{i = 1}^d r_i} \sum_{i=1}^d \sum_{b=1}^{r_i}l(h(U^1\rbr*{V^1}^\top, \hdots, \right. \\
        &\left. U^i_{:b}\rbr*{V^i_{:b}}^\top, \hdots, U^d\rbr*{V^d}^\top, W^{o}, x), y)\right],
    \end{split}
\end{align}
where $r_i = \min\cbr*{m_i, n_i}$, $l$ is a loss function, $h$ is a DNN, and $W^{o}$ are the other weights that we do not decompose. We note that our formulation aims to decompose each layer, while decompositions across layers do not directly interact. The motivation for this approach is to uncover low-rank structures within each layer that are not affected by inaccuracies from other layers due to multiple low-rank approximations.

\SetArgSty{textnormal} %
\begin{algorithm2e}[t]
    \footnotesize
    \SetAlgoLined
	\LinesNumbered
	\DontPrintSemicolon
	\KwIn{epochs $E$, dataset $\cD$, model $h$ parametrized by $U^1 \in \R^{m_1 \times r_1}$, $V^1 \in \R^{n_1 \times r_1}, \hdots, U^d \in \R^{m_d \times r_d}, V^d \in \R^{n_d \times r_d}$, $W^o$, and hyperparameters $\lambda_{gl}$, $\epsilon_{ps}$}
	
	\For( // \textit{Epochs}){$t \gets 0$ \KwTo $E-1$}{
	    \For( // \textit{Iterate over dataset}){$(x,y) \in \cD$}{
                Sample $(i, b) \sim \cbr*{\cbr*{(i, b)}_{b=1}^{r_i}}_{i=1}^d$; // \textit{LoD} \;
                $L = l(h(U^1\rbr*{V^1}^\top, \hdots, U^i_{:b}\rbr*{V^i_{:b}}^\top, \hdots, U^d\rbr*{V^d}^\top, W^{o}, x)$,
                $\quad y) $
                $+ \lambda_{gl} \sum_{i=1}^d \sum_{b=1}^{r_i}  \rbr*{\norm*{U^i_{b:}} + \norm*{V^i_{b:}}}$ // \textit{Loss}\; 
                L.backward() // \textit{Update weights}\;
                }
            \For {$i \gets 1$ \KwTo $d$} {
                \For {$b \gets 1$ \KwTo $r_i$} {
            // \textit{rank importance thresholding} \;
            \If{$\norm*{V^i_{b:}} \norm*{U^i_{b:}} \leq \epsilon_{ps}$} { 
                $r_i = b-1$ // \textit{progressive shrinking} \;
                \textbf{break}\;
                }
            }
        }
    \caption{\footnotesize \mbox{\textbf{\tool} (Training Process)}}
    \label{alg:training}
	}
\end{algorithm2e}

\subsection{Layer Factorization}
\label{sec:layer_factorization}
The following sections discuss how we implement model factorization for different architectures. 

\noindent\textbf{FC layers.}
A 2-layer fully connected (FC) neural network can be expressed as $f(x) = \sigma(\sigma(x W_1) W_2)$, where $W$s are weight matrices of each FC layer, and $\sigma(\cdot)$ is any arbitrary activation function, e.g., ReLU. The weight matrix $W$ can be factorized as $U V^\top$.

\noindent\textbf{CNN layers.} 
For a convolution layer with dimension, $W\in \mathbb{R}^{m \times n \times k \times k}$ where $m$ and $n$ are the number of input and output channels, and $k$ is the size of the convolution filters. Instead of directly factorizing the $4$D weight of a convolution layer, we factorize the unrolled $2$D matrix. Unrolling the 4D tensor $W$ leads to a 2D matrix with shape $W_{\text{unrolled}} \in \mathbb{R}^{mk^2 \times n}$, where each column represents the weight of a vectorized convolution filter. Factorization can then be conducted on the unrolled 2D matrix; see \citep{wang2021pufferfish} for details.

\noindent\textbf{Transformers.}
A Transformer layer consists of a stack of encoders and decoders~\cite{vaswani2017attention}. The encoder and decoder contain three main building blocks: the multi-head attention layer, position-wise feed-forward networks (FFN), and positional encoding. We factorize all trainable weight matrices in the multi-head attention (\rebb{MHA}) and the FFN layers. {The FFN layer factorization can directly adopt the strategy from the FC factorization.} A $p$-head attention layer learns $p$ attention mechanisms on the key, value, and query ($K, V, Q$) of each input token:
{\setlength{\abovedisplayskip}{0pt}\setlength{\belowdisplayskip}{0pt}
\begin{align*}
\texttt{MHA}(Q, K, V) = \texttt{Concat}(\text{head}_1, \dots, \text{head}_p) W^O.
\end{align*}
}
Each head performs the computation of: 
{\setlength{\abovedisplayskip}{0pt}\setlength{\belowdisplayskip}{0pt}
\begin{align*}
\texttt{head}_i &= \texttt{Attention}(Q W^{(i)}_Q, K W^{(i)}_K, V W^{(i)}_V) \\
                &= \texttt{softmax}\left(\frac{Q W^{(i)}_Q W^{(i)\top}_K K^\top}{\sqrt{d/p}}\right) V W^{(i)}_V.
\end{align*}}
where $d$ is the hidden dimension. The trainable weights $W^{(i)}_Q, W^{(i)}_K, W^{(i)}_V, i \in \{1, 2, \dots, p\}$ can be factorized by simply decomposing all learnable weights $W{\cdot}$ in an attention layer and obtaining $U{\cdot} V^\top{\cdot}$~\cite{vaswani2017attention}.

\vspace{-0.2cm}
\subsection{Training Techniques}
\label{sec:training}
Having defined the decomposition of typical layers found in DNNs, we move to formulate the training procedure of our method, formally described in Algorithm~\ref{alg:training}. 
Training the model comprises an iterative process of propagating forward on the model by \emph{sampling a rank} $b_i$ per decomposed layer $i$ up to maximal rank $r_i$ (line 3). We calculate the loss, which integrates an additional \emph{hierarchical group lasso} component (lines 4) and \emph{backpropagate} on the sampled decomposed model (line 5). At the end of each epoch, we \emph{progressively shrink} the network by updating the maximal rank $r_i$, based on an importance threshold $\epsilon_{ps}$ (line 11). We provide more details about each component below.

\noindent \textbf{Efficient training via sampling.} In Sec.~\ref{sec:theory}, we show that for the linear case \eqref{eq:low_rank_optim}, the optimal solution corresponds to PCA over the linearly transformed dataset. This means that the obtained solution contains \rebb{orthogonal} directions. This property is beneficial because it directly implies that when we employ gradient-based optimization, not only is the gradient zero at the optimum, but the gradient with respect to each summand in Equation \eqref{eq:low_rank_optim} is also zero. The same property is directly implied by overparametrization~\cite{ma2018power} or strong growth condition~\cite{schmidt2013fast}. 
As a consequence, this enables us to sample only one summand at a time and obtain the same quality solution. When considering \eqref{eq:low_rank_ful} as an extension to \eqref{eq:low_rank_optim}, it is unclear whether this property still holds, which would also imply that the set of stationary points of \eqref{eq:low_rank_optim} is a subset of stationary points of the original objective without decomposition. However, in the experiments, we observed that sampling is sufficient to converge to a good-quality solution. If this only holds approximately, \rebb{one} could leverage fine-tuning to recover the loss in performance. 

\noindent\textbf{Efficient rank extraction via hierarchical group-lasso.} By definition, \eqref{eq:low_rank_optim} leads to an ordered set of ranks for each layer. This ordered structure enables efficient rank extraction and selection. To effectively eliminate unimportant ranks while retaining the important ones, thus leading to a more efficient model, we consider Hierarchical Group Lasso (HGL)~\cite{lim2015learning} in the form
{%
\begin{align}
    \label{eq:group_lasso}
    \lambda_{gl} \sum_{i=1}^d \sum_{b=1}^{r_i}  \rbr*{\norm*{U^i_{b:}} + \norm*{V^i_{b:}}},
\end{align}
}
where $C_{b:}$ denotes the matrix that contains all the columns of $C$ except for the first $b-1$ columns.

\noindent\textbf{Progressive shrinking.} HGL encourages that unimportant ranks become zero and can be effectively removed from the model. To account for this, for each layer we remove $V^i_{b:}$ and $U^i_{b:}$ (i.e., set $r_i = b - 1$) if $\norm*{V^i_{b:}} \norm*{U^i_{b:}} \leq \epsilon_{ps}$, where $\epsilon_{ps}$ is a pre-selected threshold -- and a hyperparameter of our method.

\noindent\camready{\textbf{Hyperparameter optimization.} We provide an algorithm for finding the optimal value for the hyperparameter $\lambda_{gl}$ in Alg.~\ref{alg:hpo}. From the evaluation of Tab.~\ref{tab:lenet_gp_lambda}-\ref{tab:resnet50_gp_lambda}, this strategy typically requires at most 2-3 times the computational effort (in terms of FLOPs) compared to a single training loop with an optimally chosen $\lambda_{gl}$. This is significantly easier than tuning the per-layer maximal rank and the pretraining steps, where the full-rank model is being pretrained, as is the case in other low-rank baselines. Equivalently, the value of $\epsilon_{ps}$ represents the effective zero-point in our algorithm, typical value of which was $1e-7$ in our experiments.}

\SetArgSty{textnormal} %
\begin{algorithm2e}[h]
    \footnotesize
    \SetAlgoLined
	\LinesNumbered
	\DontPrintSemicolon
	\KwIn{constraints (e.g., min required accuracy, max \#parameters), epochs $E$, dataset $\cD$, model $h$, evaluation frequency $E^\text{eval}_\text{every}$, $\epsilon_{ps}$, limits for HPO: \textit{largeValue}, \textit{smallValue}}
        $\lambda_{gl}$ = \textit{largeValue} \;
        old\_model = NULL \;
        \While{$\lambda_{gl}$ > \textit{smallValue}} {
            model = RandInit(model) \;
            \For( // \textit{Epochs}){$t \gets 0$ \KwTo $E-1$} {
                Train(model, dataset, $\lambda_{gl}$)\;
                \If {$t$ mod $E^\text{eval}_\text{every}$ == 0} {
                    acc = CalculateAcc(model, dataset)\;
                    flops, params = MeasureFootprint(model, $\epsilon_{ps}$)\;
                    \If {acc $\geq$ constraints[`acc'] \textbf{and} params $\leq$ constraints[`params']} {
                        \Return model \textit{// model satisfying constraints was found} \;
                    }
                    \If{params $<$ few\_params}{
                        \textbf{break} \textit{// model too sparse} \;
                    }
                }
            }
            flops, params = MeasureFootprint(model, $\epsilon_{ps}$)\;
            \If{params > constraints[`params']} {
                \textit{/** no model satisfies constraints, return the last model satisfying parameters constraints**/} \;
                \Return old\_model  \;
            }
            old\_model = Copy(model)\;
            $\lambda_{gl}$ = $\lambda_{gl}$ / 2\;
        }

    \caption{\footnotesize \mbox{\textbf{\tool} (Hyper-parameter optimization)}}
    \label{alg:hpo}
\end{algorithm2e}

\noindent\textbf{Initialization.}
Initialization is a key component of the training procedure~\cite{he2015delving, mishkin2015all}. To adopt the best practices from standard non-factorized training, we follow a similar approach to \cite{khodak2021initialization, wang2021pufferfish}, where we first initialize the non-factorized model using standard initialization. For initializing factorized layers, we use the \rebb{Singular Value Decomposition} of the non-factorized initialization -- in a full-rank form -- to ensure that the resulting product matrix is the same as the original parameter decomposition. In addition, SVD is an optimal decomposition for the linear case with uniform data. {However, in contrast with the adaptive baseline method~\citep{wang2023cuttlefish} we only decompose once, rather than on every training iteration.} \red{As such, we only run decomposition once and progressively shrink the ranks in a data-centric manner. This is contrary to related work~\cite{wang2021pufferfish,wang2023cuttlefish} that requires manual rank and layer selection and full-rank warmup to achieve the desired performance, at the cost of training overhead, of course.}

\begin{figure*}[t]
  \vspace{-0.2cm}
  \centering
  \begin{subfigure}[t]{0.48\textwidth}
    \centering
    \includegraphics[width=.7\textwidth]{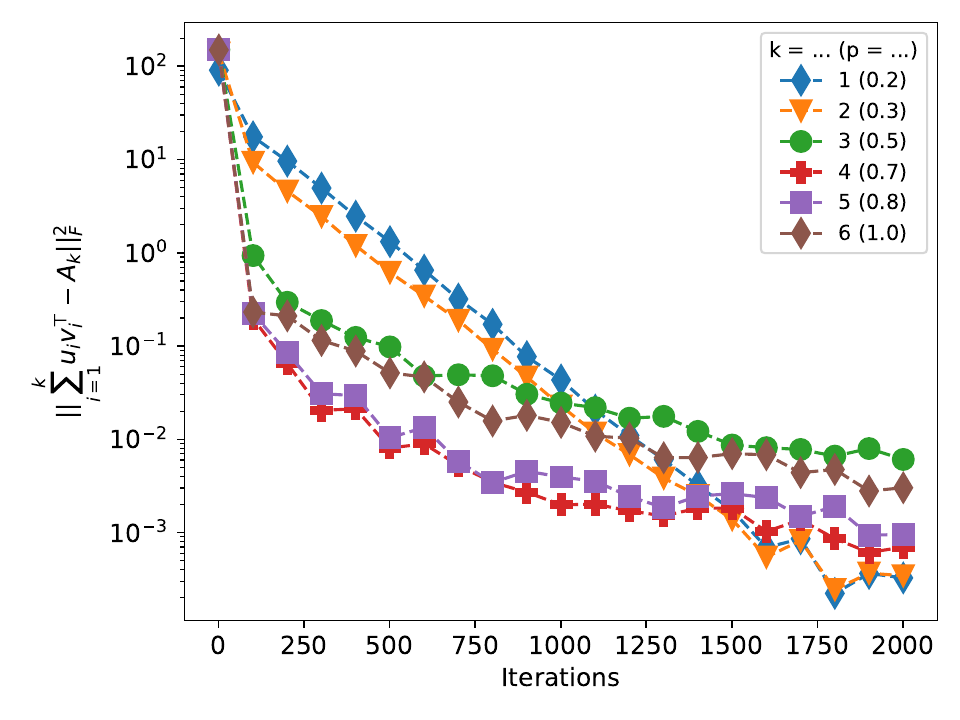}
    \vspace{-0.4cm}
    \captionsetup{font=small,labelfont=bf}
    \caption{Verification that \tool recovers SVD for linear mapping with uniform data. We display the L2 distance between the best rank $k$ and \tool's approximation of mapping $A$. 
    The target matrix was randomly generated $9 \times 6$ matrix with rank $3$. $p$ and $k$ represent relative and actual rank.}
    \label{fig:svd}
  \end{subfigure}
  \hfill
  \begin{subfigure}[t]{0.48\textwidth}
    \centering
    \includegraphics[width=.7\textwidth]{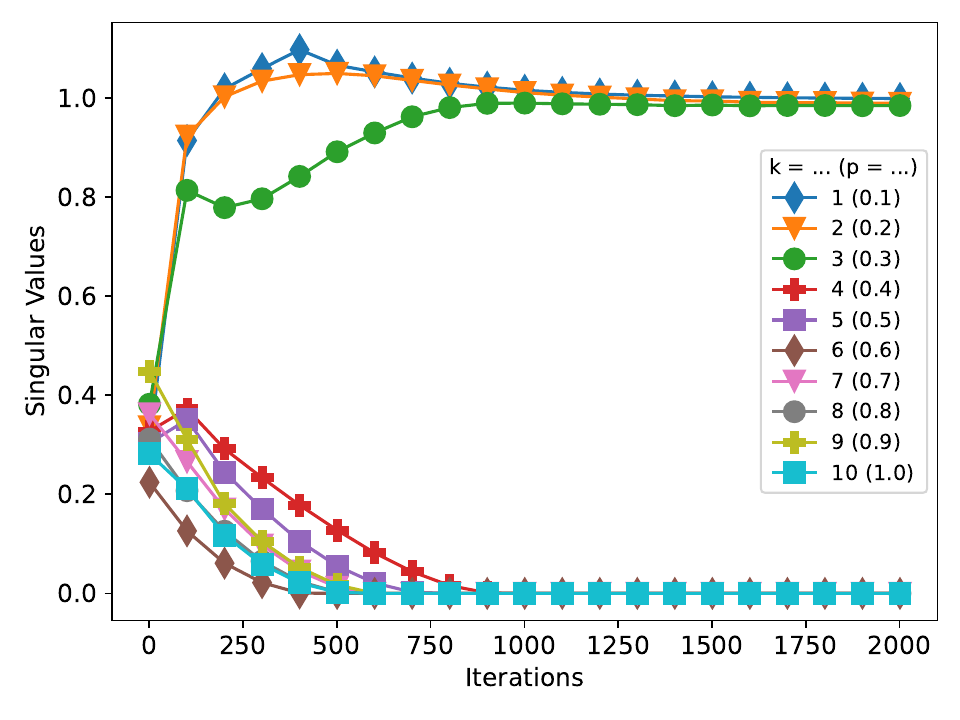}
    \captionsetup{font=small,labelfont=bf}
    \vspace{-0.4cm}
    \caption{Verification that \tool recovers PCA for identity mapping. The plot displays the estimates of singular values. The data distribution has only 3 directions. It is expected that the top $3$ ranks will converge to value one and the rest to zero. $p$ and $k$ stand for relative and actual rank, respectively.}
    \label{fig:pca}
  \end{subfigure}
  \captionsetup{font=small,labelfont=bf}
  \vspace{-0.2cm}
  \caption{Empirical showcase of theoretical properties of the \tool's formulation.} %
  \vspace{-0.2cm}
  \label{fig:both_figures}
\end{figure*}

\vspace{-0.2cm}
\subsection{Train-Once, Deploy-Everywhere}
\label{sec:deploy}

Up until now, we have described how our method works for training low-rank models, which yield computational, memory, network, and energy~\citep{wu2022sustainable} bandwidth benefits during training.
At deployment time, one can directly deploy the final model (rank $r_i$ for each layer) on the device, which we acquire from performing a threshold sweep of $\epsilon_{ps}$ over the effective range of rank importance across layers.
However, in case we want to run on even more constrained devices, such as mobile or embedded~\citep{almeida2021smart} systems, the learned decomposition also gives us the flexibility to further compress the model in a straightforward manner, effectively trading off accuracy for a smaller model footprint. Inspired by \cite{yu2019autoslim}, we propose to use \emph{greedy search}. We begin with the current model and compare model performance across various low-rank models, each created by removing a certain percentage of ranks from each layer. We then eliminate the ranks that cause the least decrease in performance. This process is iterated until we reach the desired size or accuracy constraint. To make this approach efficient, we estimate the loss using a single mini-batch with a large batch size (e.g.,~2048). This also avoids issues with BatchNorm layers; see \cite{yu2019autoslim} for details.

In summary, \tool comprises a technique for trainable low-rank approximation during training time that progressively compresses the model, reflecting the data distribution, and a method that enables a graceful trade-off between accuracy and latency for embedded deployment, by selecting the most important parts of the network. We validate these claims in Sec.~\ref{sec:performance_comparison} and~\ref{sec:tradeoff}, respectively.

\vspace{-0.3cm}
\section{Theoretical Guarantees}
\vspace{-0.2cm}
\label{sec:theory}

\begin{table*}
\begin{minipage}{0.6\linewidth}
    \centering
    \vspace{-0.25cm}
    \centering
    \begin{subfigure}[t]{0.49\linewidth}
        \centering
        \includegraphics[width=\textwidth]{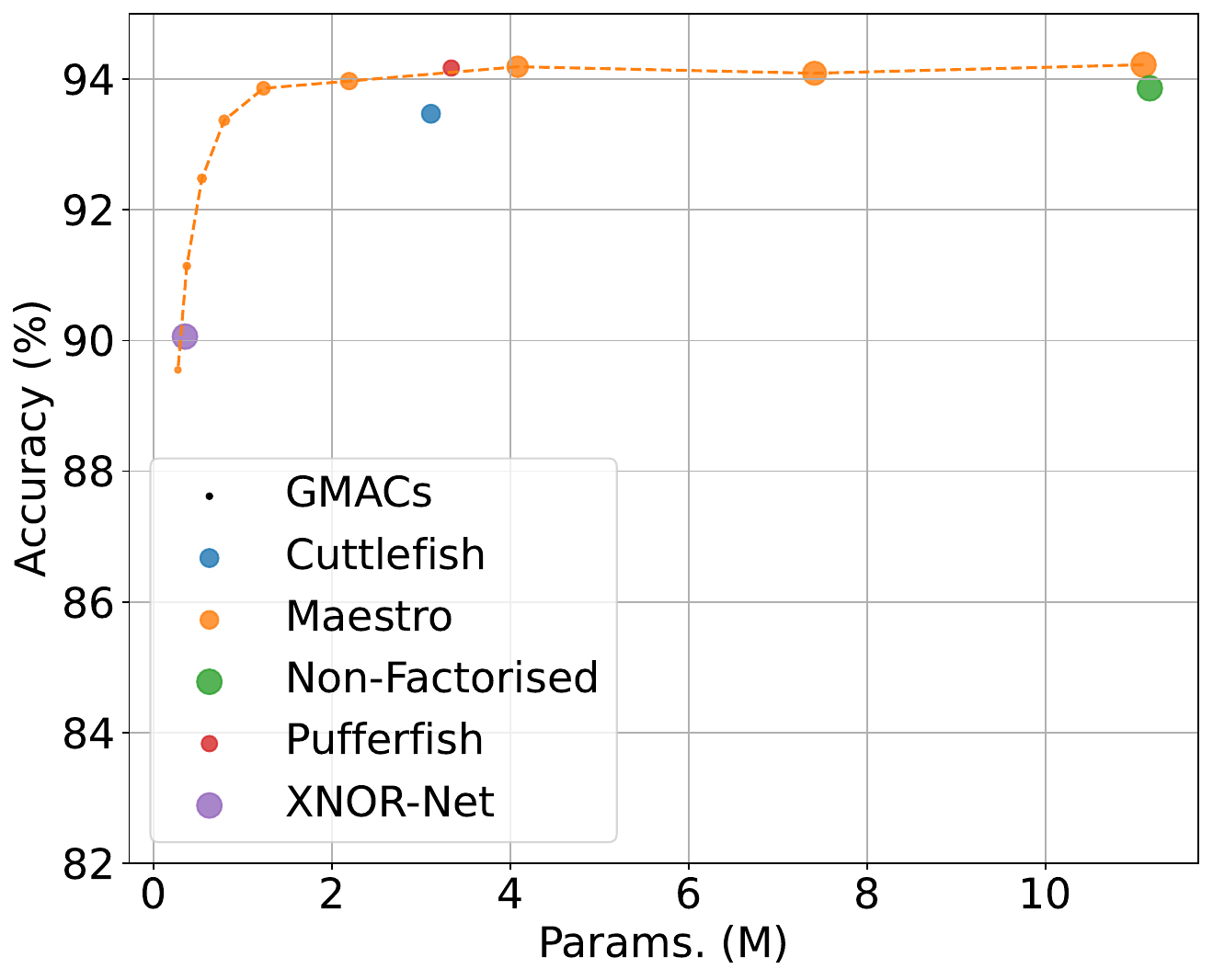}
        \vspace{-0.6cm}
        \captionsetup{font=small,labelfont=bf}
        \caption{ResNet18.}
        \label{fig:accuracy_size_resnet18}
    \end{subfigure}
    \begin{subfigure}[t]{0.49\linewidth}
        \centering
        \includegraphics[width=\textwidth]{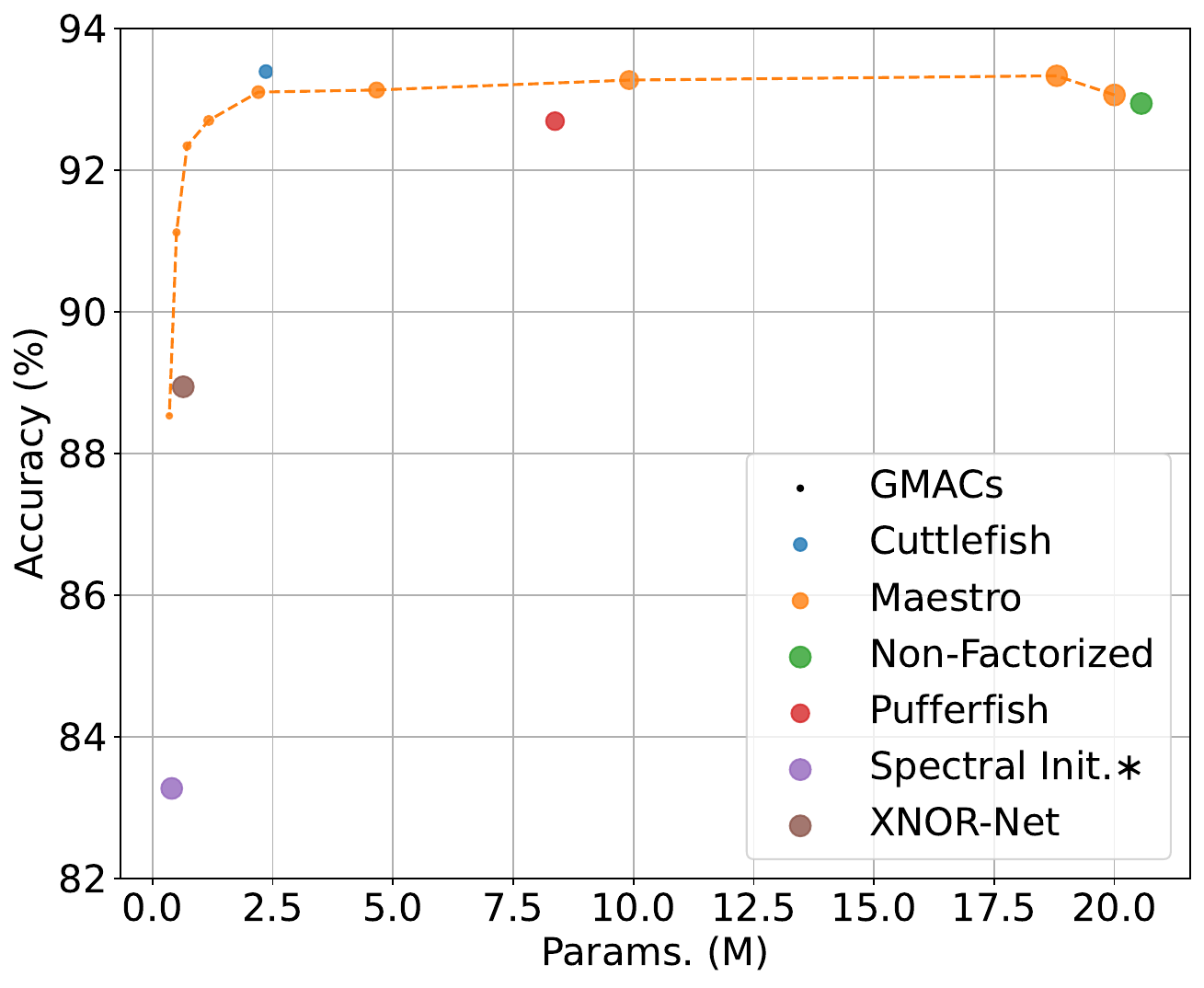}
        \captionsetup{font=small,labelfont=bf}
        \vspace{-0.6cm}
        \caption{VGG-19.}
        \label{fig:accuracy_size_vgg19}
    \end{subfigure}
    \vspace{-0.3cm}
    \captionsetup{font=small,labelfont=bf}
    \captionof{figure}{Maestro vs. baselines on CIFAR10. Spectral-Init results is taken from the original work; For XNOR-Net each weight is quantized from 32 to 1-bit. Thus, we report a compression rate of $3.125\%$; 
    Detailed results are presented in table form in the Appendix~\ref{app:detailed_baselines}.}
    \label{fig:cifar10_baselines}
    \vspace{-0.4cm}
\end{minipage}
\begin{minipage}{0.35\linewidth}
        \begin{minipage}[t]{\linewidth}
        \centering
        \captionsetup{font=small,labelfont=bf}
        \caption{Datasets and models for evaluation. The footprints depict the vanilla models.}
        \vspace{-0.2cm}
        \begin{adjustbox}{width=1.1\linewidth}
        {
            \begin{tabular}{llrrl}
                \toprule
               \rowcolor{Gray} Dataset     & Model                          & \# GMACs      & \# Params (M)     & Task \\ \midrule
                \textbf{MNIST}       & LeNet       & $2e^{-4}$              & $0.04$                   & Image classification \\
                \textbf{CIFAR10}     & ResNet-18              & $0.56$           & $11.18$               & Image classification \\
                \textbf{CIFAR10}     & VGG-19              & $0.40$           & $20.00$               & Image classification \\
                \textbf{\red{ImageNet}}    & ResNet-50       & $4.12$          & $25.56$               & Image classification      \\
                \textbf{Multi30k }      & 6-layer Transformer & $1.37$              & $48.98$               & Translation (en-ge) \\
                \bottomrule
            \end{tabular}
        }
        \end{adjustbox}
        \label{tab:datasets-models}
    \end{minipage}
    \begin{minipage}[b]{\linewidth}
    \centering
    \captionsetup{font=small,labelfont=bf}
    \vspace{0.2cm}
    \caption{Maestro vs. baselines on Multi30k.}
    \vspace{-0.2cm}
    \label{tab:transformer_baselines}
    \scalebox{0.6}{
        \begin{tabular}{lllll}
            \toprule
            \rowcolor{Gray} Variant                         & Model       & Perplexity         & GMACs & Params. ($M$) \\ 
            \midrule
            Non-factorized                  & Transformer & 9.85{\tiny$\pm0.10$} & 1.370 & 53.90 \\
            Pufferfish$^*$                  & Transformer & 7.34{\tiny$\pm0.12$} & 0.996 & 26.70\\
            \tool{}$^\dagger$ & Transformer & 
            \textbf{6.90{\tiny$\pm0.07$}} & \textbf{0.248} & \textbf{13.80} \\
            \bottomrule
            \multicolumn{5}{l}{{$^*$Results from original work; $^\dagger$ tuned $\lambda_{gp}$ from $\{2^{i}/100; i \in {0, \hdots, 9}\}$}} 
        \end{tabular}}
        \end{minipage}
\end{minipage}

\end{table*}

In this section, we further investigate the theoretical properties of \tool for the linear mappings, i.e., the setup of the problem formulation~\eqref{eq:low_rank_optim}. 
\begin{theorem}[Informal]
\label{thm:linear_tool_is_pca}
Let $A = \tilde{U} \tilde{\Sigma} \tilde{V}^\top$ be a SVD decomposition of $A$. Then, the minimization problem \eqref{eq:low_rank_optim} is equivalent to PCA applied to the transformed dataset $x \rightarrow \tilde{\Sigma} \tilde{V}^\top x$, $x \sim \cX$ projected on the column space of $\tilde{U}$. 
\end{theorem}
The formal statement can be found in Appendix~\ref{app:theory}. Theorem~\ref{thm:linear_tool_is_pca} shows that \tool can adapt to data distribution by directly operating on data $x \sim \cX$ and also to the target mapping by projecting data to its right singular vectors scaled by singular values. In particular, we show that in the special case, when $\cX$ is the uniform distribution on the unit ball, \eqref{eq:low_rank_optim}, i.e., \tool, exactly recovers truncated SVD of $A$, which is consistent with the prior results~\cite{horvath2021fjord}. In the case $A$ is the identity, it is straightforward to see that \tool is equivalent to PCA. We can see that \tool can efficiently extract low-rank solutions by filtering out directions corresponding to the null space of the target mapping $A$ and directions with no data.
We also numerically verify both of the special cases--PCA and SVD, by minimizing \eqref{eq:low_rank_optim} using stochastic gradient descent (SGD) with $\cD$ being the uniform distribution. These experiments are provided in Fig.~\ref{fig:svd} and \ref{fig:pca}. \camready{We provide further evidence on the adaptivity of \tool in Appendix~\ref{sec:correct_ordering} and~\ref{sec:rank_adaptivity}.}

We showed that \tool could recover SVD in a particular case of the linear model and the uniform data distribution on the unit ball. We note that in this case, SVD is optimal, and we cannot acquire better decomposition. Therefore, it is desired that \tool is equivalent to SVD in this scenario. More generally, we argue that \tool decomposition should be preferable to SVD due to the following reasons:
\begin{itemize}[leftmargin=1em,noitemsep,topsep=-1pt]
    \item \tool formulation is \emph{directly built into the training} and tailored to obtain the best low-rank decomposition, while SVD relies on linearity assumption.
    \item SVD does not account for data, and even in the linear NN case, the learned singular vectors might exhibit wrong ordering. We demonstrate this issue using a simple example where we take matrix $A$ with rank $3$. We construct the dataset $\cX$ in such a way that the third singular vector is the most important, the second one is the second, and the first is the third most important direction. Clearly, SVD does not look at data. Therefore, it cannot capture this phenomenon. We showcase that \tool \emph{learns the correct order}; see Fig.~\ref{fig:svd_wrong} of the Appendix.
    \item Pre-factorizing models allow us to apply \emph{hierarchical group-lasso penalty}~\citep{yuan2006model} for decomposed weights to directly regularize the rank of different layers.
    \item SVD is computationally expensive and can only run rarely, while \tool is \emph{directly built into the training} and, therefore, \emph{does not require extra computations.} In addition, \tool supports rank sampling so training can be made computationally efficient. 
\end{itemize}

\section{Experiments}
We start this section by describing the setup of our experiments, including the models, datasets and baselines with which we compare \tool. We then compare \tool against the baselines on accuracy and \rebb{training Multiply-Accumulate operations (MACs)} and discuss the results.
Subsequently, we analyze the behaviour of our system in-depth and provide additional insights on the performance of our technique, along with an ablation study and sensitivity analysis to specific hyperparameters. Finally, we showcase the performance of models upon deployment and how we can derive a smaller footprint model with some accuracy trade-off, without the need to fine-tune.

\vspace{-0.1cm}
\subsection{Experimental Setup}

\noindent\textbf{Models \& datasets.}
The datasets and models considered in our experiments span across four datasets, concisely presented along with the associated models on Tab.~\ref{tab:datasets-models}. We have implemented our solution in PyTorch~\citep{paszke2017automatic}(v1.13.0) trained our models 
on NVidia A100 (40G) GPUs. Details for the learning tasks and hyperparameters used are presented in Appendix~\ref{sec:app_experimental_setup}.

\noindent\textbf{Baselines.}
We have selected various baselines from the literature that we believe are closest to aspects of our system. On the \textit{pruning} front, we compare with the {IMP}~\citep{paul2023unmasking} and {RareGems}~\citep{sreenivasan2022rare} techniques.
On the \textit{quantization} front, we compare with {XNOR-Net}~\citep{rastegari2016xnor}. With respect to \textit{low-rank} methods, we compare with {Spectral Initialisation}~\cite{khodak2021initialization}, {Pufferfish}~\citep{wang2021pufferfish} and {Cuttlefish}~\citep{wang2023cuttlefish}. 

\begin{figure*}[t]
  \centering
  \hfill
  \begin{subfigure}[t]{0.32\linewidth}
    \includegraphics[width=\textwidth]{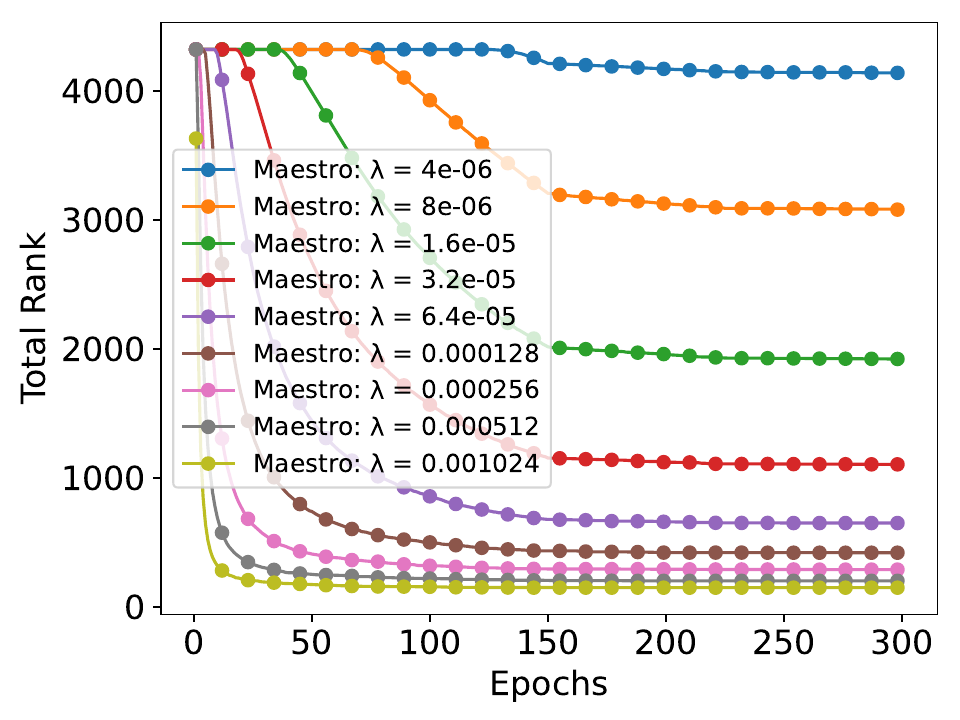}
    \captionsetup{font=small,labelfont=bf}
    \caption{Total rank ($\sum_{i=1}^d r_i$).}
    \label{fig:training_dynamics_a}
  \end{subfigure}
  \hfill
  \begin{subfigure}[t]{0.32\textwidth}
    \includegraphics[width=\textwidth]{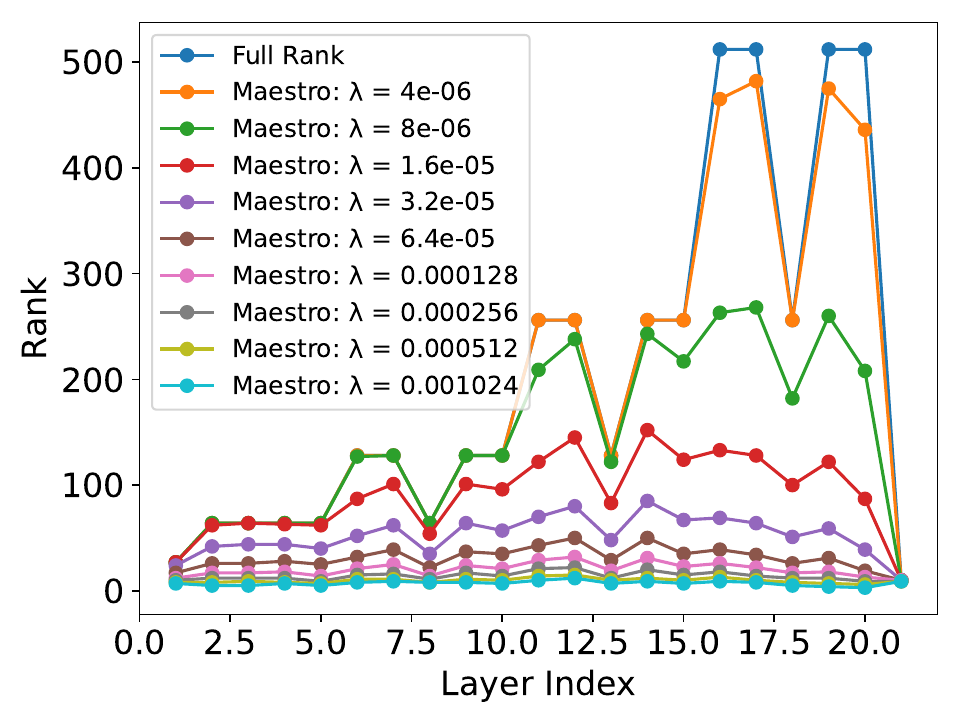}
    \vspace{-0.6cm}
    \captionsetup{font=small,labelfont=bf}
    \caption{Ranks $r_i$'s after training.}
    \label{fig:training_dynamics_b}
  \end{subfigure}
  \hfill
  \begin{subfigure}[t]{0.32\textwidth}
    \includegraphics[width=\textwidth]{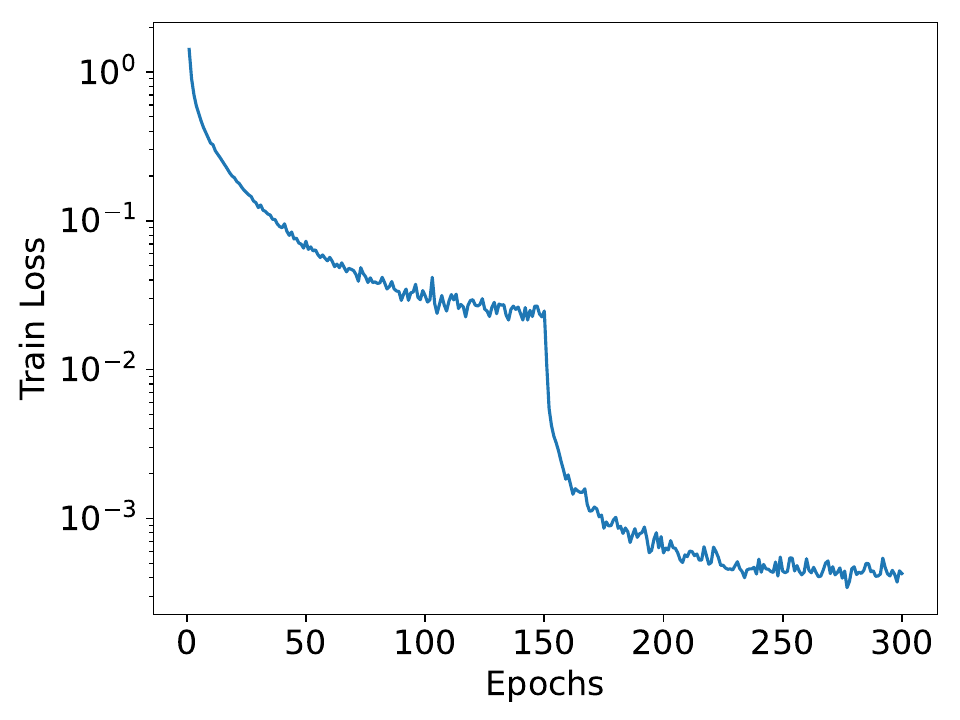}
    \captionsetup{font=small,labelfont=bf}
    \caption{Convergence for $\lambda_{gl} = 0$.}
    \label{fig:training_dynamics_c}
  \end{subfigure}
  \captionsetup{font=small,labelfont=bf}
  \caption{Training dynamics of \tool for ResNet18 on CIFAR10. \camready{Results for other datasets can be found in Appendix~\ref{app:tool_behaviour}.}}
  \vspace{-0.2cm}
  \label{fig:training_dynamics}
\end{figure*}

\vspace{-0.1cm}
\subsection{Performance Comparison}
\label{sec:performance_comparison}

We start off by comparing \tool with the \red{mentioned baselines from the literature across the datasets and models of Tab.~\ref{tab:datasets-models}}\footnote{{The operating points we select for \tool are the closest lower to the respective baseline in terms of footprint. Where the result is not present in the Fig.~\ref{fig:cifar10_baselines}, we provide the $\lambda_{gp}$ value so that it can be referenced from the Appendix, Tab.~\ref{tab:resnet_gp_lambda}, \ref{tab:vgg_gp_lambda}.}}. Results are depicted in Fig.~\ref{fig:cifar10_baselines} and Tab.~\ref{tab:transformer_baselines}, while additional performance points of \tool for different model footprints are presented in the Appendix {\ref{app:tool_behaviour} and \ref{app:tradeoff}.}

\noindent\textbf{Comparisons with low-rank methods.}
The low-rank methods we are comparing against are Pufferfish~\citep{wang2021pufferfish} and Cuttlefish~\citep{wang2023cuttlefish}. These methods try to reduce training and inference runtime while preserving model accuracy by leveraging low-rank approximations. 
For ResNet-18, we achieve 94.19$\pm$0.07\% for 4.08M parameters and 93.97$\pm$0.25\% for 2.19M parameters compared to the 94.17\% of Pufferfish at 3.3M parameters. For VGG-19, we achieve +0.41pp (percentage points) higher accuracy compared to Pufferfish and -0.29pp to Cuttlefish at 44.8\% and 93.2\% of the sizes, respectively. Finally, comparing with the spectral initialization~\citep{khodak2021initialization} for VGG-19, we achieve +5.26pp higher accuracy for 87.5\% of parameter size. Detailed results are shown in Tab.~\ref{tab:cifar10_baselines}. This performance benefits also apply in the case of Transformers (Tab.~\ref{tab:transformer_baselines}), where \tool performs 6\% better in terms of perplexity at 25\% of the cost (MACs) and 51.7\% of the size (parameters) compared to Pufferfish. \red{It is worth noting that both Pufferfish and Cuttlefish, by default, do not decompose all layers and have warm-up full-training rounds, both of which cause training \camready{and hyperparameter optimization} overheads. \camready{In contrast, our technique only introduces two hyperarameters, namely $\lambda_{gl}$ and $\epsilon_{ps}$, which govern the whole training process.}} \camready{We have scaled up our experiments to ImageNet-1k levels (Tab.~\ref{tab:imagenet_baselines}) and for the same setup of full decomposition, we achieve slightly higher accuracy (+0.51pp) at 97.8\% of the size of Pufferfish. For partial decomposition, \tool performs on par with Pufferfish and Cuttlefish at a lower training and inference cost.}

\begin{table}
\centering
\captionsetup{font=small,labelfont=bf}
\vspace{-0.3cm}
\caption{\camready{Maestro vs. baselines on ImageNet-1k.}}
\vspace{-0.3cm}
\label{tab:imagenet_baselines}
\scalebox{0.8}{
    \begin{tabular}{lllll}
        \toprule
        \rowcolor{Gray} Variant         & Model       & Acc. (\%)         & Params. ($M$) & GMACs \\
        \midrule
        \multicolumn{5}{l}{\textbf{No decomposition}} \\
        Non-factorized & ResNet-50 & 75.32 & 25.26 & 4.12 \\
        \midrule
        \multicolumn{5}{l}{\textbf{Not decomposing first four blocks and last layer}} \\
        Pufferfish$^\dagger$     & ResNet-50 & 75.99 & 15.2  & 3.6  \\
        Cuttlefish$^\dagger$     & ResNet-50 & 76.00 & 14.9  & 3.6  \\ 
        \tool$^*$          & ResNet-50 & \textbf{76.04} & \textbf{14.0}  & \textbf{3.4}  \\
        \midrule
        \multicolumn{5}{l}{\textbf{Decomposing all layers}} \\
        Pufferfish$^\dagger$     & ResNet-50 & 71.03 & 9.4   & 2.1  \\
        \tool$^*$          & ResNet-50 & \textbf{71.54} & \textbf{9.2}   & \textbf{2.0}  \\
        \bottomrule
        \multicolumn{5}{p{9.8cm}}{$^*$$\lambda_{gl}$ chosen such that the final number of parameters and accuracy is similar to the baseline models; $^\dagger$ without label smoothing (same as our setup for Maestro)}. \\
    \end{tabular}}
    \vspace{-0.5cm}
\end{table}

\noindent\textbf{Comparisons with pruning methods.}
The next family of baselines is related to the LTH~\citep{frankle2018the}. Specifically, we compare against IMP~\citep{paul2023unmasking} and witness that \tool can achieve +1.25pp ($\lambda_{gp} = 128e^{-6}$) and +0.24pp ($\lambda_{gp} = 32e^{-6}$) higher accuracy for ResNet-18 and VGG-19 respectively. \red{The detailed results are shown in Tab.~\ref{tab:cifar10_baselines} of the Appendix.}
Although we cannot scale to the size that RareGems~\citep{sreenivasan2022rare} for ResNet-18, the sparsity that they achieve is unstructured, which most modern hardware cannot take advantage of. In contrast, our technique performs ordered structured sparsity compatibly with most computation targets. On the other hand, for VGG-19, we achieve +6.82pp higher accuracy at 43.6\% of the footprint.

\noindent\textbf{Comparisons with quantized models.}
We also compare against XNOR-Net~\citep{rastegari2016xnor}, which binarizes the network to achieve efficient inference. Training continues to happen in full precision, and inference performance is dependent on the operation implementation of the target hardware. Nonetheless, assuming a compression rate of 3.125\%, for the same model size, we achieve +1.08pp ($\lambda_{gp} = 512e^{-6}$) and +2.18pp ($\lambda_{gp} = 256e^{-6}$) higher accuracy on ResNet-18 and VGG-19.

\begin{table}[t]
    \centering
    \captionsetup{font=small,labelfont=bf}
    \vspace{-0.3cm}
    \caption{Ablation study for ResNet18 on CIFAR10}
    \vspace{-0.3cm}
    \label{tab:ablation}
    \setlength{\tabcolsep}{2pt}
    \scalebox{0.85}{
        \begin{tabular}{llccc}
            \toprule
            \rowcolor{Gray} Variant                   & Acc. (\%)         & Rel. GMACs (Train.) & Params. ($M$) \\ \midrule
            \textbf{\tool}            & \textbf{94.19{\tiny$\pm$0.39}}        &  \textbf{1.00$\times$}      &   \textbf{4.08{\tiny$\pm$0.020}}           \\
            \textbf{w/out GL}         & 94.04{\tiny$\pm$0.10}        &  1.33$\times$      &   11.2{\tiny$\pm$0.000}           \\
            \textbf{w/out PS}         & 94.12{\tiny$\pm$0.36}        &  1.33$\times$      &   4.09{\tiny$\pm$0.027}           \\
            \textbf{w/ full-training} & 94.05{\tiny$\pm$0.32}        &  1.97$\times$      &   4.09{\tiny$\pm$0.032}           \\
            \bottomrule
        \end{tabular}
    }
    \vspace{-0.4cm}
\end{table}

\begin{figure*}[t]
  \vspace{-0.1cm}
  \centering
  \begin{subfigure}[t]{0.3\textwidth}
  \centering
    \includegraphics[width=\textwidth]{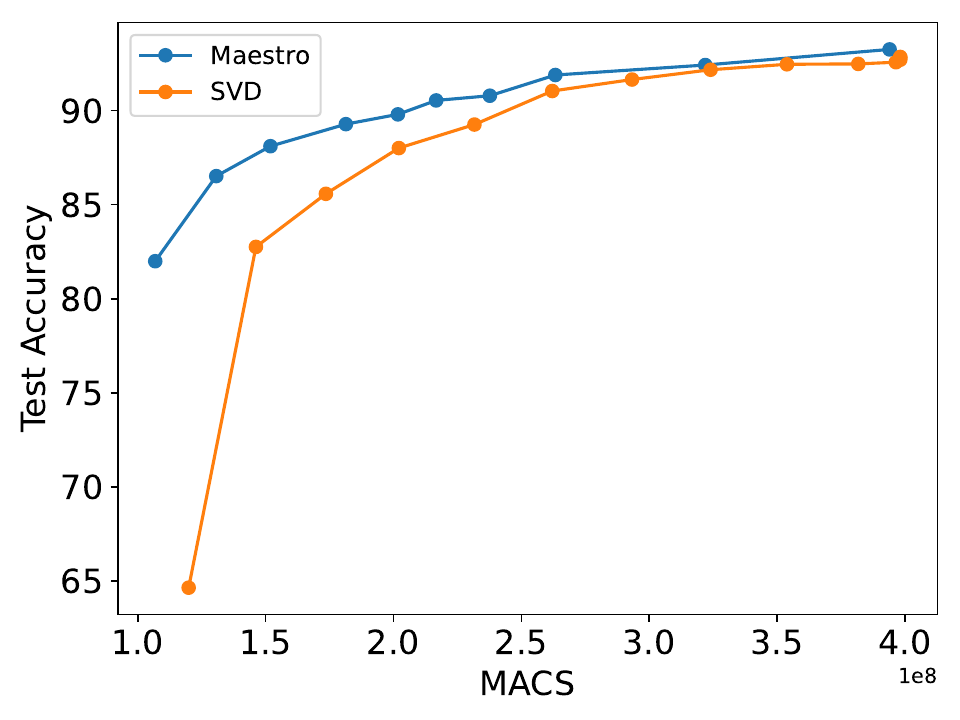}
    \vspace{-0.6cm}
    \captionsetup{font=small,labelfont=bf}
    \caption{\tool vs. SVD.}
    \label{fig:acc_latency_trade_off_a}
  \end{subfigure}
  \hfill
  \begin{subfigure}[t]{0.3\textwidth}
    \includegraphics[width=\textwidth]
    {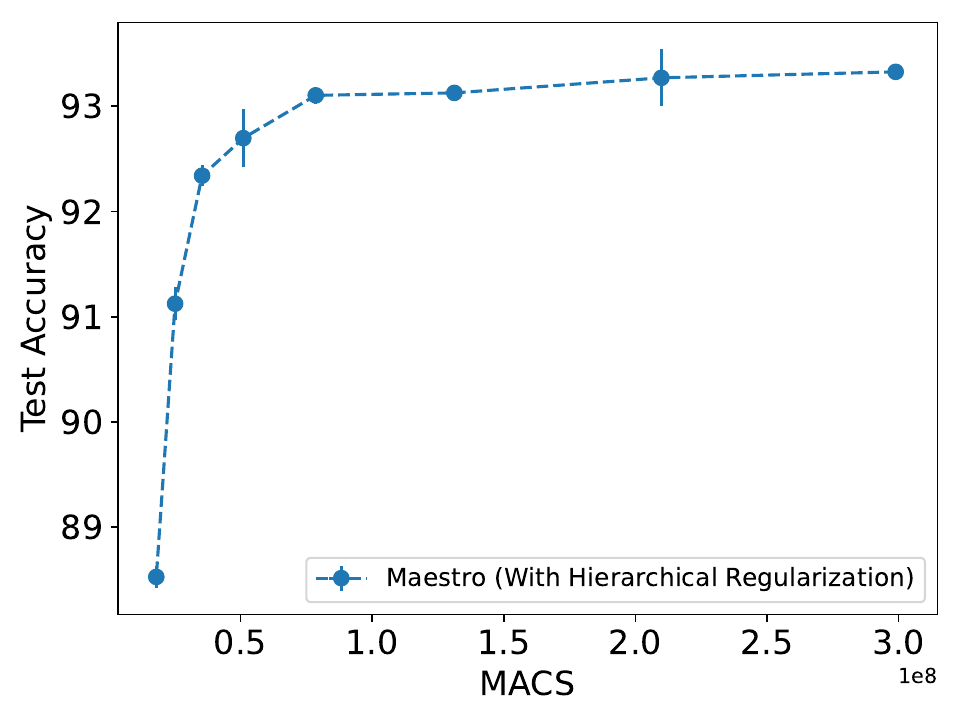}
    \vspace{-0.6cm}
    \captionsetup{font=small,labelfont=bf}
    \caption{Varying HGL.}
    \label{fig:acc_latency_trade_off_b}
  \end{subfigure}
  \hfill
  \begin{subfigure}[t]{0.3\textwidth}
  \centering
    \includegraphics[width=\textwidth]{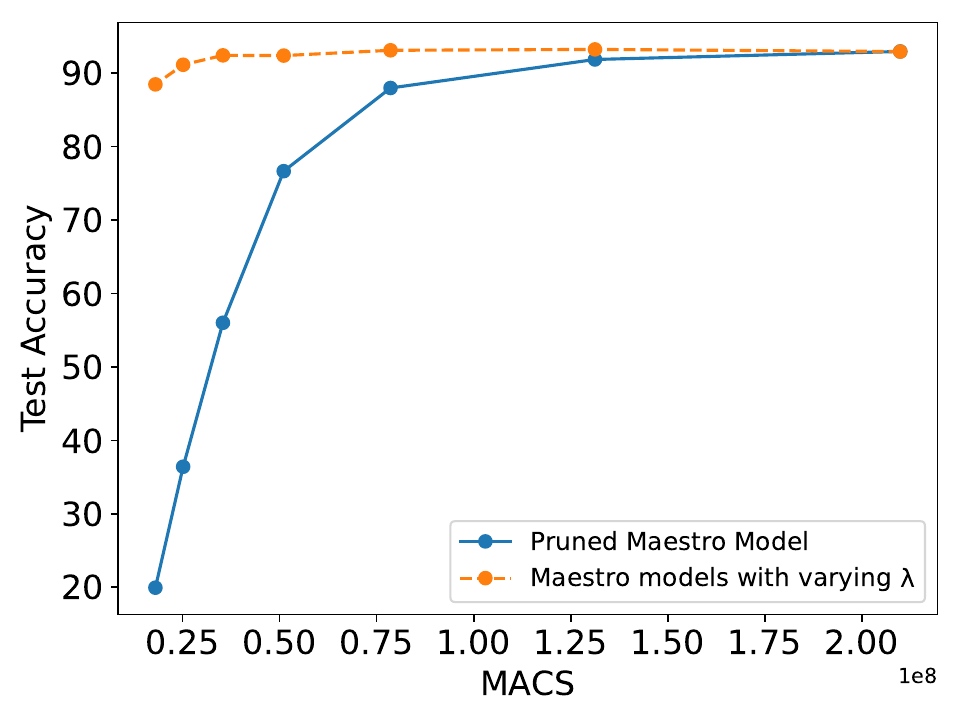}
    \vspace{-0.6cm}
    \captionsetup{font=small,labelfont=bf}
    \caption{Nested \tool.}
    \label{fig:acc_latency_trade_off_c}
  \end{subfigure}
  \captionsetup{font=small,labelfont=bf}
  \vspace{-0.1cm}
  \caption{\tool accuracy-latency trade-off under different settings for VGG19 on CIFAR10. \camready{Additional results in Appendix~\ref{app:tradeoff}.}}
  \vspace{-0.5cm}
  \label{fig:acc_latency_trade_off}
\end{figure*}

\vspace{-0.4cm}
\subsection{Training Behaviour of \tool}
\vspace{-0.2cm}
\label{sec:tool_behaviour}

Having shown the relative performance of our framework to selected baselines, we now move to investigate how our method behaves with respect to its convergence and low-rank approximations. 

\noindent\textbf{Model and rank convergence.} In Fig.~\ref{fig:training_dynamics}, we present the training dynamics for \tool. Fig.~\ref{fig:training_dynamics_a} illustrates the evolution of total rank throughout the training steps. We observe that the ranks are pruned incrementally. This aligns with the observations made during Pufferfish~\cite{wang2021pufferfish} training, where the authors suggest warm-start training with full precision to enhance the final model performance. In our situation, we do not need to integrate this heuristic because \tool automatically prunes rank. Fig.~\ref{fig:training_dynamics_b} reveals the ranks across layers after training. We notice an intriguing phenomenon: the ranks are nested for increasing $\lambda_{gl}$. This could imply apart from a natural order of ranks within each layer, a global order. We briefly examine this captivating occurrence in the following section, and we plan to investigate it more thoroughly in future work, as we believe this might contribute to a superior rank selection and sampling process. Lastly, Fig.~\ref{fig:training_dynamics_c} depicts the progression of training loss. We find that our hypothesis that sampling does not adversely impact training is also supported empirically.

\vspace{-2mm}
\subsection{Ablation Study}
\label{sec:ablation}
\vspace{-1mm}
In this section, we examine the impact of each component on the performance of \tool. Specifically, we run variants of our method \textit{i)}~without the \textit{hierarchical group lasso regularization (HGL)}, \textit{ii)}~without progressive\textit{ shrinking (PS)}. Additionally, we integrate \textit{iii)}~an \textit{extra full low-rank pass} ($b=r_i$) into the training at each step to assess whether extra sampling would be beneficial. The results are displayed in Tab.~\ref{tab:ablation}. As anticipated, our findings confirm that neither the inclusion of hierarchical group lasso with a tuned $\lambda_{gl}$ nor progressive shrinking impair the final performance, but they do significantly enhance the efficiency of \tool. Moreover, sampling more ranks at each training step does not improve the final performance, and, in fact, it hampers training efficiency, making it approximately twice as computationally demanding.

\subsection{Accuracy-Latency Trade-Off at Training and Deployment Time}
\label{sec:tradeoff}
\vspace{-1mm}
In Fig.~\ref{fig:acc_latency_trade_off}, we illustrate various approaches to balance latency (proxied through MACs operations) and accuracy in model training and deployment. Fig.~\ref{fig:acc_latency_trade_off_a} demonstrates how \tool ($\lambda_{gl}=0$) can be pruned effectively for deployment using the greedy search method discussed in Section~\ref{sec:deploy}. We contrast this with the greedy pruning of a non-factorized model that has been factorized using SVD. We reveal that this straightforward baseline does not measure up to the learned decomposition of \tool and results in a significant performance decrease.
Next, Fig.~\ref{fig:acc_latency_trade_off_b} portrays the final accuracy and the number of model parameters for varying hierarchical group lasso penalties. This leads to the optimal latency-accuracy balance for both training and inference. However, it is crucial to point out that each model was trained individually, while greedy pruning only necessitates a single training cycle.
Lastly, we delve into the observation of nested ranks across increasing $\lambda_{gl}$. Fig.~\ref{fig:acc_latency_trade_off_c} displays the performance of \tool ($\lambda_{gl} = 0$) across different ranks selected by smaller models \tool ($\lambda_{gl} > 0$). Intriguingly, we observe that \tool ($\lambda_{gl} = 0$) performs very well—for instance, we can decrease its operations in half (and parameters by 10$\times$) and still maintain an accuracy of $87.7\%$ without fine-tuning, just by reusing rank structure from independent runs. As aforementioned, we intend to further explore this in the future.

\vspace{-0.3cm}
\section{Conclusion and Future Work}
In this work, we have presented \tool, a method for trainable low-rank approximation of DNNs that leverages progressive shrinking by applying a generalized variant of Ordered Dropout to the factorized weights. We have shown the theoretical guarantees of our work in the case of linear models and empirically demonstrated its performance across different types of models, datasets, and modalities. Our evaluation has demonstrated that \tool outperforms competitive compression methods at a lower cost. In the future, we plan to expand our technique to encompass more advanced sampling techniques and apply it to different distributed learning scenarios, such as Federated Learning, where data are natively non-independent or identically distributed (non-IID).

\vspace{-0.3cm}
\section*{Impact Statement}
\vspace{-0.1cm}
The goal of our work is to make the training and deployment of DNNs more efficient, affecting the total computation, memory and bandwidth of systems, as well as the energy they require to run the respective tasks. DNN model training requires significant amounts of energy, whether in a data center or at the edge~\cite{wu2022sustainable,patterson2022carbon}. However, such techniques should not be used in lieu of making data centers less green, but as a complementary measure to further reduce the carbon footprint of \mbox{Deep Learning.}

Additionally, as our technique involves a training-aware methodology for progressively selecting ranks, it depends on the quality of data used in training. Deploying the model in the wild for various downstream tasks may result in behavior different from the intended outcomes. Therefore, it should be thoroughly tested before deployment to ensure it adheres to the required Service Level Objectives (SLOs), especially in performance-critical use cases, such as self-driving vehicles or UAV navigation.

{
\small
\bibliography{literature}

\begin{thebibliography}{70}
\providecommand{\natexlab}[1]{#1}
\providecommand{\url}[1]{\texttt{#1}}
\expandafter\ifx\csname urlstyle\endcsname\relax
  \providecommand{\doi}[1]{doi: #1}\else
  \providecommand{\doi}{doi: \begingroup \urlstyle{rm}\Url}\fi

\bibitem[Alam et~al.(2022)Alam, Liu, Yan, and Zhang]{alam2022fedrolex}
Alam, S., Liu, L., Yan, M., and Zhang, M.
\newblock Fedrolex: Model-heterogeneous federated learning with rolling sub-model extraction.
\newblock In Oh, A.~H., Agarwal, A., Belgrave, D., and Cho, K. (eds.), \emph{Advances in Neural Information Processing Systems}, 2022.
\newblock URL \url{https://openreview.net/forum?id=OtxyysUdBE}.

\bibitem[Alistarh et~al.(2017)Alistarh, Grubic, Li, Tomioka, and Vojnovic]{alistarh2017qsgd}
Alistarh, D., Grubic, D., Li, J., Tomioka, R., and Vojnovic, M.
\newblock Qsgd: Communication-efficient sgd via gradient quantization and encoding.
\newblock \emph{Advances in neural information processing systems}, 30, 2017.

\bibitem[Alistarh et~al.(2018)Alistarh, Hoefler, Johansson, Khirirat, Konstantinov, and Renggli]{alistarh2018convergence}
Alistarh, D., Hoefler, T., Johansson, M., Khirirat, S., Konstantinov, N., and Renggli, C.
\newblock The convergence of sparsified gradient methods.
\newblock \emph{arXiv preprint arXiv:1809.10505}, 2018.

\bibitem[Almeida et~al.(2021)Almeida, Laskaridis, Mehrotra, Dudziak, Leontiadis, and Lane]{almeida2021smart}
Almeida, M., Laskaridis, S., Mehrotra, A., Dudziak, L., Leontiadis, I., and Lane, N.~D.
\newblock Smart at what cost? characterising mobile deep neural networks in the wild.
\newblock In \emph{Proceedings of the 21st ACM Internet Measurement Conference}, pp.\  658--672, 2021.

\bibitem[Caldas et~al.(2019)Caldas, Konečný, McMahan, and Talwalkar]{caldas2019expanding}
Caldas, S., Konečný, J., McMahan, B., and Talwalkar, A.
\newblock Expanding the reach of federated learning by reducing client resource requirements, 2019.
\newblock URL \url{https://openreview.net/forum?id=SJlpM3RqKQ}.

\bibitem[Carreira-Perpin{\'a}n \& Idelbayev(2018)Carreira-Perpin{\'a}n and Idelbayev]{carreira2018learning}
Carreira-Perpin{\'a}n, M.~A. and Idelbayev, Y.
\newblock “learning-compression” algorithms for neural net pruning.
\newblock In \emph{Proceedings of the IEEE Conference on Computer Vision and Pattern Recognition}, pp.\  8532--8541, 2018.

\bibitem[Chen et~al.(2021)Chen, Ji, Ding, Fang, Wang, Zhu, Liang, Shi, Yi, and Tu]{chen2021only}
Chen, T., Ji, B., Ding, T., Fang, B., Wang, G., Zhu, Z., Liang, L., Shi, Y., Yi, S., and Tu, X.
\newblock Only train once: A one-shot neural network training and pruning framework.
\newblock \emph{Advances in Neural Information Processing Systems}, 34:\penalty0 19637--19651, 2021.

\bibitem[Deng et~al.(2009)Deng, Dong, Socher, Li, Li, and Fei-Fei]{deng2009imagenet}
Deng, J., Dong, W., Socher, R., Li, L.-J., Li, K., and Fei-Fei, L.
\newblock Imagenet: A large-scale hierarchical image database.
\newblock In \emph{2009 IEEE conference on computer vision and pattern recognition}, pp.\  248--255. Ieee, 2009.

\bibitem[Diao et~al.(2021)Diao, Ding, and Tarokh]{diao2021heterofl}
Diao, E., Ding, J., and Tarokh, V.
\newblock Hetero{\{}fl{\}}: Computation and communication efficient federated learning for heterogeneous clients.
\newblock In \emph{International Conference on Learning Representations}, 2021.
\newblock URL \url{https://openreview.net/forum?id=TNkPBBYFkXg}.

\bibitem[Dudziak et~al.(2019)Dudziak, Abdelfattah, Vipperla, Laskaridis, and Lane]{shrinkml2019}
Dudziak, {\L}., Abdelfattah, M.~S., Vipperla, R., Laskaridis, S., and Lane, N.~D.
\newblock Shrinkml: End-to-end asr model compression using reinforcement learning.
\newblock \emph{INTERSPEECH}, 2019.

\bibitem[{Elliott} et~al.(2016){Elliott}, {Frank}, {Sima'an}, and {Specia}]{elliott-EtAl:2016:VL16}
{Elliott}, D., {Frank}, S., {Sima'an}, K., and {Specia}, L.
\newblock Multi30k: Multilingual english-german image descriptions.
\newblock pp.\  70--74, 2016.

\bibitem[Frankle \& Carbin(2019)Frankle and Carbin]{frankle2018the}
Frankle, J. and Carbin, M.
\newblock The lottery ticket hypothesis: Finding sparse, trainable neural networks.
\newblock In \emph{International Conference on Learning Representations}, 2019.
\newblock URL \url{https://openreview.net/forum?id=rJl-b3RcF7}.

\bibitem[Han et~al.(2015)Han, Mao, and Dally]{han2015deep}
Han, S., Mao, H., and Dally, W.~J.
\newblock Deep compression: Compressing deep neural networks with pruning, trained quantization and huffman coding.
\newblock \emph{arXiv preprint arXiv:1510.00149}, 2015.

\bibitem[He et~al.(2015)He, Zhang, Ren, and Sun]{he2015delving}
He, K., Zhang, X., Ren, S., and Sun, J.
\newblock Delving deep into rectifiers: Surpassing human-level performance on imagenet classification.
\newblock In \emph{Proceedings of the IEEE international conference on computer vision}, pp.\  1026--1034, 2015.

\bibitem[He et~al.(2016)He, Zhang, Ren, and Sun]{he2016deep}
He, K., Zhang, X., Ren, S., and Sun, J.
\newblock Deep residual learning for image recognition.
\newblock In \emph{Proceedings of the IEEE conference on computer vision and pattern recognition}, pp.\  770--778, 2016.

\bibitem[He et~al.(2017)He, Zhang, and Sun]{he2017channel}
He, Y., Zhang, X., and Sun, J.
\newblock Channel pruning for accelerating very deep neural networks.
\newblock In \emph{Proceedings of the IEEE international conference on computer vision}, pp.\  1389--1397, 2017.

\bibitem[Horv{\'a}th et~al.(2021)Horv{\'a}th, Klein, Richt{\'a}rik, and Archambeau]{horvath2021hyperparameter}
Horv{\'a}th, S., Klein, A., Richt{\'a}rik, P., and Archambeau, C.
\newblock Hyperparameter transfer learning with adaptive complexity.
\newblock In \emph{International Conference on Artificial Intelligence and Statistics}, pp.\  1378--1386. PMLR, 2021.

\bibitem[Horv\'{a}th et~al.(2021)Horv\'{a}th, Laskaridis, Almeida, Leontiadis, Venieris, and Lane]{horvath2021fjord}
Horv\'{a}th, S., Laskaridis, S., Almeida, M., Leontiadis, I., Venieris, S., and Lane, N.
\newblock {FjORD: F}air and accurate federated learning under heterogeneous targets with ordered dropout.
\newblock \emph{Advances in Neural Information Processing Systems}, 34:\penalty0 12876--12889, 2021.

\bibitem[Houlsby et~al.(2019)Houlsby, Giurgiu, Jastrzebski, Morrone, De~Laroussilhe, Gesmundo, Attariyan, and Gelly]{houlsby2019parameter}
Houlsby, N., Giurgiu, A., Jastrzebski, S., Morrone, B., De~Laroussilhe, Q., Gesmundo, A., Attariyan, M., and Gelly, S.
\newblock Parameter-efficient transfer learning for nlp.
\newblock In \emph{International Conference on Machine Learning}, pp.\  2790--2799. PMLR, 2019.

\bibitem[Howard et~al.(2017)Howard, Zhu, Chen, Kalenichenko, Wang, Weyand, Andreetto, and Adam]{howard2017mobilenets}
Howard, A.~G., Zhu, M., Chen, B., Kalenichenko, D., Wang, W., Weyand, T., Andreetto, M., and Adam, H.
\newblock Mobilenets: Efficient convolutional neural networks for mobile vision applications.
\newblock \emph{arXiv preprint arXiv:1704.04861}, 2017.

\bibitem[Hu et~al.(2021)Hu, Shen, Wallis, Allen-Zhu, Li, Wang, and Chen]{hu2021lora}
Hu, E., Shen, Y., Wallis, P., Allen-Zhu, Z., Li, Y., Wang, L., and Chen, W.
\newblock Lora: Low-rank adaptation of large language models, 2021.

\bibitem[Hu et~al.(2016)Hu, Peng, Tai, and Tang]{hu2016network}
Hu, H., Peng, R., Tai, Y.-W., and Tang, C.-K.
\newblock Network trimming: A data-driven neuron pruning approach towards efficient deep architectures.
\newblock \emph{arXiv preprint arXiv:1607.03250}, 2016.

\bibitem[Jaderberg et~al.(2014)Jaderberg, Vedaldi, and Zisserman]{jaderberg2014speeding}
Jaderberg, M., Vedaldi, A., and Zisserman, A.
\newblock Speeding up convolutional neural networks with low rank expansions.
\newblock \emph{arXiv preprint arXiv:1405.3866}, 2014.

\bibitem[Khodak et~al.(2021)Khodak, Tenenholtz, Mackey, and Fusi]{khodak2021initialization}
Khodak, M., Tenenholtz, N., Mackey, L., and Fusi, N.
\newblock Initialization and regularization of factorized neural layers.
\newblock \emph{arXiv preprint arXiv:2105.01029}, 2021.

\bibitem[Kim et~al.(2023)Kim, Yu, Kim, and Moon]{kim2023depthfl}
Kim, M., Yu, S., Kim, S., and Moon, S.-M.
\newblock Depth{FL} : Depthwise federated learning for heterogeneous clients.
\newblock In \emph{The Eleventh International Conference on Learning Representations}, 2023.
\newblock URL \url{https://openreview.net/forum?id=pf8RIZTMU58}.

\bibitem[Krizhevsky et~al.(2009)Krizhevsky, Hinton, et~al.]{krizhevsky2009learning}
Krizhevsky, A., Hinton, G., et~al.
\newblock Learning multiple layers of features from tiny images.
\newblock 2009.

\bibitem[Laskaridis et~al.(2021)Laskaridis, Kouris, and Lane]{laskaridis2021adaptive}
Laskaridis, S., Kouris, A., and Lane, N.~D.
\newblock Adaptive inference through early-exit networks: Design, challenges and directions.
\newblock In \emph{Proceedings of the 5th International Workshop on Embedded and Mobile Deep Learning}, pp.\  1--6, 2021.

\bibitem[Laskaridis et~al.(2022)Laskaridis, Venieris, Kouris, Li, and Lane]{laskaridis2022future}
Laskaridis, S., Venieris, S.~I., Kouris, A., Li, R., and Lane, N.~D.
\newblock The future of consumer edge-ai computing.
\newblock \emph{arXiv preprint arXiv:2210.10514}, 2022.

\bibitem[Laskaridis et~al.(2024)Laskaridis, Kateveas, Minto, and Haddadi]{laskaridis2024melting}
Laskaridis, S., Kateveas, K., Minto, L., and Haddadi, H.
\newblock Melting point: Mobile evaluation of language transformers.
\newblock \emph{arXiv preprint arXiv:2403.12844}, 2024.

\bibitem[LeCun et~al.(2010)LeCun, Cortes, and Burges]{lecun2010mnist}
LeCun, Y., Cortes, C., and Burges, C.
\newblock Mnist handwritten digit database.
\newblock \emph{ATT Labs [Online]. Available: http://yann.lecun.com/exdb/mnist}, 2, 2010.

\bibitem[Li et~al.(2016)Li, Kadav, Durdanovic, Samet, and Graf]{li2016pruning}
Li, H., Kadav, A., Durdanovic, I., Samet, H., and Graf, H.~P.
\newblock Pruning filters for efficient convnets.
\newblock \emph{arXiv preprint arXiv:1608.08710}, 2016.

\bibitem[Lim \& Hastie(2015)Lim and Hastie]{lim2015learning}
Lim, M. and Hastie, T.
\newblock Learning interactions via hierarchical group-lasso regularization.
\newblock \emph{Journal of Computational and Graphical Statistics}, 24\penalty0 (3):\penalty0 627--654, 2015.

\bibitem[Lin et~al.(2022)Lin, Zhu, Chen, Wang, Gan, and Han]{lin2022ondevice}
Lin, J., Zhu, L., Chen, W.-M., Wang, W.-C., Gan, C., and Han, S.
\newblock On-device training under 256kb memory.
\newblock In \emph{Annual Conference on Neural Information Processing Systems (NeurIPS)}, 2022.

\bibitem[Liu et~al.(2018)Liu, Sun, Zhou, Huang, and Darrell]{liu2018rethinking}
Liu, Z., Sun, M., Zhou, T., Huang, G., and Darrell, T.
\newblock Rethinking the value of network pruning.
\newblock \emph{arXiv preprint arXiv:1810.05270}, 2018.

\bibitem[Liu et~al.(2022)Liu, Li, Fernandez-Marques, Laskaridis, Gao, Dudziak, Li, Hu, and Hospedales]{liu2022federated}
Liu, Z., Li, D., Fernandez-Marques, J., Laskaridis, S., Gao, Y., Dudziak, {\L}., Li, S.~Z., Hu, S.~X., and Hospedales, T.
\newblock Federated learning for inference at anytime and anywhere.
\newblock \emph{arXiv preprint arXiv:2212.04084}, 2022.

\bibitem[Ma et~al.(2018)Ma, Bassily, and Belkin]{ma2018power}
Ma, S., Bassily, R., and Belkin, M.
\newblock The power of interpolation: Understanding the effectiveness of sgd in modern over-parametrized learning.
\newblock In \emph{International Conference on Machine Learning}, pp.\  3325--3334. PMLR, 2018.

\bibitem[McMahan et~al.(2017)McMahan, Moore, Ramage, Hampson, and y~Arcas]{mcmahan2017communication}
McMahan, B., Moore, E., Ramage, D., Hampson, S., and y~Arcas, B.~A.
\newblock Communication-efficient learning of deep networks from decentralized data.
\newblock In \emph{Artificial intelligence and statistics}, pp.\  1273--1282. PMLR, 2017.

\bibitem[Mishkin \& Matas(2015)Mishkin and Matas]{mishkin2015all}
Mishkin, D. and Matas, J.
\newblock All you need is a good init.
\newblock \emph{arXiv preprint arXiv:1511.06422}, 2015.

\bibitem[Paszke et~al.(2017)Paszke, Gross, Chintala, Chanan, Yang, DeVito, Lin, Desmaison, Antiga, and Lerer]{paszke2017automatic}
Paszke, A., Gross, S., Chintala, S., Chanan, G., Yang, E., DeVito, Z., Lin, Z., Desmaison, A., Antiga, L., and Lerer, A.
\newblock Automatic differentiation in pytorch.
\newblock 2017.

\bibitem[Patterson et~al.(2022)Patterson, Gonzalez, H{\"o}lzle, Le, Liang, Munguia, Rothchild, So, Texier, and Dean]{patterson2022carbon}
Patterson, D., Gonzalez, J., H{\"o}lzle, U., Le, Q., Liang, C., Munguia, L.-M., Rothchild, D., So, D.~R., Texier, M., and Dean, J.
\newblock The carbon footprint of machine learning training will plateau, then shrink.
\newblock \emph{Computer}, 55\penalty0 (7):\penalty0 18--28, 2022.

\bibitem[Paul et~al.(2023)Paul, Chen, Larsen, Frankle, Ganguli, and Dziugaite]{paul2023unmasking}
Paul, M., Chen, F., Larsen, B.~W., Frankle, J., Ganguli, S., and Dziugaite, G.~K.
\newblock Unmasking the lottery ticket hypothesis: What's encoded in a winning ticket's mask?
\newblock In \emph{The Eleventh International Conference on Learning Representations}, 2023.
\newblock URL \url{https://openreview.net/forum?id=xSsW2Am-ukZ}.

\bibitem[Radford et~al.(2021)Radford, Kim, Hallacy, Ramesh, Goh, Agarwal, Sastry, Askell, Mishkin, Clark, et~al.]{radford2021learning}
Radford, A., Kim, J.~W., Hallacy, C., Ramesh, A., Goh, G., Agarwal, S., Sastry, G., Askell, A., Mishkin, P., Clark, J., et~al.
\newblock Learning transferable visual models from natural language supervision.
\newblock In \emph{International conference on machine learning}, pp.\  8748--8763. PMLR, 2021.

\bibitem[Radford et~al.(2023)Radford, Kim, Xu, Brockman, McLeavey, and Sutskever]{radford2023robust}
Radford, A., Kim, J.~W., Xu, T., Brockman, G., McLeavey, C., and Sutskever, I.
\newblock Robust speech recognition via large-scale weak supervision.
\newblock In \emph{International Conference on Machine Learning}, pp.\  28492--28518. PMLR, 2023.

\bibitem[Rastegari et~al.(2016)Rastegari, Ordonez, Redmon, and Farhadi]{rastegari2016xnor}
Rastegari, M., Ordonez, V., Redmon, J., and Farhadi, A.
\newblock Xnor-net: Imagenet classification using binary convolutional neural networks.
\newblock In \emph{Computer Vision--ECCV 2016: 14th European Conference, Amsterdam, The Netherlands, October 11--14, 2016, Proceedings, Part IV}, pp.\  525--542. Springer, 2016.

\bibitem[Rippel et~al.(2014)Rippel, Gelbart, and Adams]{rippel2014learning}
Rippel, O., Gelbart, M., and Adams, R.
\newblock {Learning Ordered Representations with Nested Dropout}.
\newblock In \emph{International Conference on Machine Learning (ICML)}, pp.\  1746--1754, 2014.

\bibitem[Sainath et~al.(2013)Sainath, Kingsbury, Sindhwani, Arisoy, and Ramabhadran]{sainath2013low}
Sainath, T.~N., Kingsbury, B., Sindhwani, V., Arisoy, E., and Ramabhadran, B.
\newblock Low-rank matrix factorization for deep neural network training with high-dimensional output targets.
\newblock In \emph{2013 IEEE international conference on acoustics, speech and signal processing}, pp.\  6655--6659. IEEE, 2013.

\bibitem[Schmidt \& Roux(2013)Schmidt and Roux]{schmidt2013fast}
Schmidt, M. and Roux, N.~L.
\newblock Fast convergence of stochastic gradient descent under a strong growth condition.
\newblock \emph{arXiv preprint arXiv:1308.6370}, 2013.

\bibitem[Seide et~al.(2014)Seide, Fu, Droppo, Li, and Yu]{seide20141}
Seide, F., Fu, H., Droppo, J., Li, G., and Yu, D.
\newblock 1-bit stochastic gradient descent and its application to data-parallel distributed training of speech dnns.
\newblock In \emph{Fifteenth annual conference of the international speech communication association}, 2014.

\bibitem[Sidahmed et~al.(2021)Sidahmed, Xu, Garg, Cao, and Chen]{sidahmed2021efficient}
Sidahmed, H., Xu, Z., Garg, A., Cao, Y., and Chen, M.
\newblock Efficient and private federated learning with partially trainable networks.
\newblock \emph{arXiv preprint arXiv:2110.03450}, 2021.

\bibitem[Simonyan \& Zisserman(2015)Simonyan and Zisserman]{Simonyan15}
Simonyan, K. and Zisserman, A.
\newblock Very deep convolutional networks for large-scale image recognition.
\newblock In \emph{International Conference on Learning Representations}, 2015.

\bibitem[Sreenivasan et~al.(2022)Sreenivasan, yong Sohn, Yang, Grinde, Nagle, Wang, Xing, Lee, and Papailiopoulos]{sreenivasan2022rare}
Sreenivasan, K., yong Sohn, J., Yang, L., Grinde, M., Nagle, A., Wang, H., Xing, E., Lee, K., and Papailiopoulos, D.
\newblock Rare gems: Finding lottery tickets at initialization.
\newblock In Oh, A.~H., Agarwal, A., Belgrave, D., and Cho, K. (eds.), \emph{Advances in Neural Information Processing Systems}, 2022.
\newblock URL \url{https://openreview.net/forum?id=Jpxd93u2vK-}.

\bibitem[Suresh et~al.(2017)Suresh, Felix, Kumar, and McMahan]{suresh2017distributed}
Suresh, A.~T., Felix, X.~Y., Kumar, S., and McMahan, H.~B.
\newblock Distributed mean estimation with limited communication.
\newblock In \emph{International Conference on Machine Learning}, pp.\  3329--3337. PMLR, 2017.

\bibitem[Tan \& Le(2019)Tan and Le]{tan2019efficientnet}
Tan, M. and Le, Q.
\newblock Efficientnet: Rethinking model scaling for convolutional neural networks.
\newblock In \emph{International conference on machine learning}, pp.\  6105--6114. PMLR, 2019.

\bibitem[Vaswani et~al.(2017)Vaswani, Shazeer, Parmar, Uszkoreit, Jones, Gomez, Kaiser, and Polosukhin]{vaswani2017attention}
Vaswani, A., Shazeer, N., Parmar, N., Uszkoreit, J., Jones, L., Gomez, A.~N., Kaiser, {\L}., and Polosukhin, I.
\newblock Attention is all you need.
\newblock In \emph{Advances in neural information processing systems}, pp.\  5998--6008, 2017.

\bibitem[Wan et~al.(2023)Wan, Wang, Liu, Alam, Zheng, LIU, QU, YAN, ZHU, ZHANG, et~al.]{wan2023efficient}
Wan, Z., Wang, X., Liu, C., Alam, S., Zheng, Y., LIU, J., QU, Z., YAN, S., ZHU, Y., ZHANG, Q., et~al.
\newblock Efficient large language models: A survey.
\newblock \emph{arXiv preprint arXiv:2312.03863}, 2023.

\bibitem[Wang et~al.(2019)Wang, Davis, Zhao, Ng, Niu, Luk, Cheung, and Constantinides]{wang2019deep}
Wang, E., Davis, J.~J., Zhao, R., Ng, H.-C., Niu, X., Luk, W., Cheung, P.~Y., and Constantinides, G.~A.
\newblock {Deep Neural Network Approximation for Custom Hardware: Where we've been, where we're going}.
\newblock \emph{ACM Computing Surveys (CSUR)}, 52\penalty0 (2):\penalty0 1--39, 2019.

\bibitem[Wang et~al.(2021)Wang, Agarwal, and Papailiopoulos]{wang2021pufferfish}
Wang, H., Agarwal, S., and Papailiopoulos, D.
\newblock Pufferfish: communication-efficient models at no extra cost.
\newblock \emph{Proceedings of Machine Learning and Systems}, 3:\penalty0 365--386, 2021.

\bibitem[Wang et~al.(2023)Wang, Agarwal, Tanaka, Xing, Papailiopoulos, et~al.]{wang2023cuttlefish}
Wang, H., Agarwal, S., Tanaka, Y., Xing, E.~P., Papailiopoulos, D., et~al.
\newblock Cuttlefish: Low-rank model training without all the tuning.
\newblock \emph{arXiv preprint arXiv:2305.02538}, 2023.

\bibitem[Wen et~al.(2016)Wen, Wu, Wang, Chen, and Li]{wen2016learning}
Wen, W., Wu, C., Wang, Y., Chen, Y., and Li, H.
\newblock Learning structured sparsity in deep neural networks.
\newblock \emph{Advances in neural information processing systems}, 29, 2016.

\bibitem[Wiesler et~al.(2014)Wiesler, Richard, Schl{\"u}ter, and Ney]{wiesler2014mean}
Wiesler, S., Richard, A., Schl{\"u}ter, R., and Ney, H.
\newblock Mean-normalized stochastic gradient for large-scale deep learning.
\newblock In \emph{2014 IEEE International Conference on Acoustics, Speech and Signal Processing (ICASSP)}, pp.\  180--184. IEEE, 2014.

\bibitem[Wu et~al.(2022)Wu, Raghavendra, Gupta, Acun, Ardalani, Maeng, Chang, Aga, Huang, Bai, et~al.]{wu2022sustainable}
Wu, C.-J., Raghavendra, R., Gupta, U., Acun, B., Ardalani, N., Maeng, K., Chang, G., Aga, F., Huang, J., Bai, C., et~al.
\newblock Sustainable ai: Environmental implications, challenges and opportunities.
\newblock \emph{Proceedings of Machine Learning and Systems}, 4:\penalty0 795--813, 2022.

\bibitem[Wu et~al.(2018)Wu, Nagarajan, Kumar, Rennie, Davis, Grauman, and Feris]{Wu_2018_CVPR}
Wu, Z., Nagarajan, T., Kumar, A., Rennie, S., Davis, L.~S., Grauman, K., and Feris, R.
\newblock Blockdrop: Dynamic inference paths in residual networks.
\newblock In \emph{Proceedings of the IEEE Conference on Computer Vision and Pattern Recognition (CVPR)}, June 2018.

\bibitem[Xue et~al.(2013)Xue, Li, and Gong]{xue2013restructuring}
Xue, J., Li, J., and Gong, Y.
\newblock Restructuring of deep neural network acoustic models with singular value decomposition.
\newblock In \emph{Interspeech}, pp.\  2365--2369, 2013.

\bibitem[Yang et~al.(2017)Yang, Chen, and Sze]{yang2017designing}
Yang, T.-J., Chen, Y.-H., and Sze, V.
\newblock Designing energy-efficient convolutional neural networks using energy-aware pruning.
\newblock In \emph{Proceedings of the IEEE conference on computer vision and pattern recognition}, pp.\  5687--5695, 2017.

\bibitem[Ye \& Du(2021)Ye and Du]{ye2021global}
Ye, T. and Du, S.~S.
\newblock Global convergence of gradient descent for asymmetric low-rank matrix factorization.
\newblock \emph{Advances in Neural Information Processing Systems}, 34:\penalty0 1429--1439, 2021.

\bibitem[Yu \& Huang(2019{\natexlab{a}})Yu and Huang]{yu2019autoslim}
Yu, J. and Huang, T.
\newblock Autoslim: Towards one-shot architecture search for channel numbers.
\newblock \emph{arXiv preprint arXiv:1903.11728}, 2019{\natexlab{a}}.

\bibitem[Yu \& Huang(2019{\natexlab{b}})Yu and Huang]{yu2019universally}
Yu, J. and Huang, T.~S.
\newblock Universally slimmable networks and improved training techniques.
\newblock In \emph{Proceedings of the IEEE/CVF international conference on computer vision}, pp.\  1803--1811, 2019{\natexlab{b}}.

\bibitem[Yu et~al.(2019)Yu, Yang, Xu, Yang, and Huang]{yu2018slimmable}
Yu, J., Yang, L., Xu, N., Yang, J., and Huang, T.
\newblock Slimmable neural networks.
\newblock In \emph{International Conference on Learning Representations}, 2019.
\newblock URL \url{https://openreview.net/forum?id=H1gMCsAqY7}.

\bibitem[Yuan \& Lin(2006)Yuan and Lin]{yuan2006model}
Yuan, M. and Lin, Y.
\newblock Model selection and estimation in regression with grouped variables.
\newblock \emph{Journal of the Royal Statistical Society: Series B (Statistical Methodology)}, 68\penalty0 (1):\penalty0 49--67, 2006.

\bibitem[Zhu \& Gupta(2017)Zhu and Gupta]{zhu2017prune}
Zhu, M. and Gupta, S.
\newblock To prune, or not to prune: exploring the efficacy of pruning for model compression.
\newblock \emph{arXiv preprint arXiv:1710.01878}, 2017.

\end{thebibliography}
\bibliographystyle{icml2024}
}

\newpage
\appendix
\onecolumn

\part*{Appendix}
{\hypersetup{linkcolor=black}
\parskip=0em
\renewcommand{\contentsname}{Contents of the Appendix}
\tableofcontents
\addtocontents{toc}{\protect\setcounter{tocdepth}{3}}
}

\newcommand{\blocka}[2]{\multirow{3}{*}{\(\left[\begin{array}{c}\text{3$\times$3, #1}\\[-.1em] \text{3$\times$3, #1} \end{array}\right]\)$\times$#2}
}
\newcommand{\blockb}[3]{\multirow{3}{*}{\(\left[\begin{array}{c}\text{1$\times$1, #2}\\[-.1em] \text{$3\times$3, #2}\\[-.1em] \text{1$\times$1, #1}\end{array}\right]\)$\times$#3}
}

\section{Limitations}

In this work, we have proposed a method for trainable low-rank approximation of DNNs that provides performance benefits for both training and inference times. While we suggest that this could have repercussions on the energy consumption of these tasks, we have not yet evaluated this hypothesis experimentally across different devices, be they data center-grade or at the edge.

Additionally, we have applied our technique to CNN and Transformer models spanning across vision and NLP tasks. While we anticipate generalization to any type of network, it remains to be seen whether our techniques can also be applied to alternative types of layers, such as recurrent ones, and the benefits they may bring.

{Although we have provided a thorough investigation of the behaviour of our proposed system, the only way we can control the end footprint of the model during training is via the $\lambda_{gl}$ and $\epsilon_{ps}$ hyperparameters. However, there is no guarantee about the final footprint of the model. If we are willing to sacrifise accuracy, then the technique illustrated in Sec.~\ref{sec:deploy} and evaluated in Sec.~\ref{sec:tradeoff} is a start. More robust ways of globally ranking per-layer importances are left as future work.}

Lastly, our sampling method during training is uniform up to the maximum rank during progressive shrinking. Although this method has proven effective, alternative sampling methods could potentially accelerate rank exploration, thereby hastening the shrinking and convergence of the network during training.

{\section{Extended Background}}

\noindent
{\textbf{Ordered Dropout.} Ordered Dropout is a technique of importance-based, nested and ordered pruning that works along the indices of a layer's parameters (neurons, filters, etc.) Introduced by~\citep{horvath2021fjord}, the authors describe a training technique where a layer's width is discretised in $|P|$ values, where $P=\{s_1, s_2, …, s_{|P|}\}$, and at each training step, they sample $p \sim U_P$ to get a specific subnetwork, extracted by selecting the first $\ceil[\big]{p*K_l-1}$ neurons per layer and dropping the rest. In contrast to our work, sampling is happening directly on model parameters (rather than ranks) and is uniform across layers (i.e. a single p-value is set). Nested-ness refers to the fact that larger p-value models include the parameters of lower p-values and importance-based pruning means that via stochastic sampling, the right-most (in terms of index) parameters train on progressively less data due to the probability of sampling and nestedness (i.e. all data pass from the parameters of minimal subnetwork, less pass the higher the p-value).}

\section{Theoretical Properties of Low-Rank Layers}
\label{app:theory}
In this section, we show that for the case of linear mappings, i.e., the problem formulation discussed in \eqref{eq:low_rank_optim}, \tool acts as PCA applied to the original dataset $\cX$ projected onto the space weighted by the corresponding singular values.  Before proceeding with the theorem, we first recall the assumptions and notations introduced in the main paper. 

We denote $C_{:b}$ as the first $b$ columns of matrix $C$, $C_{:a, :b}$ denotes the first $a$ rows, and $b$ columns of a matrix $C$, $a+1:$ denotes the all the columns/rows from index $a+1$, $:$ denotes the all the columns/rows, and for vectors, we use a single subscript. As discussed in the main paper, we reformulate the original least squares problems to the following decomposition problem 
\begin{align}
    \label{eq:low_rank_optim_app}
    \begin{split}
        &\min_{U \in \R^{m \times r}, V \in \R^{n \times r}} \EE{x, y \sim \cX}{\EE{b \sim \cD}{\norm*{U_{:b}V_{:b}^\top x - y}^2}},
    \end{split}
\end{align}
where $\cD$ is a distribution that samples $b \in \cbr{1, 2, \hdots, r}$ with probability $p_b > 0$ and we assume that $y$ is linked with $x$ through linear map $A$, i.e., $y = Ax$. 
\begin{theorem}
\label{thm:linear_tool_is_svd_app}
    Let $A = \tilde{U} \tilde{\Sigma} \tilde{V}^\top$ be a SVD decomposition of $A$. Then, the minimization problem \eqref{eq:low_rank_optim_app} is equivalent to PCA applied to the transformed dataset $x \rightarrow \tilde{\Sigma} \tilde{V}^\top x$, $x \sim \cX$ projected on the column space of $\tilde{U}$. Concretely, we can first solve  
    \begin{align}
    \label{eq:pca_final_full}
        \min_{U \in \R^{m \times r}, V \in \R^{n \times r}} \EE{z \sim \cX}{\EE{b \sim \cD}{\norm*{\rbr*{U_{:b}V_{:b}^\top - I}\tilde{\Sigma}\tilde{V}^\top x}^2}},
    \end{align}
    and then we can obtain the solutions of \eqref{eq:low_rank_optim_app}  using $U^\star = \tilde{U}^\top \Bar{U} , V^\star = \tilde{V}^\top\Bar{V}$, where $\Bar{U}, \Bar{V}$ belong to the set of optimal solutions of problem \eqref{eq:pca_final_full}.zx
    
    In the particular case, where $\cX$ is a uniform distribution on the unit ball,  \eqref{eq:low_rank_optim_app} recovers the best rank approximation of $A$ across all ranks, i.e., up to the scale of $U$ and $V$ recovers truncated SVD. In the case, $A$ is identity, \eqref{eq:low_rank_optim_app} leads to standard PCA decomposition. 
\end{theorem}
\begin{proof}
    From the assumptions that $y = Ax$ and $A = \tilde{U} \tilde{\Sigma} \tilde{V}^\top$, we can rewrite \eqref{eq:low_rank_optim_app} as
    \begin{align*}
        \min_{U \in \R^{m \times r}, V \in \R^{n \times r}} \EE{x \sim \cX}{\EE{b \sim \cD}{\norm*{\rbr*{U_{:b}V_{:b}^\top -  \tilde{U} \tilde{\Sigma} \tilde{V}^\top}x}^2}}.
    \end{align*}
    Since $\tilde{U}$ is orthogonal, we have $\norm{z} = \norm{\tilde{U}^\top z}$. Therefore, the above problem is equivalent to 
    \begin{align*}
        \min_{U \in \R^{m \times r}, V \in \R^{n \times r}} \EE{x \sim \cX}{\EE{b \sim \cD}{\norm*{\rbr*{\tilde{U}^\top U_{:b}V_{:b}^\top - \tilde{\Sigma} \tilde{V}^\top}x}^2}},
    \end{align*}
    which is also equivalent to 
    \begin{align*}
        \min_{U \in \R^{m \times r}, V \in \R^{n \times r}} \EE{x\sim \cX}{\EE{b \sim \cD}{\norm*{\rbr*{U_{:b}V_{:b}^\top - \tilde{\Sigma} \tilde{V}^\top}x}^2}}
    \end{align*}
    after reparametrization. The next step involves injecting identity in the form $\tilde{V} \tilde{V}^\top$ as that leads to the equivalent reformulation 
    \begin{align*}
        \min_{U \in \R^{m \times r}, V \in \R^{n \times r}} \EE{x \sim \cX}{\EE{b \sim \cD}{\norm*{\rbr*{U_{:b}V_{:b}^\top \tilde{V} - \tilde{\Sigma}}\tilde{V}^\top x}^2}}.
    \end{align*}
    As for the previous case, we can reparametrise the problem to obtain 
    \begin{align*}
        \min_{U \in \R^{m \times r}, V \in \R^{n \times r}} \EE{x \sim \cX}{\EE{b \sim \cD}{\norm*{\rbr*{U_{:b}V_{:b}^\top - \tilde{\Sigma}}\tilde{V}^\top x}^2}}.
    \end{align*}
    Let $k = \rank(\tilde{\Sigma}) = \rank(A) \leq r$ and $z = \tilde{V}^\top x$. Furthermore, let $g = \tilde{\Sigma}z$ for any $z \in \R^n$, then $g_{k+1:} = \Vec{0}$. This, combined with the nested structure of the optimization problem, implies that the optimal solution for $U$ has to be of the form $u_{i, k+1:} = \Vec{0}$ for all interesting (non-zero mapping) directions, i.e., there exists $x \in \cX$ such that $v_i^\top \tilde{V}^\top x \neq 0$. These are the only interesting solutions since the case where for all $x \in \cX:$ $v_i^\top \tilde{V}^\top x = 0$ yields zero mapping on $\cX$, which is not of interest and could be dropped, e.g., using group lasso penalty discussed in the main part. Therefore, to solve the original problem, we could first solve the following problem
    \begin{align*}
        \min_{U \in \R^{k \times r}, V \in \R^{n \times r}} \EE{z \sim \cX}{\EE{b \sim \cD}{\norm*{\rbr*{U_{:k, :b}V_{:b}^\top - \tilde{\Sigma}_{:k, :}}z}^2}}
    \end{align*}
    and then reconstruct the corresponding solution of the original problem by appending zeros to the resulting matrix $U$. By a similar argument, we can argue that for all non-zero mapping directions, it has to be the case that $v_{i, k+1:} = \Vec{0}$. Therefore, solving the original minimization reduces to 
    \begin{align*}
        \min_{U \in \R^{k \times r}, V \in \R^{k \times r}} \EE{z \sim \cX}{\EE{b \sim \cD}{\norm*{\rbr*{U_{:b}V_{:b}^\top - \tilde{\Sigma}_{:k, :k}}z_{:k}}^2}}.
    \end{align*}
    This can be further simplified using reparametrization $V^\top \to V^\top \tilde{\Sigma}_{:k, :k}^{-1}$, which leads to
    \begin{align}
    \label{eq:pca_final}
        \min_{U \in \R^{k \times r}, V \in \R^{k \times r}} \EE{z \sim \cX}{\EE{b \sim \cD}{\norm*{\rbr*{U_{:b}V_{:b}^\top - I_k}\tilde{\Sigma}_{:k, :k}z_{:k}}^2}},
    \end{align}
    where $I_k$ is $k \times k$ identity. If $\cX$ is centred around zero, then $\tilde{\Sigma}_{:k, :k}z_{:k}$ is also centred around zero, and the above problem is up to scaling equivalent to PCA of $\tilde{\Sigma}_{:k, :k}z_{:k}$ as shown by Rippel et al.~\citep{rippel2014learning}. Since $\tilde{\Sigma}$ is a diagonal matrix with only $k \times k$ non-zero upper left sub-matrix, therefore, PCA on $\tilde{\Sigma}_{:k, :k}z_{:k}$ is equivalent to PCA on $\tilde{\Sigma}z$ by appending zeros to the obtained principal component vectors. Thus, we can write an equivalent formulation
    \begin{align*}
        \min_{U \in \R^{m \times r}, V \in \R^{n \times r}} \EE{z \sim \cX}{\EE{b \sim \cD}{\norm*{\rbr*{U_{:b}V_{:b}^\top - I}\tilde{\Sigma}\tilde{V}^\top x}^2}}.
    \end{align*}
    Furthermore, let $\Bar{U} , \Bar{V}$ belong to the set of optimal solutions of problem \eqref{eq:pca_final_full}. Then $U^\star = \tilde{U}^\top \Bar{U} , V^\star = \tilde{V}^\top\Bar{V}$ belong to the set of optimal solutions of problem \eqref{eq:low_rank_optim_app}. This can be proved by reversing our construction and ignoring scaling since \eqref{eq:pca_final_full} is scaling invariant.

    For the case $\cX$ is a uniform distribution on the unit ball, we have $\tilde{\Sigma}_{:k, :k}z_{:k}$ is a $k$-dimensional ellipsoid with principal axes being standard basis vectors $\cbr*{e_i}_{i=1}^k$, where the length of the axes is given by ordered singular values, i.e., the first basis vector corresponds to the largest singular vector.  Therefore, its principal component vectors correspond to the basis vectors. Following our construction, one can see that the solution to the original problems leads to truncated SVD up to the scaling factor. 

    For the case $A$ is an identity, we have $k=r=m=m$, $\tilde{\Sigma}$ is an identity, and $\tilde{U} = \tilde{V}$. Under this setting, the principal component vectors obtained from \eqref{eq:pca_final} corresponds to principal component vectors of $\cX$ in basis given by columns of $\tilde{U}$. Similarly, as in the previous case, reversing the transformations to return back to the original problem, we conclude that the optimal solution of the original problem corresponds to principal component vectors of $\cX$ since we reverse the transformation by $\tilde{U}^\top$.
\end{proof}

\section{Experimental setup}
\label{sec:app_experimental_setup}

\subsection{Datasets}

\noindent
\textbf{MNIST.}
The MNIST dataset~\citep{lecun2010mnist} is a database of 28$\times$28 greyscale handwritten digits, with a training set of 60k examples and a test set of 10k samples.

\noindent
\textbf{CIFAR-10.}
The CIFAR10 dataset~\citep{krizhevsky2009learning} is a computer vision dataset that consists of 32$\times$32 RGB images classified into 10 labels. It is split into 50k training images and 10k test images which are balanced across labels.

\noindent
\textbf{ImageNet-1k.}
\camready{The ImageNet dataset (ILSVRC)~\citep{deng2009imagenet} is an image classification challenge. The task comprises to classify an 300$\times$300 RGB image among 1000 classes. In total there are 1.2M training samples and 50k test images.}

\noindent
\textbf{WMT16.}
The WMT dataset from statmt is machine translation dataset, spanning news commentaries and parliament proceedings, that aims to investigate the applicability of machine translation techniques when translating between language pairs. Specifically, we focus on the task of German-English language translation of image descriptions, commonly referred to as \textbf{Multi30k}~\citep{elliott-EtAl:2016:VL16}. We only utilise the text modality for the translation task. Data is taken straight from \texttt{torchtext}.

\subsection{Models}

\noindent
\textbf{LeNet.}
LeNet is a simple convolutional network, introduced by LeCun at al. for recognizing handwritten digits~\citep{lecun2010mnist}. It consists of a sequence of two convolutional layers, followed by three fully-connected layers. However, we are using a ReLU instead of the initially proposed sigmoid activation. The detailed architecture of the network is depicted in Tab.~\ref{arch:lenet}

\noindent
\textbf{ResNet.}
ResNet~\citep{he2016deep} is a deep neural network whose prominent feature is the existence of skip (or residual) connections, that is connections that perform identity mappings merged with the target layer it joins with through summation. Multiple residual blocks are stacked to form the network. The result is an easier to optimise network that offers enhanced accuracy. We use ResNet-18 in our experiments, the architecture of which is depicted in Tab.~\ref{tab:resnet18-cifar10-arch}, except for ImageNet, where we use ResNet-50.

\begin{table}[ht] %
    \centering
    \adjustbox{valign=b}{
    \begin{minipage}{0.45\textwidth}
    \captionsetup{font=small,labelfont=bf}
    \caption{Detailed architecture of the LeNet-5 architecture used in our experiments. Each convolution and linear layer is followed by a ReLU activation that is ommitted from the table. The shapes for convolution layers follows $(m, n, k, k)$.}
		 \footnotesize{
		 \resizebox{0.9\linewidth}{!}{
			\begin{tabular}{lll}
				\toprule \rowcolor{Gray} \textbf{Parameter}
				& Shape &  Layer hyper-parameter \bigstrut\\
				\midrule
				\textbf{layer1.conv1.weight} & $1 \times 6 \times 5 \times 5$ & stride:$1$;padding:$1$ \bigstrut\\
                    \textbf{pooling.max} & N/A & kernel size:$2$;stride:$1$;dilation:$1$  \bigstrut\\
				\textbf{layer2.conv2.weight} & $6 \times 16 \times 5 \times 5$ & stride:$1$;padding:$0$;dilation:$1$  \bigstrut\\
				\textbf{pooling.max} & N/A & kernel size:$2$;stride:$2$  \bigstrut\\
				\textbf{layer3.fc1.weight} & $256\times 120$ & N/A \bigstrut\\
                    \textbf{layer4.fc2.weight} & $120\times 84$ & N/A \bigstrut\\
                    \textbf{layer5.fc3.weight} & $84\times 10$ & N/A \bigstrut\\
				\bottomrule
			\end{tabular}}%
			}
	\label{arch:lenet}
    \end{minipage}}
    \adjustbox{valign=b}{
    \begin{minipage}{0.5\textwidth}
        \centering
        \captionsetup{font=small,labelfont=bf}
        \caption{The hybrid ResNet architecture for the CIFAR-10 and ImageNet datasets used in the experiments.}
        \footnotesize{
        \resizebox{\linewidth}{!}{
        \begin{tabular}{lll}
            \toprule
            \rowcolor{Gray}  \textbf{Layer Name} & \textbf{ResNet-18} & \textbf{ResNet-50} \\
            \midrule
            \textbf{conv1} & \multicolumn{1}{c}{3$\times$3, 64, stride 1, padding 1} & \multicolumn{1}{c}{7$\times$7, 64, stride 2, padding 1} \\
            \midrule
            \multirow{4}{*}{\textbf{conv2\_x}} & %
            & 3$\times$3 maxpool, stride 2 \\
            & \blocka{64}{2} & \blockb{256}{64}{3} \\
            & & \\
            & & \\
            \midrule
            \multirow{3}{*}{\textbf{conv3\_x}}  & \blocka{128}{2} & \blockb{512}{128}{4}
                                          \\
            & \\
            & \\
            \midrule
            \multirow{3}{*}{\textbf{conv4\_x}} & \blocka{256}{2} & \blockb{1024}{256}{6} \\
            & \\
            & \\
            \midrule
            \multirow{3}{*}{\textbf{conv5\_x}}  & \blocka{512}{2} & \blockb{2048}{512}{3} \\
              & \\
              & \\
            \midrule
            & \multicolumn{1}{c}{Avg Pool, 10-dim FC, SoftMax} & \multicolumn{1}{c}{Avg Pool, 20-dim FC, SoftMax} \\
            \bottomrule
        \end{tabular}
        }}
    \label{tab:resnet18-cifar10-arch}
    \end{minipage}}
\end{table}

\noindent
\textbf{VGG.}
VGG~\citep{Simonyan15} is a also a convolutional network that leverages smaller 3$\times$3 convolutions that enables deeper architecture than before. For our experiments we are using VGG-19, the architecture of which is depicted in Tab.~\ref{arch:vgg}.

\begin{table}[ht]
    \centering
    \captionsetup{font=small,labelfont=bf}
    \caption{Detailed architecture of the VGG-19 architecture used in our experiments. There is a BatchNorm layer followed by a ReLU activation (omitted in the table) after each convolution layer. The shapes for convolution layers follows $(m, n, k, k)$.}
		 \footnotesize{
		 \resizebox{0.38\linewidth}{!}{
			\begin{tabular}{lll}
				\toprule \rowcolor{Gray} \textbf{Parameter}
				& Shape &  Layer hyper-parameter \bigstrut\\
				\midrule
				\textbf{layer1.conv1.weight} & $3 \times 64 \times 3 \times 3$ & stride:$1$;padding:$1$ \bigstrut\\
				\textbf{layer2.conv2.weight} & $64 \times 64 \times 3 \times 3$ & stride:$1$;padding:$1$  \bigstrut\\
				\textbf{pooling.max} & N/A & kernel size:$2$;stride:$2$  \bigstrut\\
				\textbf{layer3.conv3.weight} & $64\times 128 \times 3 \times 3$ & stride:$1$;padding:$1$ \bigstrut\\
				\textbf{layer4.conv4.weight} & $128\times 128 \times 3 \times 3$ & stride:$1$;padding:$1$ \bigstrut\\
                \textbf{pooling.max} & N/A & kernel size:$2$;stride:$2$  \bigstrut\\
				\textbf{layer5.conv5.weight} & $128 \times 256 \times 3 \times 3$ & stride:$1$;padding:$1$  \bigstrut\\
				\textbf{layer6.conv6.weight} & $256\times 256 \times 3 \times 3$ & stride:$1$;padding:$1$  \bigstrut\\
				\textbf{layer7.conv7.weight} & $256 \times 256 \times 3 \times 3$ & stride:$1$;padding:$1$  \bigstrut\\
				\textbf{layer8.conv8.weight} & $256 \times 256 \times 3 \times 3$ & stride:$1$;padding:$1$  \bigstrut\\
                \textbf{pooling.max} & N/A & kernel size:$2$;stride:$2$  \bigstrut\\
				\textbf{layer9.conv9.weight} & $256 \times 512 \times 3 \times 3$ & stride:$1$;padding:$1$  \bigstrut\\
				\textbf{layer10.conv10.weight} & $512 \times 512 \times 3 \times 3$ & stride:$1$;padding:$1$  \bigstrut\\
				\textbf{layer11.conv11.weight} & $512 \times 512 \times 3 \times 3$ & stride:$1$;padding:$1$  \bigstrut\\
				\textbf{layer12.conv12.weight} & $512 \times 512 \times 3 \times 3$ & stride:$1$;padding:$1$  \bigstrut\\
				\textbf{pooling.max} & N/A & kernel size:$2$;stride:$2$  \bigstrut\\
				\textbf{layer13.conv13.weight} & $512 \times 512 \times 3 \times 3$ & stride:$1$;padding:$1$  \bigstrut\\
				\textbf{layer14.conv14.weight} & $512 \times 512 \times 3 \times 3$ & stride:$1$;padding:$1$  \bigstrut\\
				\textbf{layer15.conv15.weight} & $512 \times 512 \times 3 \times 3$ & stride:$1$;padding:$1$  \bigstrut\\
				\textbf{layer16.conv16.weight} & $512 \times 512 \times 3 \times 3$ & stride:$1$;padding:$1$  \bigstrut\\
				\textbf{pooling.avg} & N/A & kernel size:$2$  \bigstrut\\
				\textbf{classifier.weight} & $512 \times 10$ & N/A  \bigstrut\\
				\textbf{classifier.bias} & $10$ & N/A  \bigstrut\\
				\bottomrule
			\end{tabular}}%
			}
	\label{arch:vgg}   
\end{table}

\noindent
\textbf{Transformers.} The transformer architecture~\citep{vaswani2017attention} has been lately revolutionising deep learning. Based on the notion of self-attention, for each input token, it produces a weighted combination of other relevant tokens weighed by the attention weight. Each attention unit has three weight matrices, namely $W_Q$, $W_K$, $W_V$, for query, key and value weights respectively producing the equivalent vectors. Attention is defined as the scaled dot product between key and query. For our translation task, we use the architecture depicted in Tab.~\ref{table:architecture-transformer-decoder}.

\begin{table*}[h!] %
    \centering
    \adjustbox{valign=b}{
    \begin{minipage}{0.48\textwidth}
    \captionsetup{font=small,labelfont=bf}
	\caption{Detailed information of the encoder layer in the Transformer architecture in our experiment}
	\label{table:architecture-transformer-encoder}
	\begin{center}
		 \scriptsize{
                \resizebox{0.8\linewidth}{!}{
			\begin{tabular}{ccc}
				\toprule \rowcolor{Gray} \textbf{Parameter}
				& Shape & Hyper-param. \bigstrut\\
				\midrule
				\textbf{embedding} & $9521\times 512$ & padding index: 1 \bigstrut\\
				\textbf{positional encoding} & N/A & N/A \bigstrut\\
				\textbf{dropout} & N/A & $p=0.1$ \bigstrut\\
				\textbf{encoder.self-attention.wq}($W^Q$) & $512\times 512$ & N/A \bigstrut\\
                \textbf{encoder.self-attention.wk}($W^K$) & $512\times 512$ & N/A \bigstrut\\
                \textbf{encoder.self-attention.wv}($W^V$) & $512\times 512$ & N/A \bigstrut\\
                \textbf{encoder.self-attention.wo}($W^O$) & $512\times 512$ & N/A \bigstrut\\
                \textbf{encoder.self-attention.dropout} & N/A & $p=0.1$ \bigstrut\\
                \textbf{encoder.self-attention.layernorm} & $512$ & $\epsilon=10^{-6}$
                \bigstrut\\
                \textbf{encoder.ffn.layer1} & $512\times 2048$ & N/A
                \bigstrut\\
                \textbf{encoder.ffn.layer2} & $2048\times 512$ & N/A
                \bigstrut\\
                \textbf{encoder.layernorm} & $512$ & $\epsilon=10^{-6}$
                \bigstrut\\
				\textbf{dropout} & N/A & $p=0.1$ \bigstrut\\
				\bottomrule
			\end{tabular}}}%
	\end{center}
    \end{minipage}} %
    \adjustbox{valign=b}{
    \begin{minipage}{0.38\textwidth}
        \captionsetup{font=small,labelfont=bf}
        \caption{Detailed information of the decoder layer in the Transformer architecture in our experiment}
	\label{table:architecture-transformer-decoder}
	\begin{center}
		 \footnotesize{
            \resizebox{0.99\linewidth}{!}{
			\begin{tabular}{lll}
				\toprule \rowcolor{Gray} \textbf{Parameter}
				& Shape & Hyper-param. \bigstrut\\
				\midrule
				\textbf{embedding} & $9521\times 512$ & padding index: 1 \bigstrut\\
				\textbf{positional encoding} & N/A & N/A \bigstrut\\
				\textbf{dropout} & N/A & $p=0.1$ \bigstrut\\
				\textbf{decoder.self-attention.wq}($W^Q$) & $512\times 512$ & N/A \bigstrut\\
                \textbf{decoder.self-attention.wk}($W^K$) & $512\times 512$ & N/A \bigstrut\\
                \textbf{decoder.self-attention.wv}($W^V$) & $512\times 512$ & N/A \bigstrut\\
                \textbf{decoder.self-attention.wo}($W^O$) & $512\times 512$ & N/A \bigstrut\\
                \textbf{decoder.self-attention.dropout} & N/A & $p=0.1$ \bigstrut\\
                \textbf{decoder.self-attention.layernorm} & $512$ & $\epsilon=10^{-6}$
                \bigstrut\\
				\textbf{decoder.enc-attention.wq}($W^Q$) & $512\times 512$ & N/A \bigstrut\\
                \textbf{decoder.enc-attention.wk}($W^K$) & $512\times 512$ & N/A \bigstrut\\
                \textbf{decoder.enc-attention.wv}($W^V$) & $512\times 512$ & N/A \bigstrut\\
                \textbf{decoder.enc-attention.wo}($W^O$) & $512\times 512$ & N/A \bigstrut\\
                \textbf{decoder.enc-attention.dropout} & N/A & $p=0.1$ \bigstrut\\
                \textbf{decoder.enc-attention.layernorm} & $512$ & $\epsilon=10^{-6}$
                \bigstrut\\
                \textbf{decoder.ffn.layer1} & $512\times 2048$ & N/A
                \bigstrut\\
                \textbf{decoder.ffn.layer2} & $2048\times 512$ & N/A
                \bigstrut\\
                \textbf{encoder.layernorm} & $512$ & $\epsilon=10^{-6}$
                \bigstrut\\
				\textbf{dropout} & N/A & $p=0.1$ \bigstrut\\
				\bottomrule
			\end{tabular}
            }
            }
	\end{center}
    \end{minipage}
    }
\end{table*}

\subsection{Hyperparameter Selection}

\textbf{LeNet.} We use a standard configuration that is commonly employed for training LeNet models — a step size of 0.01, a momentum of 0.9, and no weight decay. We train for a total of 20 epochs.

\textbf{VGG and ResNet-18.} Similarly, we use a standard configuration that is commonly employed for training VGG and ResNet-18 models — a step size of 0.01, a momentum of 0.9, weight decay of $1e^{-4}$, and a learning schedule with step size reductions by a factor of 10 at epochs 150 and 250. We train for a total of 300 epochs.

\textbf{ResNet-50.} Similarly, we use a standard configuration that is commonly employed for training ResNet-50 models — a step size of 0.01, a momentum of 0.9, weight decay of $1e^{-4}$, and a learning schedule with step size reductions by a factor of 10 at epochs 30 and 60. We train for a total of 90 epochs.

\textbf{Transformers.} For the Transformer model, we use the Adam optimizer with an initial learning rate at $0.001$, $\beta s=(0.9, 0.98), \epsilon=10^{-8}$ batch size at $256$. We also conduct gradient norm clipping with norm bound at $0.25$. The entire training takes $400$ epochs. For the vanilla warm-up training, we use warm-up epoch $E_{wu}=10$. We enable label smoothing, weight sharing for the source and target word embedding, and weight sharing between target word embedding and the last dense layer. The learning rate schedule follows directly from the one proposed~\cite{vaswani2017attention}.

\subsection{Deciding Against Decomposition}
During inference, if the rank of a given layer is so large that keeping it as a non-decomposed layer is more efficient, we opt not to decompose that particular layer.

\section{Extended evaluation}

\subsection{\tool Recovers Correct Ordering}
\label{sec:correct_ordering}

In the main text, we pointed out that SVD fails to consider data. Indeed, even in the case of linear NN, the acquired singular vectors may exhibit incorrect ordering. To illustrate this problem, we provide a simple example in which we use a matrix $A$ with a rank of $3$. We organize the dataset $\cX$ such that the third singular vector has the highest importance, followed by the second and then the first singular vector in decreasing order of significance. It is clear that SVD doesn't consider the data, and as a result, it cannot comprehend this behavior. Below (in Fig.~\ref{fig:svd_wrong}), we demonstrate how \tool is able to correctly discern the order.

\begin{figure}[h!]
    \centering
    \includegraphics[width=0.4\textwidth]{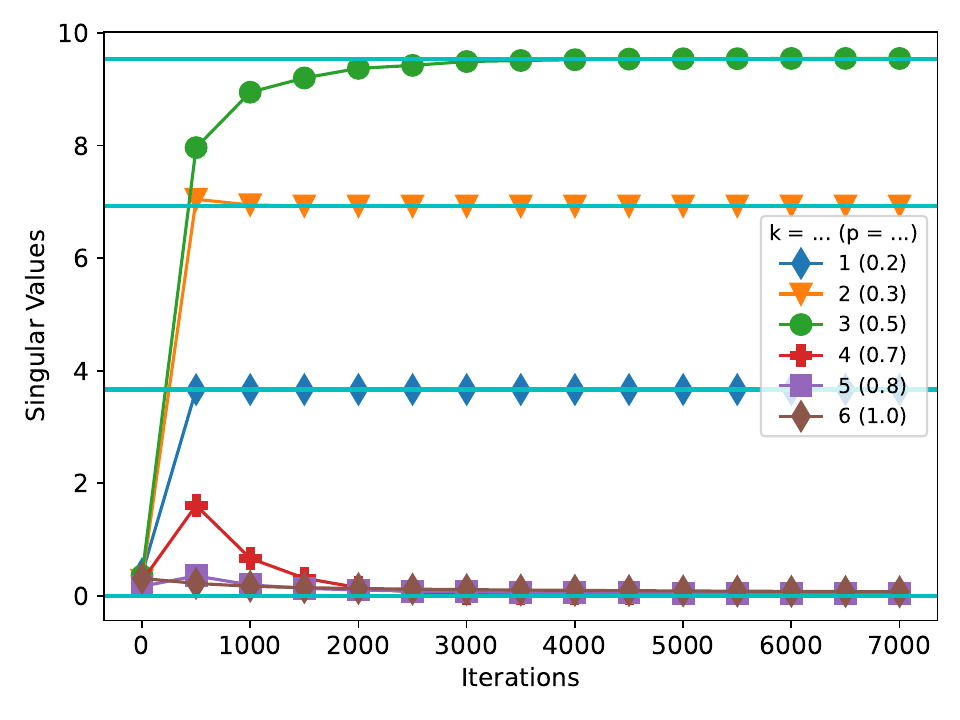}
    \captionsetup{font=small,labelfont=bf}
    \caption{Verification that \tool recovers correct order of importance. Target mapping is of rank $3$, and the dataset is constructed in such a way that the singular vectors have reversed the order of importance. $p$ and $k$ stand for relative and actual rank, respectively.}
    \label{fig:svd_wrong}
\end{figure}

\subsection{Rank Adaptivity of \tool to Data Complexity}
\label{sec:rank_adaptivity}

\camready{So far, we have found that different models can have different ranks on different datasets. However, we did not reach the conclusion that more complex tasks lead to higher ranks because the model architecture is not invariant, i.e., we cannot compare ranks between layers of different dimensionality.}

\begin{figure}[h!]
    \centering
    \includegraphics[width=0.4\textwidth]{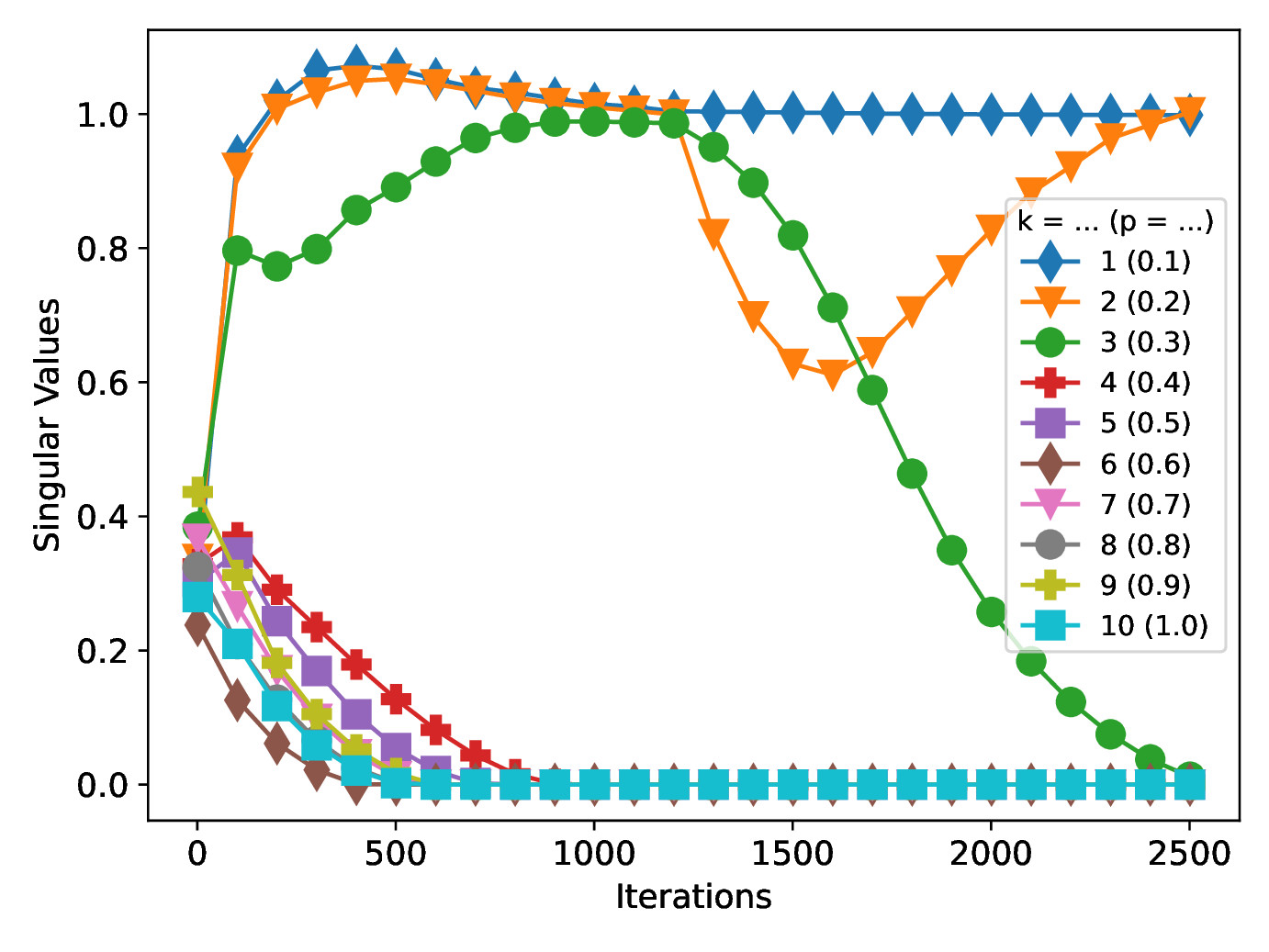}
    \captionsetup{font=small,labelfont=bf}
    \caption{\tool adaptivity when PCA dimensionality drops during training. The plot displays the estimates of singular values. The data distribution has initially 3 directions. It is expected that the top $3$ ranks will converge to value one and the rest to zero. After removing one direction, ranks drop to 2, as the data complexity changes. $p$ and $k$ stand for relative and actual rank, respectively.}
    \label{fig:pca-dim-drop}
\end{figure}

\camready{To test this hypothesis in silo, we designed a simplified experiment on a linear autoencoder example with the same setup as considered in Fig.~\ref{fig:pca}). To showcase the adaptivity of \tool to changing data, we start by training with data that have intrinsic dimension 3. In the middle of the training (iteration 1250/2500), we removed one dimension using projection, thus simplifying the data. In the resulting graph (Fig.~\ref{fig:pca-dim-drop}), we see that while the initial rank had converged to 3, it now drops to 2 as the data complexity changes. For completeness, this adaptivity is further showcased in Appendix~\ref{sec:correct_ordering}, where we have illustrated how SVD fails to consider data-centric factors, whereas Maestro recovers the correct order of importance.}

\subsection{Training Behaviour of \tool}
\vspace{-0.1cm}
\label{app:tool_behaviour}
 
For completeness, we also include an extended version of Fig.~\ref{fig:training_dynamics} from the main paper, where we presented the training dynamics for \tool. Fig.~\ref{fig:rank_evolution},~\ref{fig:final_rank} and ~\ref{fig:loss} present similar plots, but across both MNIST and CIFAR-10. Specifically, Fig.~\ref{fig:rank_evolution} illustrates the evolution of total rank throughout the training steps. We observe that the ranks are pruned incrementally. This aligns with the observations made during Pufferfish~\cite{wang2021pufferfish} training, where the authors suggested warm-start training with full precision to enhance the final model performance.  In our case, the necessity to implement this heuristic is avoided, as \tool prunes rank automatically. Fig.~\ref{fig:final_rank} demonstrates the ranks across layers post-training. 
An intriguing trend is observed: the ranks are nested for increasing $\lambda_{gl}$, suggesting a potential inherent ordering of ranks not only within each layer but also possibly a global one. We provide a preliminary exploration of this fascinating pattern in the subsequent section and intend to probe it more deeply in future studies. We believe this may enhance the rank selection and sampling process. Finally, Fig.~\ref{fig:loss} portrays the evolution of the training loss. Our premise that sampling does not negatively affect training is validated by empirical performance.

\begin{figure*}[h]
  \centering
  \begin{subfigure}[t]{0.32\textwidth}
    \includegraphics[width=\textwidth]{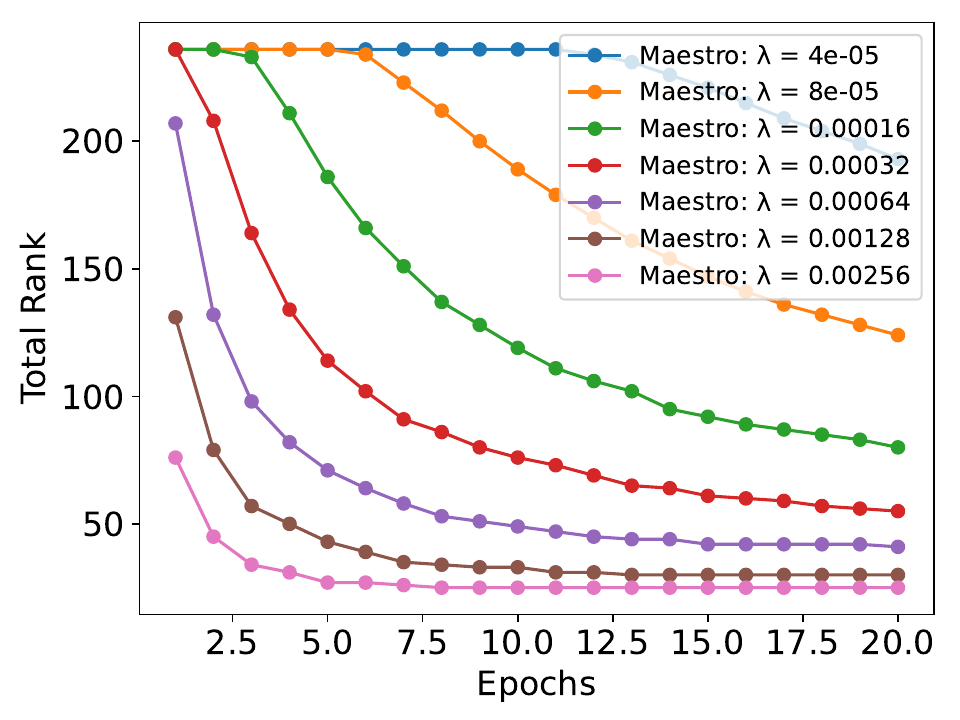}
    \captionsetup{font=small,labelfont=bf}
    \caption{LeNet on MNIST}
  \end{subfigure}
  \hfill
  \begin{subfigure}[t]{0.32\textwidth}
    \includegraphics[width=\textwidth]{plots/resnet_svd_ss_hierarchical_rank.pdf}
    \captionsetup{font=small,labelfont=bf}
    \caption{ResNet-18 on CIFAR10}
  \end{subfigure}
  \hfill
  \begin{subfigure}[t]{0.32\textwidth}
    \includegraphics[width=\textwidth]{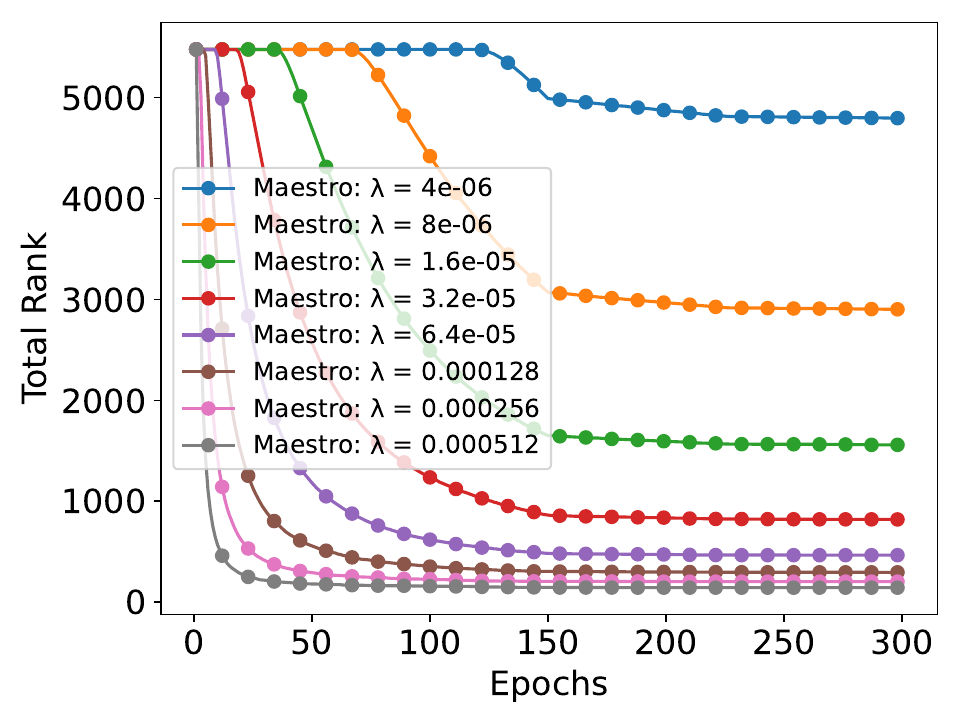}
    \captionsetup{font=small,labelfont=bf}
    \caption{VGG19 on CIFAR10}
  \end{subfigure}
  \captionsetup{font=small,labelfont=bf}
  \caption{Total rank ($\sum_{i=1}^d r_i$) progression during training.}
  \label{fig:rank_evolution}
\end{figure*}
\begin{figure*}[h]
  \centering
  \begin{subfigure}[t]{0.32\textwidth}
    \includegraphics[width=\textwidth]{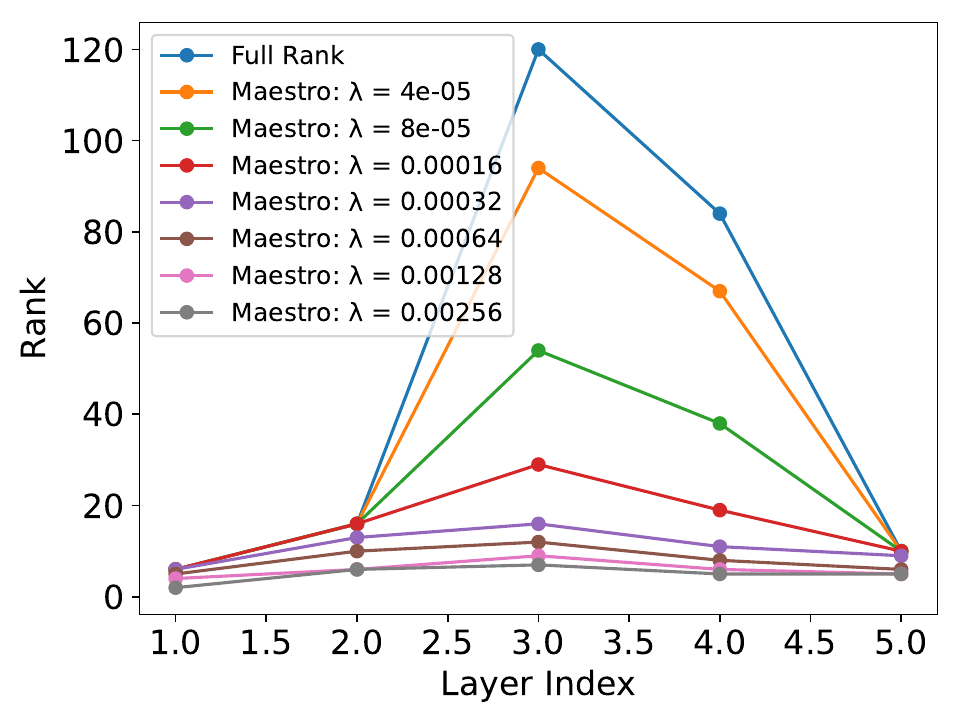}
    \captionsetup{font=small,labelfont=bf}
    \caption{LeNet on MNIST}
  \end{subfigure}
  \hfill
  \begin{subfigure}[t]{0.32\textwidth}
    \includegraphics[width=\textwidth]{plots/resnet_svd_ss_hierarchical_final_rank.pdf}
    \captionsetup{font=small,labelfont=bf}
    \caption{ResNet-18 on CIFAR10}
  \end{subfigure}
  \hfill
  \begin{subfigure}[t]{0.32\textwidth}
    \includegraphics[width=\textwidth]{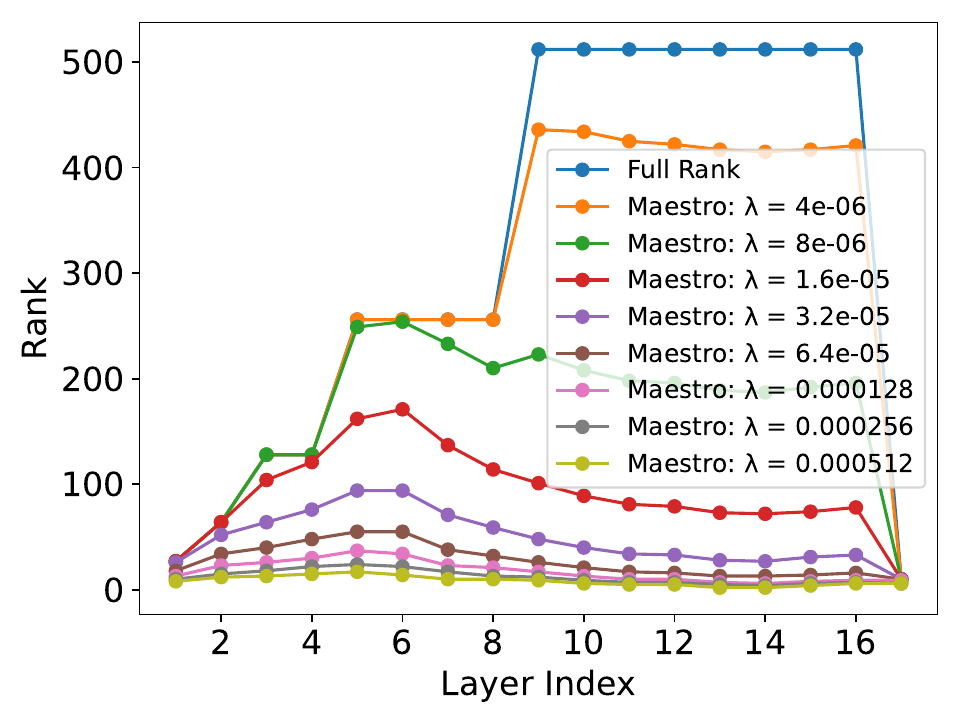}
    \captionsetup{font=small,labelfont=bf}
    \caption{VGG19 on CIFAR10}
  \end{subfigure}
  \captionsetup{font=small,labelfont=bf}
  \caption{Ranks $r_i$'s across different layers after training.}
  \label{fig:final_rank}
\end{figure*}

\begin{figure*}[h]
  \centering
  \begin{subfigure}[t]{0.32\textwidth}
    \includegraphics[width=\textwidth]{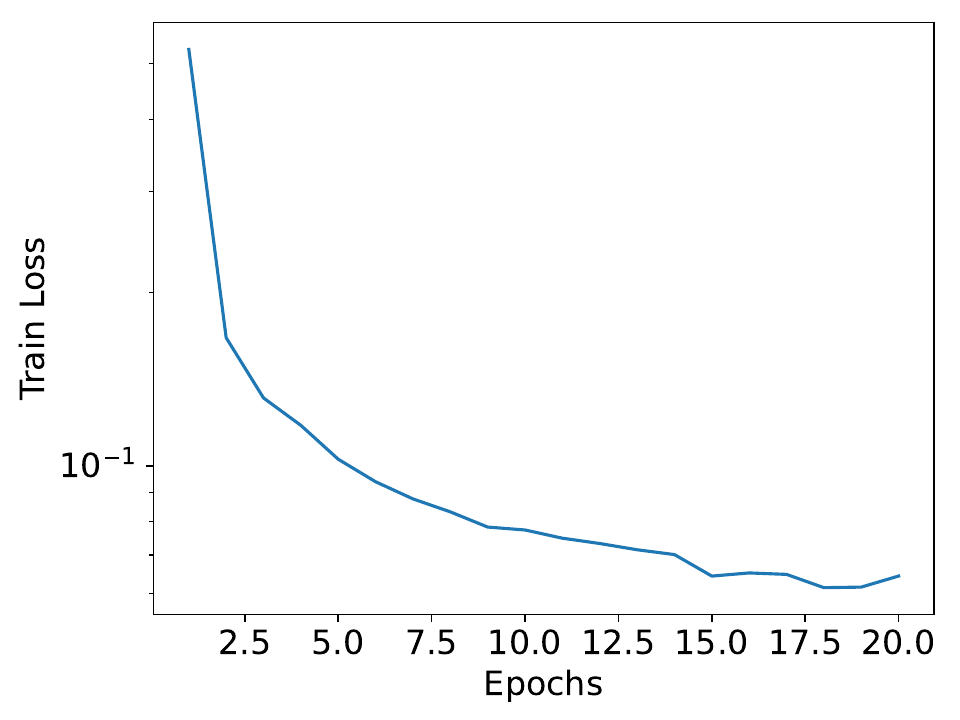}
    \captionsetup{font=small,labelfont=bf}
    \caption{LeNet on MNIST}
  \end{subfigure}
  \hfill
  \begin{subfigure}[t]{0.32\textwidth}
    \includegraphics[width=\textwidth]{plots/resnet_svd_ss_train_loss.pdf}
    \captionsetup{font=small,labelfont=bf}
    \caption{ResNet-18 on CIFAR10}
  \end{subfigure}
  \hfill
  \begin{subfigure}[t]{0.32\textwidth}
    \includegraphics[width=\textwidth]{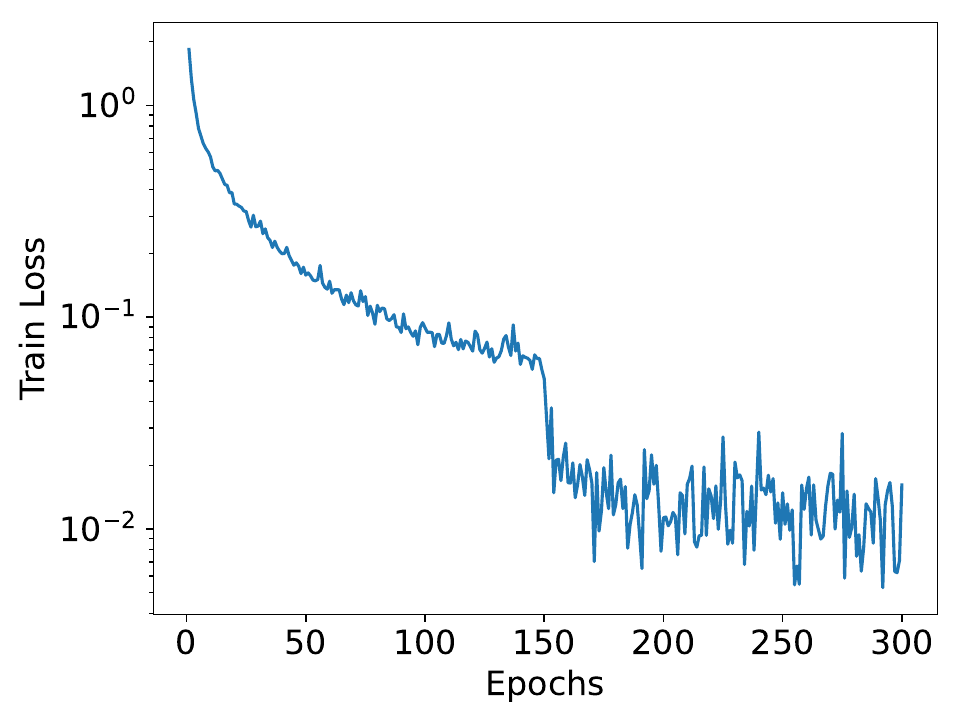}
    \captionsetup{font=small,labelfont=bf}
    \caption{VGG19 on CIFAR10}
  \end{subfigure}
  \captionsetup{font=small,labelfont=bf}
  \caption{Convergence of \tool with $\lambda_{gl} = 0$.}
  \label{fig:loss}
\end{figure*}

\subsection{Model Size-Accuracy Trade-Off at Training and Deployment Time}
\label{app:tradeoff}

In addition to the original illustrations, we present an extended interpretation of Fig.~\ref{fig:acc_latency_trade_off}, where we depict diverse strategies to maintain a balance between model size and accuracy in the process of model training and deployment. In Fig.~\ref{fig:acc_latency_trade_off_svd}, we demonstrate the effective pruning of \tool ($\lambda_{gl}=0$) for deployment, utilizing the greedy search methodology discussed in Section~\ref{sec:deploy}. This is juxtaposed with the greedy pruning of a model not originally factorized but later factorized through SVD. Our findings reveal that this straightforward baseline does not match the performance of \tool's learned decomposition, leading to a considerable performance drop.

\begin{figure*}[h]
  \centering
  \begin{subfigure}[t]{0.32\textwidth}
    \includegraphics[width=\textwidth]{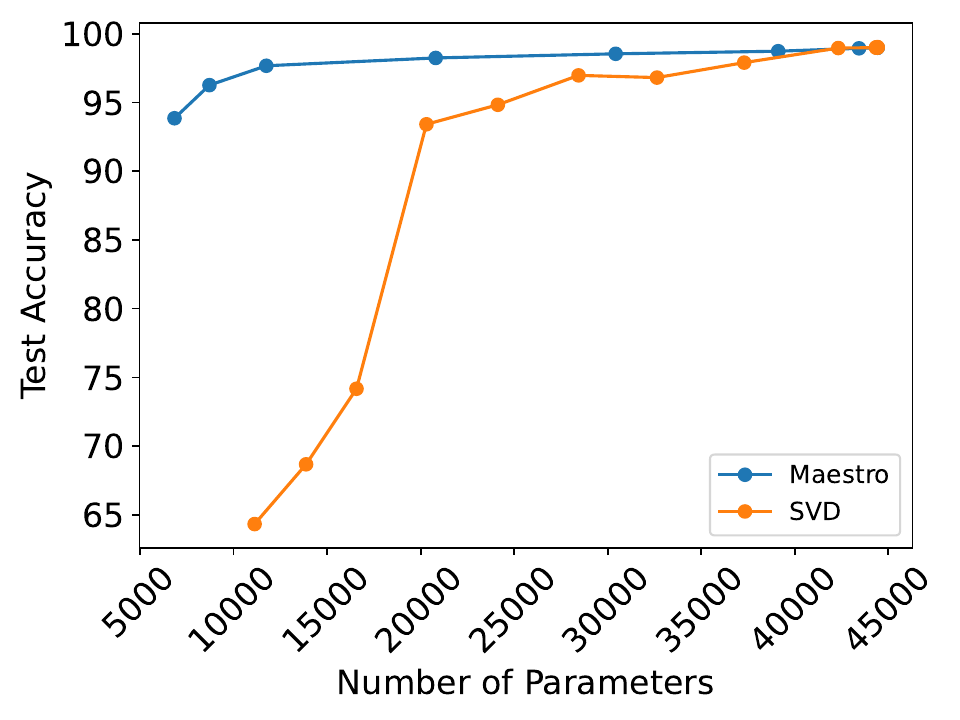}
    \captionsetup{font=small,labelfont=bf}
    \caption{LeNet on MNIST}
  \end{subfigure}
  \hfill
  \begin{subfigure}[t]{0.32\textwidth}
    \includegraphics[width=\textwidth]{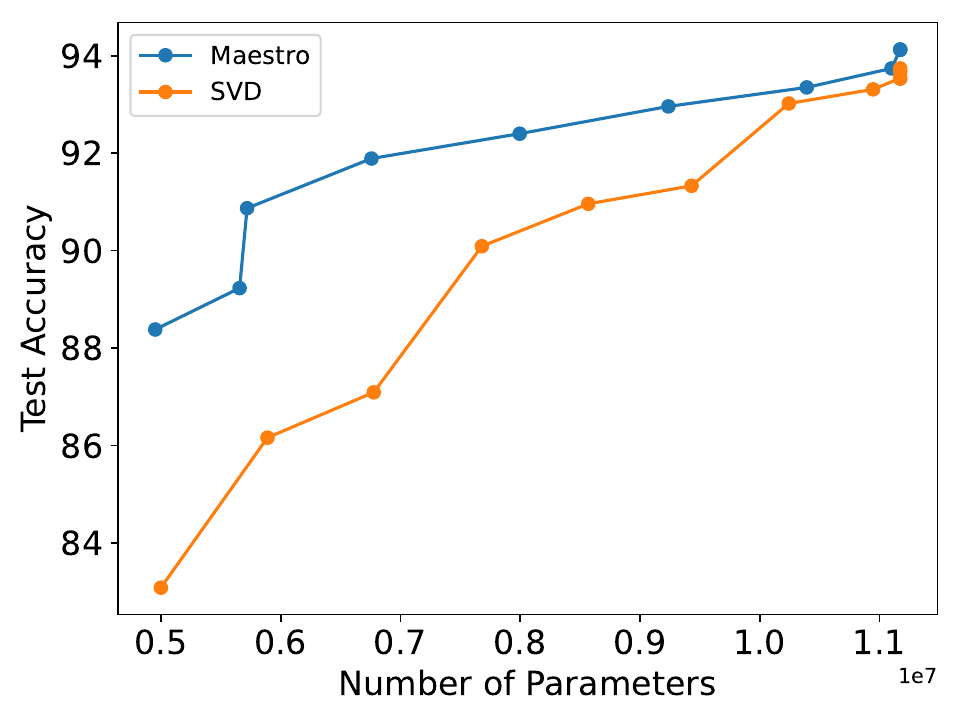}
    \captionsetup{font=small,labelfont=bf}
    \caption{ResNet-18 on CIFAR10}
  \end{subfigure}
  \hfill
  \begin{subfigure}[t]{0.32\textwidth}
    \includegraphics[width=\textwidth]{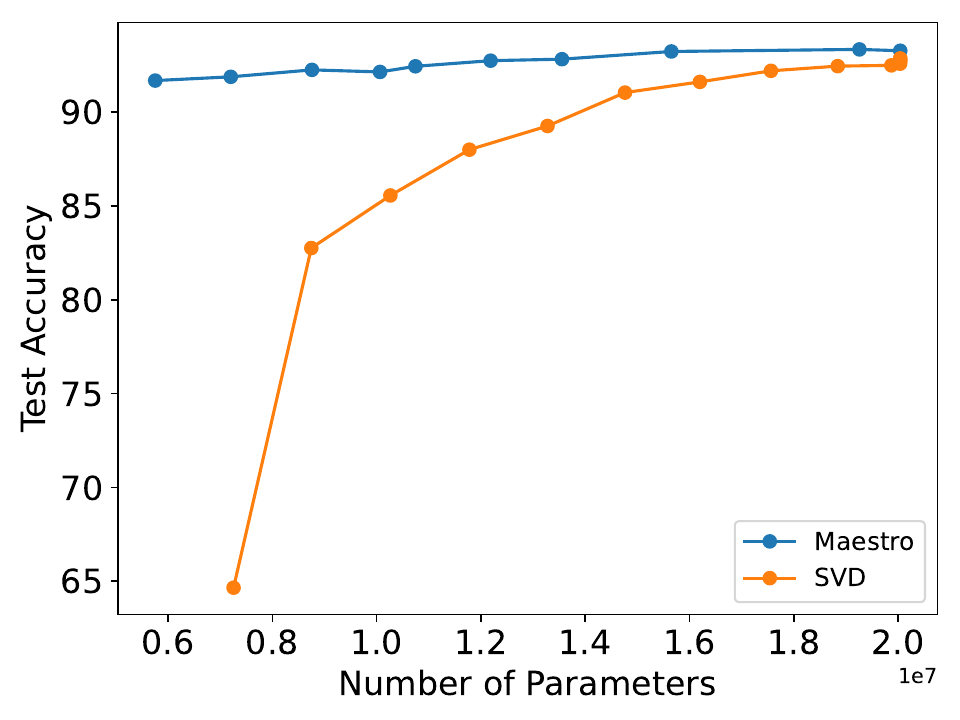}
    \captionsetup{font=small,labelfont=bf}
    \caption{VGG19 on CIFAR10}
  \end{subfigure}
  \captionsetup{font=small,labelfont=bf}
  \captionsetup{font=small,labelfont=bf}
  \caption{Accuracy-latency trade-off comparing \tool ($\lambda_{\text{gl} = 0}$) and SVD.}
  \label{fig:acc_latency_trade_off_svd}
\end{figure*}

Subsequently, Fig.~\ref{fig:hierarchical_group_lasso} displays the end accuracy and the count of model parameters corresponding to various hierarchical group lasso penalties. This results in an optimal compromise between latency and accuracy for both the training and inference stages. It's worth noting, though, that each model was trained separately, in contrast to greedy pruning, which demands just a single training round. Additionally, we scrutinize the training expense for each model illustrated in Fig.~\ref{fig:hierarchical_group_lasso}, the results of which are exhibited in Tables~\ref{tab:lenet_gp_lambda}, \ref{tab:resnet_gp_lambda}, \ref{tab:vgg_gp_lambda}, \ref{tab:transf_gp_lambda} and \ref{tab:resnet50_gp_lambda}, where we display and the final accuracy of the model, MACs and the number of parameters for inference, and relative total training cost in terms of the number of model parameters and MACs compared to the non-factorized model.  Interestingly, smaller models are not only advantageous in terms of inference efficiency, but they can also be trained at a small portion of the cost required by full-rank models. On the downside, the smallest models cause a non-negligible reduction in performance.

\begin{figure*}[h]
  \centering
  \begin{subfigure}[t]{0.32\textwidth}
    \includegraphics[width=\textwidth]{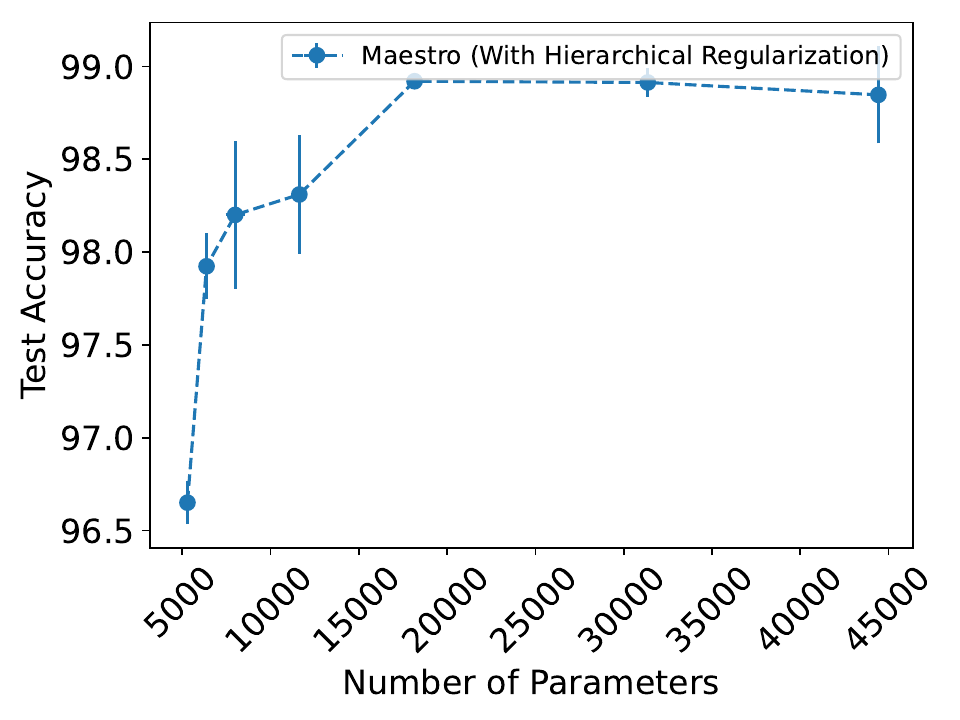}
    \captionsetup{font=small,labelfont=bf}
    \caption{LeNet on MNIST}
  \end{subfigure}
  \hfill
  \begin{subfigure}[t]{0.32\textwidth}
    \includegraphics[width=\textwidth]{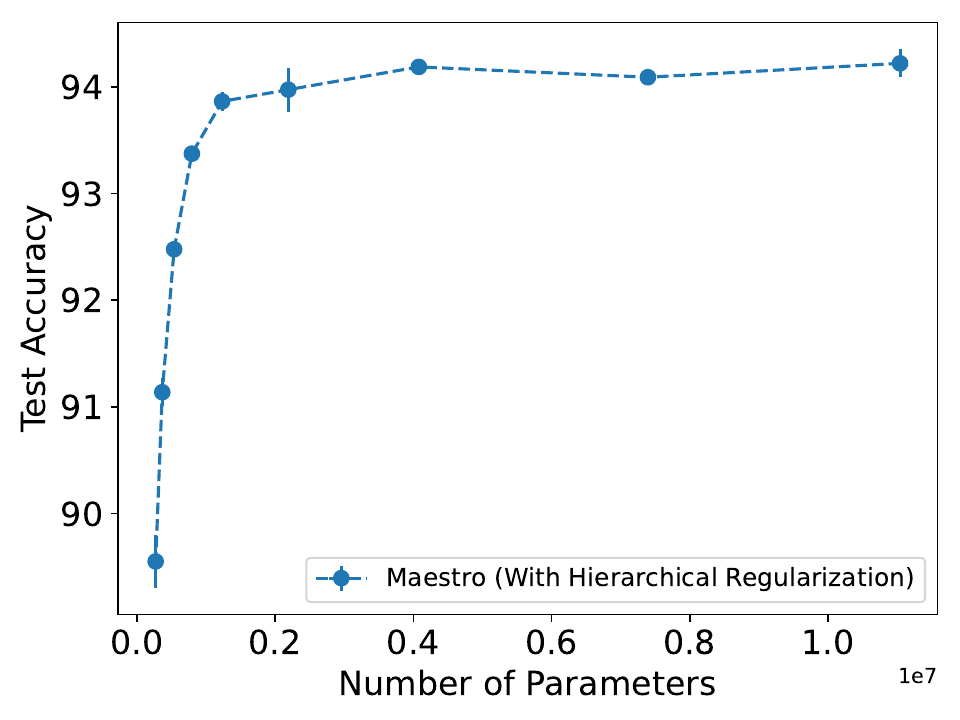}
    \captionsetup{font=small,labelfont=bf}
    \caption{ResNet-18 on CIFAR10}
  \end{subfigure}
  \hfill
  \begin{subfigure}[t]{0.32\textwidth}
    \includegraphics[width=\textwidth]{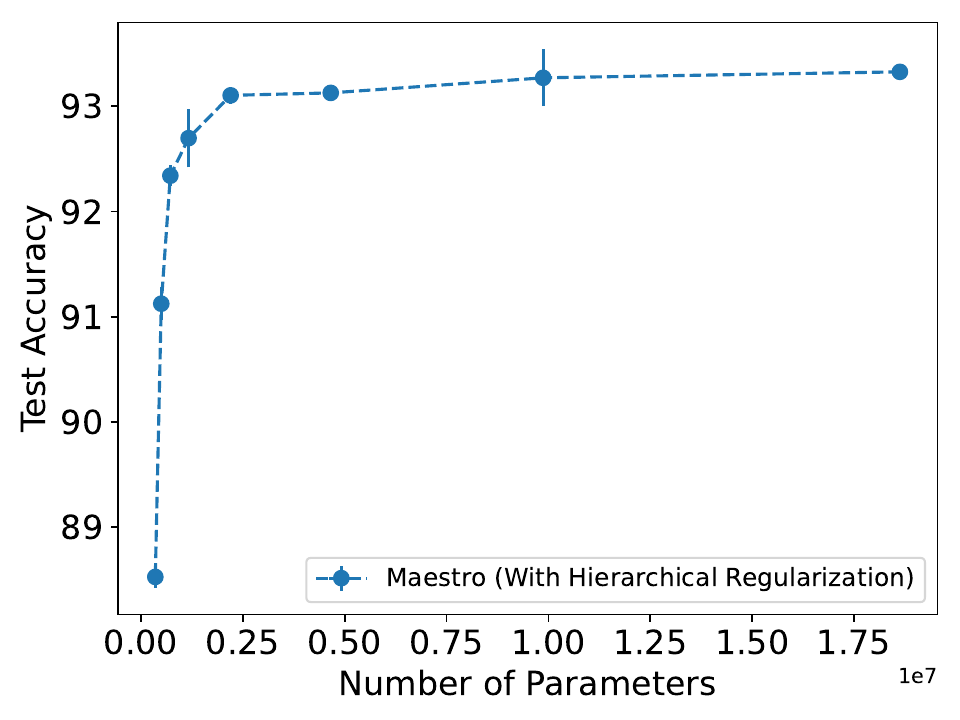}
    \captionsetup{font=small,labelfont=bf}
    \caption{VGG19 on CIFAR10}
  \end{subfigure}
  \captionsetup{font=small,labelfont=bf}
  \caption{Impact of hierarchical group lasso on the accuracy-memory trade-off. Exact values are provided in Tables~\ref{tab:lenet_gp_lambda}, \ref{tab:resnet_gp_lambda} and \ref{tab:vgg_gp_lambda}, respectively.}
  \label{fig:hierarchical_group_lasso}
\end{figure*}

\begin{table}[h]
  \centering
  \captionsetup{font=small,labelfont=bf}
  \caption{LeNet performance on MNIST for different regularization parameters. The last column in the table displays the relative total training cost in terms of the number of Multiply-Accumulate operations (MACs) and model parameters, compared to the non-factorized model.}
  \label{tab:lenet_gp_lambda}
  \begin{adjustbox}{max width=.8\linewidth}
      \begin{tabular}{llllcc}
          \toprule
          \rowcolor{Gray} Variant                         & Acc. (\%)     & MACs (Inf.) &  Params. (Inf.) & Rel. MACs / Params. (Train.) \\ \midrule
          Non-Factorized     & 98.99{\tiny$\pm$0.09}        &  281640{\tiny$\pm$0} (1.00$\times$)      &   44426{\tiny$\pm$0} (1.00$\times$)  &  1.00$\times$ / 1.00$\times$       \\
          \tool ($\lambda_{gp} = 0.$) & 99.06{\tiny$\pm$0.09}        &  281640{\tiny$\pm$0} (1.00$\times$)    &   44426{\tiny$\pm$0}(1.00$\times$) & 1.14$\times$/ 1.49$\times$   \\
          \tool ($\lambda_{gp} = 8e^{-5}$) & 98.91{\tiny$\pm$0.09}        &  268577{\tiny$\pm$389} (0.95$\times$)       &   31363{\tiny$\pm$0} (0.71$\times$) & 1.08$\times$/ 1.14$\times$    \\
          \tool ($\lambda_{gp} = 16e^{-5}$) & 98.92{\tiny$\pm$0.05}        &  255369{\tiny$\pm$217} (0.91$\times$) &   44426{\tiny$\pm$217} (0.41$\times$) & 1.06$\times$/ 0.80$\times$   \\
          \tool ($\lambda_{gp} = 32e^{-5}$) & 98.31{\tiny$\pm$0.39}        &  237084{\tiny$\pm$6268} (0.84$\times$)     &   18155{\tiny$\pm$271} (0.26$\times$) & 0.93$\times$/ 0.53$\times$  \\
          \tool ($\lambda_{gp} = 64e^{-5}$) & 98.20{\tiny$\pm$0.49}        &  178165{\tiny$\pm$19098} (0.63$\times$)     &   7996{\tiny$\pm$662} (0.18$\times$) & 0.77$\times$/ 0.33$\times$   \\
          \tool ($\lambda_{gp} = 128e^{-5}$) & 97.92{\tiny$\pm$0.22}        &  131789{\tiny$\pm$8965} (0.47$\times$)      &   6375{\tiny$\pm$77} (0.14$\times$) & 0.54$\times$/ 0.21$\times$   \\
          \tool ($\lambda_{gp} = 256e^{-5}$) & 96.65{\tiny$\pm$0.14}        &  99969{\tiny$\pm$6252} (0.35$\times$)      &   5293{\tiny$\pm$214} (0.12$\times$) & 0.39$\times$/ 0.14$\times$   \\
          \bottomrule
      \end{tabular}
  \end{adjustbox}
\end{table}
\begin{table}[h]
  \centering
  \captionsetup{font=small,labelfont=bf}
  \caption{ResNet-18 performance on CIFAR10 for different regularization parameters. The last column in the table displays the relative total training cost in terms of the number of Multiply-Accumulate operations (MACs) and model parameters, compared to the non-factorized model.}
  \label{tab:resnet_gp_lambda}
  \begin{adjustbox}{max width=.8\linewidth}
      \begin{tabular}{llllc}
          \toprule
          \rowcolor{Gray} Variant                         & Acc. (\%)     & GMACs (Inf.) &  Params. (M) (Inf.) & Rel. MACs / Params. (Train.) \\ \midrule
          Non-Factorized     & 93.86{\tiny$\pm$0.20}        &  0.56{\tiny$\pm$0} (1.00$\times$)     &   11.2{\tiny$\pm$0} (1.00$\times$)  &  1.00$\times$ / 1.00$\times$         \\
          \tool ($\lambda_{gp} = 0.$) & 94.04{\tiny$\pm$0.10}        &  0.56{\tiny$\pm$0} (1.00$\times$)      &   11.2{\tiny$\pm$0} (1.00$\times$)  &  1.10$\times$ / 1.13$\times$  \\
          \tool ($\lambda_{gp} = 4e^{-6}$) & 94.22{\tiny$\pm$0.16}        &  0.55{\tiny$\pm$0.0047} (1.00$\times$)       &   11.1{\tiny$\pm$0.030} (0.99$\times$)  &  1.09$\times$ / 1.10$\times$  \\
          \tool ($\lambda_{gp} = 8e^{-6}$) & 94.09{\tiny$\pm$0.01}        &  0.49{\tiny$\pm$0.0002} (0.89$\times$) &   7.41{\tiny$\pm$0.004} (0.66$\times$)  &  1.00$\times$ / 0.85$\times$  \\
          \tool ($\lambda_{gp} = 16e^{-6}$) & 94.19{\tiny$\pm$0.07}        &  0.39{\tiny$\pm$0.0008} (0.70$\times$)       &   4.08{\tiny$\pm$0.020} (0.37$\times$)  &  0.83$\times$ / 0.58$\times$   \\
          \tool ($\lambda_{gp} = 32e^{-6}$) & 93.97{\tiny$\pm$0.25}        &  0.25{\tiny$\pm$0.0013} (0.45$\times$)      &   2.19{\tiny$\pm$0.007} (0.20$\times$)  &  0.60$\times$ / 0.36$\times$  \\
          \tool ($\lambda_{gp} = 64e^{-6}$) & 93.86{\tiny$\pm$0.11}        &  0.15{\tiny$\pm$0.0006} (0.27$\times$)      &   1.23{\tiny$\pm$0.004} (0.11$\times$)  &  0.39$\times$ / 0.22$\times$  \\
          \tool ($\lambda_{gp} = 128e^{-6}$) & 93.37{\tiny$\pm$0.07}        &  0.094{\tiny$\pm$0.0006} (0.17$\times$)      &   0.79{\tiny$\pm$0.009} (0.07$\times$)  &  0.25$\times$ / 0.13$\times$  \\
          \tool ($\lambda_{gp} = 256e^{-6}$) & 92.48{\tiny$\pm$0.04}        &  0.064{\tiny$\pm$0.0002} (0.12$\times$)      &   0.54{\tiny$\pm$0.006} (0.05$\times$)  &  0.16$\times$ / 0.08$\times$  \\
          \tool ($\lambda_{gp} = 512e^{-6}$) & 91.14{\tiny$\pm$0.16}        &  0.044{\tiny$\pm$0.0004} (0.08$\times$)      &   0.37{\tiny$\pm$0.007} (0.03$\times$)  &  0.11$\times$ / 0.05$\times$  \\
          \tool ($\lambda_{gp} = 1024e^{-6}$) & 89.55{\tiny$\pm$0.30}        &  0.032{\tiny$\pm$0.0002} (0.06$\times$)      &   0.27{\tiny$\pm$0.007} (0.02$\times$)  &  0.07$\times$ / 0.03$\times$  \\
          \bottomrule
      \end{tabular}
  \end{adjustbox}
\end{table}
\begin{table}[h]
  \centering
  \captionsetup{font=small,labelfont=bf}
  \caption{VGG19 performance on CIFAR10 for different regularization parameters. The last column in the table displays the relative total training cost in terms of the number of Multiply-Accumulate operations (MACs) and model parameters, compared to the non-factorized model.}
  \label{tab:vgg_gp_lambda}
  \begin{adjustbox}{max width=.8\linewidth}
      \begin{tabular}{llllc}
          \toprule
          \rowcolor{Gray} Variant                         & Acc. (\%)   & GMACs (Inf.) &  Params. (M) (Inf.) & Rel. MACs / Params. (Train.) \\ \midrule
          Non-Factorized     & 92.94{\tiny$\pm$0.17}        &  0.40{\tiny$\pm$0} (1.00$\times$)      &   20{\tiny$\pm$0} (1.00$\times$)  &  1.00$\times$ / 1.00$\times$         \\
          \tool ($\lambda_{gp} = 0.$) & 93.06{\tiny$\pm$0.17}        &  0.40{\tiny$\pm$0} (1.00$\times$)      &   20{\tiny$\pm$0} (1.00$\times$)  &  1.10$\times$ / 1.12$\times$   \\
          \tool ($\lambda_{gp} = 4e^{-6}$) & 93.33{\tiny$\pm$0.08}        &  0.39{\tiny$\pm$0.0017} (0.97$\times$)      &   18.8{\tiny$\pm$0} (0.94$\times$)  &  1.06$\times$ / 1.04$\times$   \\
          \tool ($\lambda_{gp} = 8e^{-6}$) & 93.27{\tiny$\pm$0.33}        &  0.30{\tiny$\pm$0.0017} (0.76$\times$)&   9.91{\tiny$\pm$0.008} (0.49$\times$)  &  0.90$\times$ / 0.73$\times$   \\
          \tool ($\lambda_{gp} = 16e^{-6}$) & 93.13{\tiny$\pm$0.07}        &  0.21{\tiny$\pm$0.0014} (0.53$\times$)      &   4.66{\tiny$\pm$0.052} (0.23$\times$)  &  0.69$\times$ / 0.46$\times$   \\
          \tool ($\lambda_{gp} = 32e^{-6}$) & 93.10{\tiny$\pm$0.10}        &  0.13{\tiny$\pm$0.0009} (0.33$\times$)      &   2.20{\tiny$\pm$0.025} (0.11$\times$)  &  0.47$\times$ / 0.27$\times$   \\
          \tool ($\lambda_{gp} = 64e^{-6}$) & 92.70{\tiny$\pm$0.34}        &  0.08{\tiny$\pm$0.0005} (0.20$\times$)      &   1.17{\tiny$\pm$0.010} (0.06$\times$)  &  0.30$\times$ / 0.16$\times$   \\
          \tool ($\lambda_{gp} = 128e^{-6}$) & 92.34{\tiny$\pm$0.12}        &  0.05{\tiny$\pm$0.0005} (0.13$\times$)      &   0.72{\tiny$\pm$0.002} (0.04$\times$)  &  0.19$\times$ / 0.09$\times$   \\
          \tool ($\lambda_{gp} = 256e^{-6}$) & 91.12{\tiny$\pm$0.19}        &  0.04{\tiny$\pm$0.0007} (0.09$\times$)      &   0.50{\tiny$\pm$0.023} (0.02$\times$)  &  0.12$\times$ / 0.05$\times$   \\
          \tool ($\lambda_{gp} = 512e^{-6}$) & 88.53{\tiny$\pm$0.13}        &  0.03{\tiny$\pm$0.0003} (0.06$\times$)      &   0.35{\tiny$\pm$0.003} (0.02$\times$)  &  0.08$\times$ / 0.03$\times$   \\
          \bottomrule
      \end{tabular}
  \end{adjustbox}
\end{table}
\begin{table}[!h]
  \centering
  \captionsetup{font=small,labelfont=bf}
  \caption{Transformer performance on Multi30k for different regularization parameters. The last column in the table displays the relative total training cost in terms of the number of Multiply-Accumulate operations (MACs) and model parameters, compared to the non-factorized model.}
  \label{tab:transf_gp_lambda}
  \begin{adjustbox}{max width=.8\linewidth}
      \begin{tabular}{lllllc}
          \toprule
          \rowcolor{Gray} Variant                         & Acc. (\%) & Ppl.   & GMACs (Inf.) &  Params. (M) (Inf.) & Rel. MACs / Params. (Train.) \\ \midrule
          Non-Factorized     & 65.33{\tiny$\pm$1.13}  & 9.85{\tiny$\pm$0.10}      &  1.370{\tiny$\pm$0.0000} (1.00$\times$)      &   53.9{\tiny$\pm$0.000} (1.00$\times$)  &  1.00$\times$ / 1.00$\times$         \\
          \tool ($\lambda_{gp} = 0.32$) & 61.30{\tiny$\pm$0.26}   & 12.99{\tiny$\pm$0.31}     &  1.125{\tiny$\pm$0.0030} (0.82$\times$)      &   45.1{\tiny$\pm$0.101} (0.84$\times$)  &  1.03$\times$ / 1.14$\times$   \\
          \tool ($\lambda_{gp} = 0.64$) & 63.78{\tiny$\pm$0.14}   & 9.37{\tiny$\pm$0.32}     &  0.957{\tiny$\pm$0.0112} (0.70$\times$)       &   39.1{\tiny$\pm$0.413} (0.73$\times$)  &  0.95$\times$ / 1.05$\times$   \\
          \tool ($\lambda_{gp} = 1.28$) & 66.14{\tiny$\pm$0.08}    & 7.02{\tiny$\pm$0.17}    &  0.570{\tiny$\pm$0.0088} (0.42$\times$)       &   25.3{\tiny$\pm$0.315} (0.47$\times$)  &  0.75$\times$ / 0.86$\times$   \\
          \tool ($\lambda_{gp} = 2.56$) & 66.08{\tiny$\pm$0.09}    & 6.90{\tiny$\pm$0.07}    &  0.248{\tiny$\pm$0.0032} (0.18$\times$)       &   13.8{\tiny$\pm$0.113} (0.26$\times$)  &  0.47$\times$ / 0.58$\times$   \\
          \tool ($\lambda_{gp} = 5.12$) & 57.70{\tiny$\pm$0.13}    & 13.97{\tiny$\pm$0.43}    &  0.123{\tiny$\pm$0.0002} (0.9$\times$)       &    9.3{\tiny$\pm$0.001} (0.17$\times$)  &  0.28$\times$ / 0.39$\times$   \\
          \bottomrule
      \end{tabular}
  \end{adjustbox}
\end{table}

\begin{table}[!h]
  \centering
  \captionsetup{font=small,labelfont=bf}
  \caption{ResNet50 performance on ImageNet-1k for different regularization parameters. The last column in the table displays the relative total training cost in terms of the number of Multiply-Accumulate operations (MACs) and model parameters, compared to the non-factorized model.}
  \label{tab:resnet50_gp_lambda}
  \begin{adjustbox}{max width=.75\linewidth}
      \begin{tabular}{llllc}
          \toprule
          \rowcolor{Gray} Variant          & Acc. (\%)     & GMACs (Inf.)             &  Params. (M) (Inf.)   &      Rel. MACs / Params. (Train) \\ 
          \midrule
          \multicolumn{5}{l}{\textbf{No decomposition}} \\
          Non-Factorized                   & 76.00         &  4.12 (1.00$\times$)       &   25.56 (1.00$\times$) & 1.00$\times$ / 1.00$\times$    \\
          \midrule
          \multicolumn{5}{l}{\textbf{Not decomposing first four blocks and last layer}} \\
          \tool ($\lambda_{gp} = 2e^{-6}$) & 76.04         & 3.43 (0.83$\times$)        & 14.02 (0.55$\times$)   & 0.87$\times$ / 0.64$\times$    \\
          \tool ($\lambda_{gp} = 4e^{-6}$) & 75.74         & 3.39 (0.82$\times$)        & 13.11 (0.51$\times$)   & 0.85$\times$ / 0.59$\times$    \\
          \tool ($\lambda_{gp} = 8e^{-6}$) & 75.15         & 3.21 (0.78$\times$)        & 11.46 (0.45$\times$)   & 0.83$\times$ / 0.55$\times$    \\
          \midrule
          \multicolumn{5}{l}{\textbf{Decomposing all layers}} \\
          \tool ($\lambda_{gp} = 0.$)      & 72.82         & 4.12 (1.00$\times$)        & 25.56 (1.00$\times$)   & 1.22$\times$ / 1.24$\times$    \\
          \tool ($\lambda_{gp} = 1e^{-6}$) & 72.81         & 3.62 (0.88$\times$)        & 18.77 (0.73$\times$)   & 1.00$\times$ / 0.87$\times$    \\
          \tool ($\lambda_{gp} = 2e^{-6}$) & 72.07         & 2.66 (0.65$\times$)        & 11.54 (0.45$\times$)   & 0.76$\times$ / 0.59$\times$    \\
          \tool ($\lambda_{gp} = 4e^{-6}$) & 71.54         & 2.01 (0.49$\times$)        & 9.21  (0.36$\times$)   & 0.57$\times$ / 0.57$\times$    \\
          \tool ($\lambda_{gp} = 8e^{-6}$) & 71.02         & 1.69 (0.41$\times$)        & 7.21  (0.28$\times$)   & 0.50$\times$ / 0.39$\times$    \\
          \bottomrule
      \end{tabular}
  \end{adjustbox}
\end{table}
 
Lastly, we delve deeper into the observation of nested ranks with increasing $\lambda_{gl}$. Fig.~\ref{fig:one_step_pruning} outlines the performance of \tool ($\lambda_{gl} = 0$) across various ranks chosen by smaller models \tool ($\lambda_{gl} > 0$). We observe that \tool ($\lambda_{gl} = 0$) delivers impressive results—for example, we can reduce its parameters by 10x for VGG while preserving an accuracy of $87.7\%$ without any fine-tuning simply by leveraging rank structure from separate runs. For LeNet, a reduction in model size by a factor of three is achievable without sacrificing accuracy. Last, for ResNet-18 the reduction is 1.7$\times$. As highlighted earlier, we aim to delve deeper into this subject in future studies. %

\begin{figure*}[h]
  \centering
  \begin{subfigure}[t]{0.32\textwidth}
    \includegraphics[width=\textwidth]{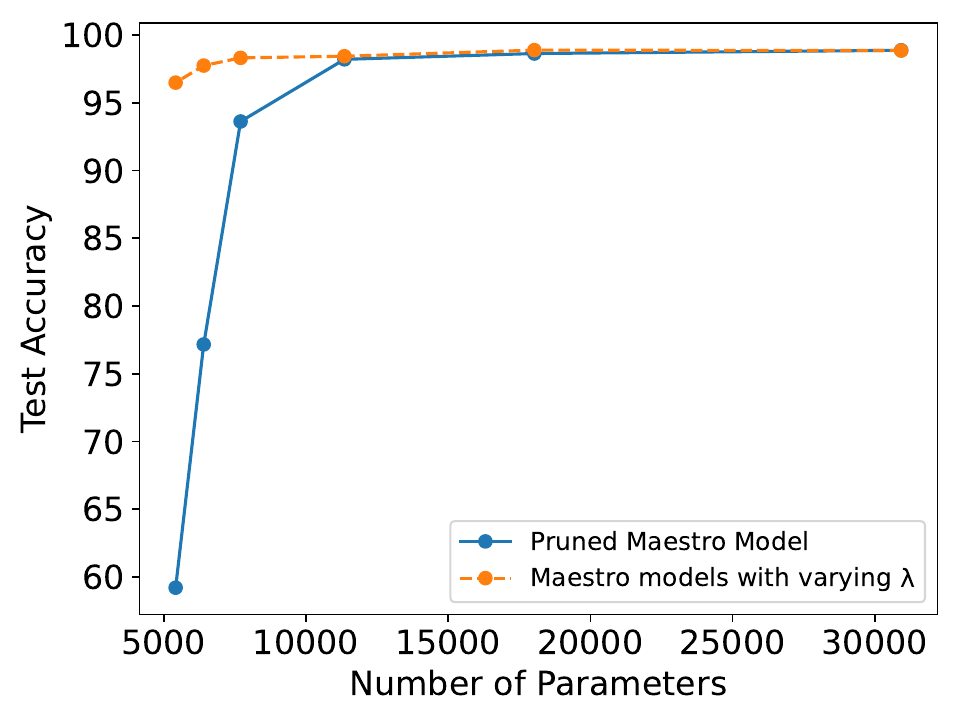}
    \captionsetup{font=small,labelfont=bf}
    \caption{LeNet on MNIST}
  \end{subfigure}
  \hfill
  \begin{subfigure}[t]{0.32\textwidth}
    \includegraphics[width=\textwidth]{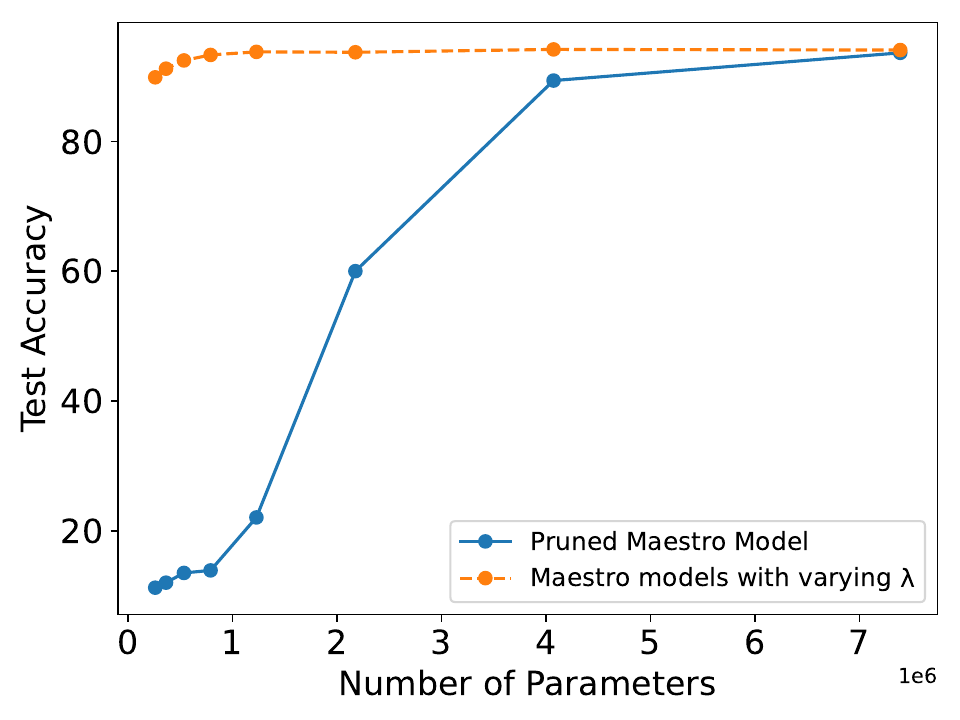}
    \captionsetup{font=small,labelfont=bf}
    \caption{ResNet-18 on CIFAR10}
  \end{subfigure}
  \hfill
  \begin{subfigure}[t]{0.32\textwidth}
    \includegraphics[width=\textwidth]{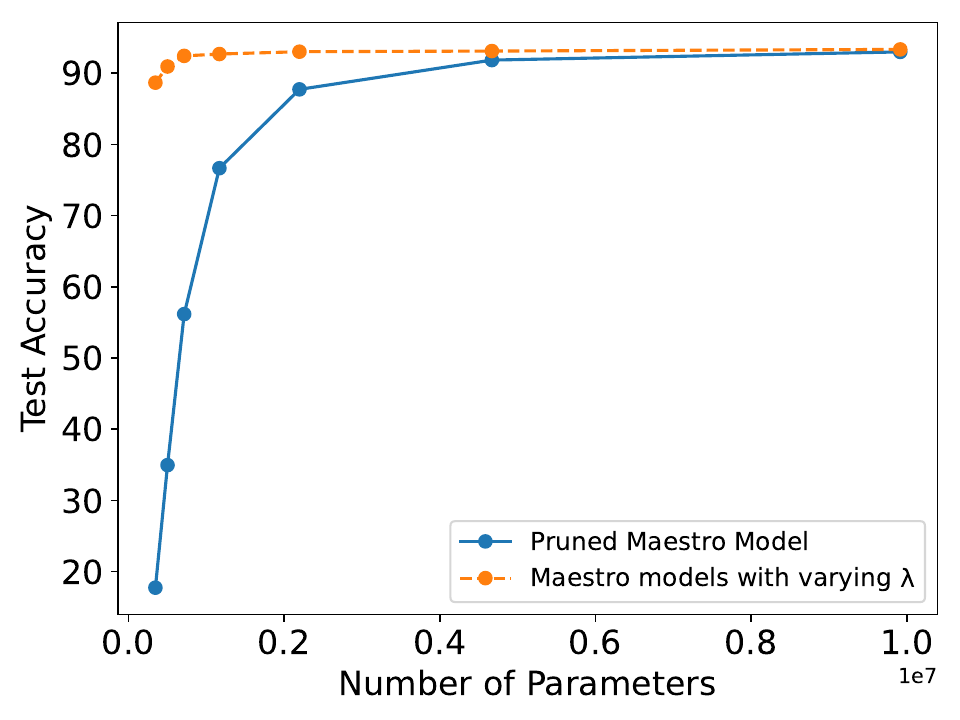}
    \captionsetup{font=small,labelfont=bf}
    \caption{VGG19 on CIFAR10}
  \end{subfigure}
  \captionsetup{font=small,labelfont=bf}
  \caption{\tool with progressive pruning to showcase nested rank importance structure. The original model corresponds to an evaluation in Fig.~\ref{fig:hierarchical_group_lasso}, and pruned models are based on \tool with $\lambda_{gl} = 0$, and they are pruned using the same ranks as selected by \tool with $\lambda_{gl} > 0$.}
  \label{fig:one_step_pruning} 
\end{figure*}

\subsection{Detailed Comparison with Baselines}
\label{app:detailed_baselines}

\begin{table}[t]
    \centering
    \vspace{-0.25cm}
    \captionsetup{font=small,labelfont=bf}
    \caption{Maestro vs. baselines on CIFAR10.}
    \vspace{-0.2cm}
    \label{tab:cifar10_baselines}
    \scalebox{0.75}{
        \begin{tabular}{lllll}
            \toprule
            \rowcolor{Gray} Variant                         & Model       & Acc. (\%)       & GMACs & Params. ($M$) \\ 
            \midrule
            Non-factorized              & ResNet-18 & 93.86{\tiny$\pm0.20$} & 0.56 & 11.17 \\
            Pufferfish                      & ResNet-18 & 94.17 & 0.22 & 3.336 \\
            Cuttlefish                      & ResNet-18 & 93.47 & 0.3 & 3.108 \\
            IMP                             & ResNet-18 & 92.12 & - & 0.154\\
            RareGems                        & ResNet-18 &92.83 & - & \textbf{0.076}\\
            XNOR-Net                        & ResNet-18 & 90.06 & - & 0.349$^\dagger$ \\
            \tool{}$^\dagger$ & \multirow{2}{*}{ResNet-18} & \multirow{2}{*}{\textbf{94.19{\tiny$\pm$0.07}}} & \multirow{2}{*}{{0.39{\tiny$\pm$0.00}}} & \multirow{2}{*}{4.08{\tiny$\pm$0.02}} \\
            ($\lambda_{gp} = 16e^{-6}$) & & & &  \\
            \tool{}$^\dagger$ & \multirow{2}{*}{ResNet-18} & \multirow{2}{*}{93.86{\tiny$\pm$0.11}} & \multirow{2}{*}{{0.15{\tiny$\pm$0.00}}} & \multirow{2}{*}{1.23{\tiny$\pm$0.00}} \\
           ($\lambda_{gp} = 64e^{-6}$) & & & &  \\
            \midrule
            Non-factorized              & VGG-19 & 92.94{\tiny$\pm0.17$} & 0.40 & 20.56 \\
            Pufferfish                      & VGG-19 & 92.69 & 0.29 & 8.37 \\
            Cuttlefish                      & VGG-19 & \textbf{93.39} & 0.15  & 2.36 \\
            RareGems                        & VGG-19 & 86.28 & - & 5.04 \\
            IMP                             & VGG-19 & 92.86 & - & 5.04\\
            XNOR-Net                        & VGG-19 & 88.94 & - & {0.64}$^\dagger$ \\
            Spectral Init.$^*$              & VGG-19 & 83.27 & - & $\approx$ 0.4 \\
            \tool{}$^\dagger$ & \multirow{2}{*}{VGG-19} & \multirow{2}{*}{93.10{\tiny$\pm$0.10}} & \multirow{2}{*}{0.13{\tiny$\pm$0.00}} & \multirow{2}{*}{2.20{\tiny$\pm$0.03}} \\
            ($\lambda_{gp} = 32e^{-6}$) & & & &  \\
            \tool{}$^\dagger$ & \multirow{2}{*}{VGG-19} & \multirow{2}{*}{88.53{\tiny$\pm$0.13}} & \multirow{2}{*}{\textbf{0.03{\tiny$\pm$0.00}}} & \multirow{2}{*}{\textbf{0.35{\tiny$\pm$0.00}}} \\
            ($\lambda_{gp} = 512e^{-6}$) & & & &  \\
            \bottomrule
            \multicolumn{5}{p{10.4cm}}{{$^*$Results from original work; $\dagger$: XNOR-Net employs binary weights and activations; although the overall \#trainable parameters remain the same as the vanilla network, each model weight is quantized from 32-bit to 1-bit. Therefore, we report a compression rate of $3.125\% (\nicefrac{1}{32})$.}}
        \end{tabular}
    }
\end{table}

\red{Tab.~\ref{tab:cifar10_baselines} presents the details of \tool's performance compared to the selected baselines leveraging pruning, quantization and low-rank techniques presented in Sec.~\ref{sec:performance_comparison} for CIFAR-10. These numbers along with the operating points from Tab.~\ref{tab:resnet_gp_lambda} and~\ref{tab:vgg_gp_lambda} are illustrated in Fig.~\ref{fig:cifar10_baselines}.}

\end{document}